\documentclass[11pt]{article}
\usepackage[T1]{fontenc}
\usepackage{lmodern}
\usepackage[letterpaper,margin=1in]{geometry}
\usepackage{amsmath,amssymb,amsthm,mathtools}
\usepackage{microtype,nicefrac}
\usepackage{algorithm}
\usepackage[noend]{algorithmic}
\usepackage[section]{placeins}
\usepackage{enumitem,etoolbox,xcolor}
\usepackage{refcount}

\usepackage{tikz}
\usetikzlibrary{arrows.meta,decorations.pathreplacing,calc}
\usepackage[font=footnotesize]{caption}
\usepackage[numbers]{natbib}
\usepackage[colorlinks=true,linkcolor=blue!55!black,citecolor=blue!55!black,urlcolor=blue!55!black]{hyperref}
\usepackage[nameinlink,noabbrev,capitalise]{cleveref}
\crefname{equation}{Eq.}{Eqs.}
\Crefname{equation}{Eq.}{Eqs.}
\crefname{appendix}{Appendix}{Appendices}
\Crefname{appendix}{Appendix}{Appendices}
\crefname{ALC@unique}{line}{lines}
\Crefname{ALC@unique}{Line}{Lines}
\appto\appendix{\crefalias{section}{appendix}\crefalias{subsection}{appendix}\crefalias{subsubsection}{appendix}}
\newtheorem{theorem}{Theorem}
\newcounter{restatedthm}

\newtheorem{restatedtheorem}[restatedthm]{Theorem}
\crefname{restatedthm}{Theorem}{Theorems}
\Crefname{restatedthm}{Theorem}{Theorems}
\crefname{restatedtheorem}{Theorem}{Theorems}
\Crefname{restatedtheorem}{Theorem}{Theorems}
\newtheorem{lemma}{Lemma}
\newtheorem{proposition}{Proposition}
\newtheorem{corollary}{Corollary}
\theoremstyle{definition}
\newtheorem{definition}{Definition}

\newcommand{\E}{\mathbb E}
\newcommand{\Prb}{\mathbb P}
\newcommand{\R}{\mathbb R}
\newcommand{\one}{\mathbf 1}
\newcommand{\Var}{\operatorname{Var}}
\newcommand{\diag}{\operatorname{diag}}

\newcommand{\Reg}{\mathrm{Reg}}
\newcommand{\ip}[2]{\langle #1,#2\rangle}
\newcommand{\cQ}{\mathcal Q}
\newcommand{\cP}{\mathcal P}
\newcommand{\cC}{\mathcal C}
\newcommand{\cT}{\mathcal T}
\newcommand{\cU}{\mathcal U}
\newcommand{\cV}{\mathcal V}
\newcommand{\cF}{\mathcal F}
\newcommand{\lmax}{\lambda_{\max}}
\newcommand{\proc}[1]{\textnormal{\textsc{#1}}}
\setlist[enumerate]{label=(\roman*),itemsep=3pt,topsep=5pt}
\setlist[itemize]{itemsep=3pt,topsep=5pt}
\allowdisplaybreaks[1]

\title{The Price of Partial Information with Switching Costs:\\A Complete Characterization for Feedback Graphs\thanks{A Lean formalization of the mathematical results of this paper, excluding computational complexity claims, is available at \url{https://github.com/robertocolomboni/feedback-graphs-lean}.}}
\author{
Nicol\`o Cesa-Bianchi\thanks{Computer Science Department, University of Milan, Milan, Italy.}
\and
Tommaso Cesari\thanks{School of Electrical Engineering and Computer Science, University of Ottawa, Ottawa, Canada.}
\and
Roberto Colomboni\thanks{School of Mathematics, University of Bristol, Bristol, United Kingdom.}
\and
Emmanuel Esposito\thanks{FairPlay Team, CREST, ENSAE, IP Paris, Paris, France.}
\and
Federico Fusco\thanks{Department of Computer, Control, \& Management Engineering, Sapienza University of Rome, Rome, Italy.}
\and
Tomer Koren\thanks{The Blavatnik School of Computer Science and AI, Tel Aviv University, Tel Aviv, Israel.}
}
\date{}

\begin{document}
\pagenumbering{roman}
\maketitle
\thispagestyle{empty}

\begin{abstract}We study the effect of switching costs on the minimax rates of regret minimization, and how it depends on the available feedback. It is well known that with full information, switching costs leave the rate $\widetilde\Theta(\sqrt T)$ over $T$ rounds unchanged (Kalai and Vempala, COLT 2003), whereas with bandit feedback they increase it from $\sqrt T$ to $T^{2/3}$ (Dekel, Ding, Koren, and Peres, STOC 2014).
This raises a natural question: which features of the feedback determine whether switching costs preserve the rate $\sqrt T$ or force the rate $T^{2/3}$? We address this question in the framework of feedback graphs, which interpolate between full information and bandits: playing an action reveals the losses of its out-neighbors in a known directed graph. Prior work on feedback graphs with switching costs gave algorithms with regret of order $T^{2/3}$, but matching lower bounds only for restricted classes of graphs, leaving the question open.

We establish a complete characterization of the minimax regret with switching costs for all feedback graphs: $\widetilde\Theta(\sqrt T)$ if the graph is \emph{strongly revealing}, meaning that every pair of distinct actions has an endpoint that observes the losses of both; $\widetilde\Theta(\gamma^{1/3}T^{2/3})$ for every other observable graph, where $\gamma$ is the domination number; and linear otherwise. In particular, \emph{tournaments}, in which every action observes its own loss and each pair of actions is joined by one directed edge, sit at the phase transition: they behave like full information, but deleting the edge of any pair forces the rate $T^{2/3}$. Hence an action that observes all losses is neither necessary nor sufficient for the rate $\sqrt T$. Our algorithm \proc{TournamentSteering} for the strongly revealing case rests on a new principle: each observed comparison of two actions serves both to estimate losses and to steer play, and the loss reduction from steering pays for the variance of the estimates. For every strongly revealing graph, distributing the comparisons among observing actions keeps this variance, and hence the regret, polylogarithmic in the number of actions, by an argument similar to the greedy construction of logarithmic dominating sets in tournaments.

\end{abstract}
\newpage
\begingroup
\makeatletter
\patchcmd{\l@section}{1.0em}{3pt}{}{\PackageError{compacttoc}{Could not adjust TOC section spacing}{}}
\makeatother
\tableofcontents
\endgroup

\clearpage
\pagenumbering{arabic}

\section{Introduction}
\label{sec:intro}

Online learning algorithms
often randomize their actions to hedge against adversaries. However, sampling independently in each round can then produce frequent switches: the Hedge algorithm \citep{freund1997decision}, for instance, may keep mixing between two experts of comparable performance rather than committing to one of them. In many applications such instability can be costly, because changing a route, reallocating a resource, replacing a deployed model, or switching between treatments incurs a cost that the loss does not reflect. The model of regret minimization with switching costs accounts for these changes. In each of $T$ rounds, a learner chooses one of $K$ actions, incurs the loss of the chosen action, and pays an additional unit cost whenever its action differs from the previous one. The losses lie in $[0,1]$ and are fixed in advance by an oblivious adversary. The regret is defined as the learner's expected total cost (including switching costs) minus the total loss of the best fixed action.
Switching costs are a specific instance of losses with memory \citep{arora2012online,Cesa-BianchiDS13,AnavaHM15} and connect regret minimization with competitive analysis of metrical task systems \citep{BlumB00}. They also reveal a striking separation between feedback models: with full information, where all action losses are observed after each round, the minimax regret is of order $\sqrt{T\log K}$ with or without switching costs \citep{KalaiV05,GeulenVW10,DevroyeLN13}. With bandit feedback, where only the loss of the chosen action is observed, the minimax regret is of order $\sqrt{KT}$ without switching costs and jumps to $K^{1/3}T^{2/3}$ (up to logarithmic factors) when switches are charged for \citep{arora2012online,Cesa-BianchiDS13,DekelDingKorenPeres2014}. The jump for adversarial bandits has a simple cause: a bandit learner learns about an action only while playing it, so it can compare two actions only by moving between them, and every move is charged. Switching costs thus do more than add a stability penalty of lower order: they put a price on information. This raises a natural question: \emph{which features of the feedback determine whether switching costs preserve the $\sqrt{T}$ rate or force the $T^{2/3}$ rate?}

We investigate this question through the feedback-graph model \citep{MannorS11,AlonCDK15}, which interpolates between experts and bandits. A directed graph on the actions, known to the learner, specifies the feedback: playing action $i$ reveals the losses of all actions $j$ such that $(i,j)$ is an edge in the graph. Bandit feedback corresponds to the graph with only self-loops, and full information is the complete graph, including all self-loops. A graph is \emph{observable} if every action is observed by some action, and \emph{strongly observable} if every action observes its own loss or is observed by all other actions. Without switching costs, Alon, Cesa-Bianchi, Dekel, and Koren \citep{AlonCDK15} showed that the minimax regret is of order $\sqrt T$ for strongly observable graphs, of order $T^{2/3}$ for the other observable graphs, which are called \emph{weakly observable}, and linear for the rest, up to logarithmic factors and graph-dependent constants. For weakly observable graphs, switching costs leave the exponent unchanged. Bandit feedback and full information are both strongly observable, so the phase transition from $\sqrt T$ to $T^{2/3}$ caused by switching costs runs through the family of strongly observable graphs, and existing observability notions do not capture it.

A natural guess is that switching costs force the rate $T^{2/3}$ whenever no single action observes all losses, since the learner must then move to collect information. The known lower bounds cover undirected graphs with at least two nonadjacent actions \citep{AroraMM19} and graphs in which no action observes both actions of some pair \citep{FengL18}, but leave open the case of \emph{tournaments} (with self-loops), in which every action observes its own loss and every pair of actions is joined by one directed edge. In the cyclic tournament on three actions, a directed $3$-cycle (\cref{fig:classes}(a)), each action observes itself and its successor, so every action is blind to one of the others. Yet, every pair can be compared while playing one of its members, so the known lower bounds do not apply. It was therefore unknown whether even this $3$-cycle behaves like bandits or like experts, and more generally which feedback graphs preserve the square-root regret rate under switching costs.

\paragraph{Contributions.}

We close this gap and characterize the minimax regret with switching costs for \emph{every} feedback graph, up to logarithmic factors. We call a graph \emph{strongly revealing} if every pair of distinct actions has an endpoint that observes the losses of both.
Our main result is the following trichotomy for every feedback graph, where the domination number $\gamma$ is the smallest number of actions that together observe all losses---see also \cref{fig:classes}.
\begin{theorem}[Characterization]
\label{thm:intro-char}
For every feedback graph on $K \ge 2$ actions with domination number $\gamma$, the minimax regret with switching costs over $T \ge 2$ rounds is\footnote{Throughout, $\widetilde O$, $\widetilde\Omega$ and $\widetilde\Theta$ suppress constants and factors polylogarithmic in $K$ and $T$.}
\begin{enumerate}[nosep]
\item $\widetilde\Theta(\sqrt T)$ if the graph is strongly revealing;
\item $\widetilde\Theta(\min\{T,\gamma^{1/3}T^{2/3}\})$
if the graph is observable but not strongly revealing;
\item $\Theta(T)$ if the graph is not observable.
\end{enumerate}
\end{theorem}

Our main technical contribution regards strongly revealing graphs: we give an algorithm with running time polynomial in $K$ and $T$ whose expected regret and expected number of switches are $O(\sqrt{T\log K}\,\log^2(KT))$ (\cref{thm:main}). For observable graphs that are not strongly revealing, we give upper and lower bounds with matching dependence on the domination number $\gamma$, up to logarithmic factors in $K$ and $T$ (\cref{thm:explore,cor:domination}). In particular, the $3$-cycle attains the $\sqrt T$-rate despite having no action that reveals all losses, whereas two actions that observe only themselves and a third that reveals all losses form a graph with regret $T^{2/3}$ (\cref{fig:classes} (a,b)).

\begin{figure}[t]
\centering
\begin{tikzpicture}[font=\footnotesize,
  act/.style={circle,draw,thick,minimum size=12pt,inner sep=0pt},
  obs/.style={-{Stealth[length=4.5pt]},thick},
  lp/.style={-{Stealth[length=4pt]},thick,looseness=8,min distance=8pt},
  title/.style={font=\footnotesize\itshape,align=center},
  row/.style={font=\scriptsize\itshape,anchor=east}]
\begin{scope}
\node[act] (a1) at (90:0.8) {$1$};
\node[act] (a2) at (210:0.8) {$2$};
\node[act] (a3) at (330:0.8) {$3$};
\draw[obs,line width=1.2pt] (a1) -- (a2);
\draw[obs] (a2) -- (a3);
\draw[obs] (a3) -- (a1);
\draw[lp,line width=1.2pt] (a1) to[out=60,in=120] (a1);
\draw[lp] (a2) to[out=150,in=210] (a2);
\draw[lp] (a3) to[out=-30,in=30] (a3);
\node[title] at (0,1.75) {(a) strongly revealing};
\node at (0,-1.25) {$\widetilde\Theta(\sqrt T)$};
\node at (0,-1.7) {$\widetilde\Theta(\sqrt T)$};
\end{scope}
\begin{scope}[xshift=3.4cm]
\node[act,fill=gray!30] (r) at (90:0.8) {$r$};
\node[act] (u) at (210:0.8) {$u$};
\node[act] (v) at (330:0.8) {$v$};
\draw[obs] (r) -- (u);
\draw[obs] (r) -- (v);
\draw[densely dashed,thick,gray] (u) -- (v);
\draw[lp] (r) to[out=60,in=120] (r);
\draw[lp] (u) to[out=150,in=210] (u);
\draw[lp] (v) to[out=-30,in=30] (v);
\node[title] at (0,1.93) {(b) strongly observable,\\not strongly revealing};
\node at (0,-1.25) {$\widetilde\Theta(\gamma^{1/3}T^{2/3})$};
\node at (0,-1.7) {$\widetilde\Theta(\sqrt{\alpha T})$};
\end{scope}
\begin{scope}[xshift=6.8cm]
\node[act,fill=gray!30] (r) at (90:0.8) {$r$};
\node[act] (u) at (210:0.8) {$u$};
\node[act] (v) at (330:0.8) {$v$};
\draw[obs] (r) -- (u);
\draw[obs] (r) -- (v);
\draw[lp] (r) to[out=60,in=120] (r);
\node[title] at (0,1.75) {(c) weakly observable};
\node at (0,-1.25) {$\widetilde\Theta(\gamma^{1/3}T^{2/3})$};
\node at (0,-1.7) {$\widetilde\Theta(\delta^{1/3}T^{2/3})$};
\end{scope}
\begin{scope}[xshift=10.1cm]
\node[act] (x) at (-0.55,0) {$x$};
\node[act,dashed] (y) at (0.55,0) {$y$};
\draw[obs] (y) -- (x);
\draw[lp] (x) to[out=150,in=210] (x);
\node[title] at (0,1.75) {(d) not observable};
\node at (0,-1.25) {$\Theta(T)$};
\node at (0,-1.7) {$\Theta(T)$};
\end{scope}
\node[row] at (-1.45,-1.25) {with switching costs};
\node[row] at (-1.45,-1.7) {without switching costs};
\end{tikzpicture}
\caption{The minimax regret with switching costs (\cref{thm:intro-char}) and without them \citep{AlonCDK15}, where $\alpha$ is the independence number and $\gamma,\delta$ the domination and weak domination numbers. An arrow $i \to j$ means that playing $i$ reveals the loss of $j$, and a self-loop that the played action reveals its own loss. (a) In the cyclic tournament on three actions, no action observes all losses, but every pair has an endpoint that observes both of its losses, as action $1$ does for the pair $\{1,2\}$ (thick). (b) Action $r$ observes all losses, so $\gamma=1$ (shaded), but the dashed pair $\{u,v\}$ has no observing endpoint, which forces regret of order $T^{2/3}$. (c) Without their self-loops, $u$ and $v$ are observed only by $r$, so the graph is weakly observable. (d) No action observes the loss of the dashed action $y$.}
\label{fig:classes}
\end{figure}

\paragraph{Challenges and techniques.}
Our algorithm uses each observed \emph{comparison}, the difference between the losses of a pair, both to guide its actions and to estimate the corresponding losses. This approach not only reduces the incurred losses, but also offsets the estimation penalty in the regret analysis. Three challenges of increasing generality lead to our general approach.
\begin{enumerate}[wide=0pt]

\item \emph{Apple tasting.} The minimal strongly revealing graph is apple tasting \citep{HelmboldLL2000}. There are only two actions: the revealing action observes both losses while the blind action observes nothing (\cref{fig:geometric-probing}(a)). Every two-action subgraph of a strongly revealing graph contains a copy of apple tasting and, surprisingly, this graph already challenges known techniques for switching costs.
Whenever the blind action is better, the learner must play it. Lacking any feedback, the only way to discover that the blind action has become worse is by switching to the revealing action. Treating these switches as exploration, as mini-batching does, leads to the well-known mini-batching technique ensuring only a $T^{2/3}$ rate \citep[e.g.,][]{arora2012online}.
Our \emph{comparison probing} algorithm, instead, uses the comparisons observed during a visit of the revealing action, showing whether staying pays off, to decide how long to stay. More specifically, it starts a probe independently in each round with probability $\kappa$ while staying at the blind action, and continues the probe while the revealing action has smaller cumulative loss since the probe began. Each completed probe, however long, costs at most one unit of loss and two switches more than being blind, so the regret against the blind action is $O(\kappa T)$. The regret against the revealing action is $O(1/\kappa+\kappa T)$, since the learner misses, in expectation, at most $1/\kappa$ chances to start a probe that would last to the horizon. Balancing gives $O(\sqrt T)$ regret without any logarithmic factors (see \cref{app:breakeven}).

\item \emph{The 3-cycle.} In the cyclic tournament on three actions (\cref{fig:classes}(a)), each pair contains a copy of apple tasting, but comparison probing does not directly apply for two reasons. First, every action is the blind endpoint of one pair and the revealing endpoint of another, so the three probing rules would prescribe different actions for starting a probe. Our solution, instead, plays a soft version of the probing rule via a slowly switching Markov chain whose stationary law is close to exponential weights. The weights are only updated between blocks and, within a block, the chain \emph{steers} each move by the comparisons just observed. Since the chain is reversible, it moves between two actions equally often in either direction, so steering from the observing endpoint alone yields, on average, the same loss response as steering from both endpoints.
Second, unlike apple tasting, exponential weights need loss estimates because no action observes all losses. As these estimates come from the learner's own trajectory, standard full-information low-switching analyses (e.g., \citep{KalaiV05,GeulenVW10}), which require updates independent of the current action, do not apply directly. Moreover, the hard-stopping rule of probing would create a bias hard to correct. Steering, instead, changes trajectory probabilities by a bounded positive factor, which can be compensated for via importance weighting. Steering is shown to gain when the comparisons it reacts to persist over the time scale of the switching rate, but with nonuniform weights, estimation errors persist on several time scales. We show that randomizing the rate provides gains on all of them on average, so that the loss reduction pays for the second moment of the estimates up to the order of the switching cost.

\item \emph{Scaling to $K$ actions.} The remaining challenge is to prevent the loss estimates from becoming less accurate as the number of actions grows. A direct extension of the $3$-cycle construction would assign each pairwise comparison to its observing endpoint. This can make an action's loss estimate depend heavily on observations collected only while playing that action. Under uniform sampling, these observations arrive only a $1/K$ fraction of the time, so importance weighting produces a second moment proportional to $K$, as with bandit feedback. To improve on this, we exploit the overlap in tournament feedback: playing one action can reveal comparisons between several other actions, allowing the same observation to contribute to many loss estimates. A weighted greedy argument, inspired by logarithmic dominating sets of tournaments \citep{megiddo1988dominating}, gives a fractional assignment that distributes the comparisons among observing actions while balancing their contributions. We then choose both the sampling distribution and the switching rate so that actions carrying large estimation weights are sampled often enough and left quickly, limiting the persistence of their errors. Together, these choices reduce the second moment's dependence on $K$ from linear to logarithmic, and give the regret bound of \cref{thm:intro-char}(i).
\end{enumerate}

Finally, for a matching lower bound in the non-strongly-revealing case, we face two difficulties: (1) a pair without an observing endpoint may still be observed by a third action, which reveals the loss difference of the pair in one round under the multiscale random walk of \citep{DekelDingKorenPeres2014}; (2) the domination number $\gamma$ can be large even when every action observes its own loss. Our construction runs a version of this walk, with noisy loss differences, on a large set of actions of which every action observes at most a few, so that information about the better member comes only from switches and from rounds outside this set. Learning which member is better thus requires paying in switching costs or losses, which yields the $\widetilde\Omega(\gamma^{1/3}T^{2/3})$ lower bound.

\paragraph{Related work.}
Feedback graphs were introduced by Mannor and Shamir \citep{MannorS11}. When every action observes its own loss, the minimax regret without switching costs is of order $\sqrt{\alpha T}$ where $\alpha$ is the independence number of the graph \citep{MannorS11,alon2013bandits,alon2017nonstochastic}, see \citep{kocak2014efficient,AlonCDK15,ZimmertL19,ChenHuangLiZhang2021,EldowaECC23,KocakCarpentier2023,ChenHeZhang2024} for refinements. The characterization in \Citep{AlonCDK15}
covers all graphs: the rate is
$\sqrt{\alpha T}$ for strongly observable graphs, and $T^{2/3}$ times the cube root of the weak domination number (the smallest number of actions that together observe all weakly observable actions) for weakly observable graphs. Other work lets the graph change from round to round without revealing it to the learner \citep{CohenHK16}, draws the graph at random \citep{EspositoFHC22}, extends the model to cooperating learners \citep{CesaBianchiCesariDellaVecchia2024}, or seeks bounds that also adapt to stochastic losses \citep{DannWeiZimmert2023}. These works do not charge switches; with switching costs, we show that the strongly observable graphs split along a pairwise condition and the middle regime is governed by the domination number, which can be much larger than the weak domination number.

For feedback graphs with switching costs, Rangi and Franceschetti \citep{rangi2019online} give bounds of order $T^{2/3}$ for undirected graphs. Arora, Marinov, and Mohri \citep{AroraMM19} give an upper bound of order $\gamma^{1/3}T^{2/3}$ up to logarithmic factors, together with lower bounds of order $T^{2/3}/\log T$ for undirected graphs with two nonadjacent actions and $\gamma^{1/3}T^{2/3}/\log T$ for disjoint unions of star graphs, leaving matching domination-dependent lower bounds for general graphs open. Feng and Loh \citep{FengL18} prove a lower bound of order $T^{2/3}$ for strongly observable graphs in which no action observes both actions of some pair.
These results leave the minimax rate for nontrivial tournaments unresolved. Other models allow extra observations \citep{cheng2023understanding} or several actions held at once \citep{shi2022power}.

For bandits with switching costs, the upper bound is through mini-batching \citep{arora2012online} and the lower bound is due to \citep{DekelDingKorenPeres2014} (and to \citep{Cesa-BianchiDS13} for unbounded losses). Later work studies more refined movement costs \citep{koren2017bandits,koren2017multi} and bounds that adapt to stochastic losses \citep{rouyer2021algorithm,amir2022better}. With a budget on switches instead of a cost, stochastic bandits attain the rate $\sqrt{KT}$ with only $O(K\log\log T)$ switches \citep{simchilevi2023phase}, and the budget produces a phase transition for full information but not for adversarial bandits \citep{altschuler2021online}.

Partial monitoring is the more general model in which each action reveals a fixed function of a hidden outcome \citep{CesaBianchiLugosiStoltz2006}. Excluding trivial and degenerate games, finite partial-monitoring games split into rates $\sqrt T$, $T^{2/3}$, and $T$, up to logarithmic factors \citep{AntosBPS13,BartokFPRS14}, and the line between the first two rates, \emph{local observability}, requires that the loss difference of every two neighboring actions must be estimable from the feedback of those two actions together.
It is not hard to see that local observability is equivalent to strong observability in feedback graphs, whereas our strongly revealing condition implies both.
In label-efficient prediction \citep{CesaBianchiLugosiStoltz2005,CesaBianchiLugosiStoltz2006}, observing the outcome requires a separate, costly revealing action, and the minimax regret is of order $T^{2/3}$. In apple tasting \citep{HelmboldLL1992,HelmboldLL2000}, the revealing action is itself one of the competing choices; recent work studies its combinatorial dimensions \citep{RamanSubediRamanTewari2024} and strategic variants \citep{HarrisPW23}.

Tournaments and pairwise comparisons also appear in dueling bandits \citep{YueBKJ12,SahaKM21}, in a different role: there the learner chooses a pair of actions and observes the outcome of their comparison, whereas here it plays a single action, pays its loss, and the tournament describes which comparisons that action reveals.

\section{Technical overview}
\label{sec:tech}

This section sketches the proof of \cref{thm:intro-char}, with emphasis on the upper bound for strongly revealing graphs; \cref{sec:setting-result,sec:algorithm,sec:analysis,sec:regimes} give the complete statements and proofs. Throughout, $\ell_t\in[0,1]^K$ is the loss vector of round $t$. For actions $i$ and $j$, $\Delta_t(i,j)=\ell_t(i)-\ell_t(j)$ is the advantage of $j$ over $i$, the extra loss of playing $i$ rather than $j$. An action that observes the losses of both $i$ and $j$ observes this advantage, which we call the \emph{comparison} of the pair.

The algorithm uses only comparisons on a tournament. Orienting each pair of a strongly revealing graph from an endpoint that observes both of its losses gives a spanning tournament in which the source $s$ of each edge $s\to j$ observes $s$ and $j$, and in which at most one action, a sink, lacks a self-loop (see \cref{prop:graph}). Each edge $s\to j$ defines an instance of apple tasting, with $s$ revealing and $j$ blind, so a strongly revealing graph superimposes $\binom{K}{2}$ such instances on a single trajectory of play. The remainder of the section addresses the three main challenges: (1) apple tasting and comparison probing (\cref{sec:tech-two}); (2) the cyclic tournament on three actions, in which the probing rules of the three pairs conflict (\cref{sec:tech-obstacles}) and steering must gain on every time scale (\cref{sec:tech-energy}); (3) the scaling to $K$ actions with regret polylogarithmic, rather than linear, in $K$ (\cref{sec:tech-scaling}). Finally, \Cref{sec:tech-lower} sketches the lower bounds.

\subsection{Apple tasting}
\label{sec:tech-two}

Apple tasting (\cref{fig:geometric-probing}(a)) is the smallest graph on the $\sqrt T$ side of the regret landscape: other strongly revealing graphs on two actions contain it, and removing its edge or its self-loop leaves a graph with linear regret. Its solution is therefore the first step toward the upper bound, and also the case in which the intuition from bandits is strongest: every comparison must be paid for by a visit to the revealing action, which suggests the bandit rate $T^{2/3}$.

Let action $1$ be the revealing action and action $2$ the blind action. Let $d_t=\Delta_t(1,2)\in[-1,1]$ be the advantage of the blind action, and let $D_t=\sum_{u\le t}d_u$. Subtracting $\ell_t(2)$ from both losses does not change the regret, so action $2$ may be taken to have loss zero and action $1$ the signed loss $d_t$. Comparison probing ``rests'' at action $2$ and makes each visit (probe) to action $1$ pay for itself. When resting at the start of a round, it starts a probe in that round with probability $\kappa$, independently of everything else. A probe that starts at round $t_0$ plays action $1$ and continues after round $t$ if and only if its break-even statistic $D_t-D_{t_0-1}$ is negative, that is, while action $1$ has been better since the probe began. For every probe rate $\kappa$, comparison probing has regret of order $1/\kappa+\kappa T$, so that $\kappa \approx 1/\sqrt{T}$ gives $\sqrt{T}$ regret with no logarithmic factors, matching the $\sqrt T$ lower bound up to constants. Intuitively, the term $\kappa T$ ``pays for the probes'', each of which stops as soon as it stops paying and so costs at most one unit of loss and two switches more than resting. The term $1/\kappa$, instead, bounds the expected number of missed chances to start a probe that would never end, each costing at most one unit. \Cref{fig:geometric-probing}(b) illustrates the algorithm and its analysis.

\begin{figure}[t]
\centering
\begin{tikzpicture}[x=0.47cm,y=0.95cm,font=\footnotesize]
\begin{scope}[x=1cm,y=1cm,shift={(-3.3cm,0.35cm)},act/.style={circle,draw,thick,minimum size=12pt,inner sep=0pt}]
\node[act] (g1) at (0,1.15) {$1$};
\node[act] (g2) at (0,-0.35) {$2$};
\draw[-{Stealth[length=4.5pt]},thick] (g1) -- (g2);
\draw[-{Stealth[length=4pt]},thick,looseness=8,min distance=8pt] (g1) to[out=60,in=120] (g1);
\node[anchor=west] at (g1.east) {revealing};
\node[anchor=west] at (g2.east) {blind};
\end{scope}
\node at (-3.5cm,2.8cm) {(a)};
\node at (-1.0cm,2.8cm) {(b)};
\draw[->] (0,-1.3) -- (0,2.95) node[above] {$D_t$};
\draw[->] (0,0) -- (24.8,0) node[right] {$t$};
\draw[densely dotted,thick,gray] (3,2.4) -- (24,2.4);
\draw[densely dotted,thick,gray] (10,2) -- (24,2);
\draw[dashed] (4,1.5) -- (9,1.5);
\draw[dashed] (12,1.4) -- (15,1.4);
\draw[dashed,thick] (16,1.6) -- (24,1.6);
\draw[gray,thick,line cap=round] (0,0) -- (1,0.8);
\draw[gray,thick,line cap=round] (1,0.8) -- (2,1.6);
\draw[gray,thick,line cap=round] (2,1.6) -- (3,2.4);
\draw[gray,thick,line cap=round] (3,2.4) -- (4,1.5);
\draw[line width=1.6pt,line cap=round] (4,1.5) -- (5,0.8);
\draw[line width=1.6pt,line cap=round] (5,0.8) -- (6,0.3);
\draw[line width=1.6pt,line cap=round] (6,0.3) -- (7,0.7);
\draw[line width=1.6pt,line cap=round] (7,0.7) -- (8,1.3);
\draw[line width=1.6pt,line cap=round] (8,1.3) -- (9,1.8);
\draw[gray,thick,line cap=round] (9,1.8) -- (10,2);
\draw[gray,thick,line cap=round] (10,2) -- (11,1.1);
\draw[gray,thick,line cap=round] (11,1.1) -- (12,1.4);
\draw[line width=1.6pt,line cap=round] (12,1.4) -- (13,0.7);
\draw[line width=1.6pt,line cap=round] (13,0.7) -- (14,0.9);
\draw[line width=1.6pt,line cap=round] (14,0.9) -- (15,1.5);
\draw[gray,thick,line cap=round] (15,1.5) -- (16,1.6);
\draw[line width=1.6pt,line cap=round] (16,1.6) -- (17,0.8);
\draw[line width=1.6pt,line cap=round] (17,0.8) -- (18,1.3);
\draw[line width=1.6pt,line cap=round] (18,1.3) -- (19,0.7);
\draw[line width=1.6pt,line cap=round] (19,0.7) -- (20,0.2);
\draw[line width=1.6pt,line cap=round] (20,0.2) -- (21,0.5);
\draw[line width=1.6pt,line cap=round] (21,0.5) -- (22,-0.2);
\draw[line width=1.6pt,line cap=round] (22,-0.2) -- (23,-0.5);
\draw[line width=1.6pt,line cap=round] (23,-0.5) -- (24,-1);
\draw[fill=white,thick] (3,2.4) circle (2.6pt);
\draw[fill=white,thick] (10,2) circle (2.6pt);
\fill (16,1.6) circle (2.8pt) node[anchor=south east,inner sep=1.5pt] {$t^\star$};
\fill[gray!20] (0,-1.7) rectangle (4,-1.5);
\node[below=1pt] at (2,-1.7) {rest};
\fill[gray!65] (4,-1.7) rectangle (9,-1.5);
\node[below=1pt] at (6.5,-1.7) {probe};
\fill[gray!20] (9,-1.7) rectangle (12,-1.5);
\node[below=1pt] at (10.5,-1.7) {rest};
\fill[gray!65] (12,-1.7) rectangle (15,-1.5);
\node[below=1pt] at (13.5,-1.7) {probe};
\fill[gray!20] (15,-1.7) rectangle (16,-1.5);
\fill[gray!65] (16,-1.7) rectangle (24,-1.5);
\node[below=1pt] at (20,-1.7) {final probe};
\end{tikzpicture}
\caption{(a) Apple tasting: action $1$ reveals the losses of both actions, and action $2$ reveals nothing. (b) Comparison probing, shown through the cumulative advantage $D_t$ of the blind action. The path is gray in resting rounds and black in probes, as marked on the strip below. A probe that starts at round $t_0$ continues while $D_t$ is below its starting level $D_{t_0-1}$ (dashed) and ends once the path is back at this level. The open circles mark the levels $D_{t-1}$ of the terminal rounds $t<t^\star$, which the path never reaches again (dotted); these rounds are resting rounds whose coins failed. The probe that starts at the terminal round $t^\star$ never ends.}
\label{fig:geometric-probing}
\end{figure}

To see why each missed chance costs at most one unit, compare the learner with action $1$, against which it loses the advantage $-\sum_{t\in\mathcal R}d_t$ hidden in the set $\mathcal R$ of resting rounds. Call a round $t$ \emph{terminal} if a probe started at $t$ would never end, that is, if $D_{t-1}>\overline D_t$, where $\overline D_t=\max_{u\ge t}D_u$, and let $t^\star$ be the first terminal round at which a probe starts, with $t^\star=T+1$ if there is none. All resting rounds precede $t^\star$, and each completed probe adds between zero and one to $D$. Since the maximum $\overline D_t$ of the remaining path starts at $\overline D_0\ge0$, equals $D_{t^\star-1}$ at $t=t^\star-1$, and decreases only at terminal rounds, by at most one each time, the hidden advantage $-\sum_{t\in\mathcal R}d_t$ is at most the number of probes plus the number $N^{\rm miss}$ of terminal rounds before $t^\star$. These rounds never lie inside completed probes, so each of them is a resting round whose coin toss failed, and $N^{\rm miss}$ is dominated by a geometric random variable with mean at most $1/\kappa$. \Cref{app:breakeven} states the algorithm and gives the complete proof.
\subsection{The 3-cycle: conflicting probing rules}
\label{sec:tech-obstacles}

The cyclic tournament on three actions already shows the difficulty beyond two actions: each pair is an instance of apple tasting but requires the learner to rest at a different action (\cref{fig:steering}(a)), so the three probing rules conflict while no action observes all losses. Our algorithm therefore spreads its time over the actions using a form of exponential weights run on unbiased loss estimates. These estimates are built on importance weighting, as in bandit algorithms \citep{auer2002nonstochastic}. Two problems stand in the way. First, these estimates depend on the trajectory, so they cannot drive a lazy version of exponential weights, such as Shrinking Dartboard \citep{GeulenVW10}, whose analysis requires updates independent of the current action. The alternative option of updating exponential weights only between blocks makes the second moments of the block estimates grow quadratically with the block length when loss differences persist, which, unless compensated, leads again to regret of order $T^{2/3}$. Second, the hard stopping rule of comparison probing, applied to every pair, would give probability zero to some trajectories and bias the estimates in a way difficult to correct. Our algorithm therefore runs, in each block of $\widetilde\Theta(\sqrt T)$ rounds, a slowly switching reversible Markov chain whose stationary law is close to the exponential weights of the block. Each transition is slightly tilted toward the actions that the current comparisons favor, an approach which we call \emph{steering}. This soft probing rule changes path probabilities only by a bounded factor, which importance weighting can correct exactly. Importantly, the loss saved by tilting pays for the second moment of the estimates.

\paragraph{Self-paying blocks.}
Next, we describe exponential weights run on loss estimates. Then, we show how their second moment enters a block-by-block bound on the regret. Adding the same scalar to every coordinate of a loss estimate does not change the exponential-weights update, so it suffices to estimate the centered losses $C_w\ell_t=\ell_t-(w^{\mathsf T}\ell_t)\one$, where $w$ is the current vector of smoothed exponential weights (i.e., mixed with the uniform distribution). Centered losses are defined on comparisons, and their estimates are exact when all actions have the same loss. The loss saved by steering, which pays for their second moment, comes only from loss differences. An \emph{assignment} distributes the comparisons among the actions that observe them. We use matrices $Q_s$ such that $Q_s\ell_t$ uses only comparisons observed at $s$ and $\sum_sQ_s=C_w$. In the 3-cycle, $Q_s$ collects the two terms of $C_w\ell_t$ that contain the comparison observed at $s$. If the action $X_t$ played in round $t$ has law $p$, the importance-weighted vector $C_wQ_{X_t}\ell_t/p_{X_t}$ is observable and has mean $\sum_sC_wQ_s\ell_t=C_w\ell_t$. Exponential weights are updated only between blocks.
We write $C_b$ for the cost of block $b$, including its switching costs, $L_b$ for the sum of its loss vectors, $w_b$ for its smoothed exponential weights, and $\widehat G_b$ for its estimate of $C_{w_b}L_b$, built from these vectors with $w=w_b$. The standard second-order analysis of exponential weights with learning rate $\eta$ gives, up to mixing and bias terms and when $\eta\|\widehat G_b\|_\infty\le1$,
\begin{align}
\label{eq:my-block-cost}
\Reg_T\le\frac{\log K}{\eta}+\sum_b\E\Bigl[C_b - \ip{w_b}{L_b} + 2\eta\|\widehat G_b\|_{w_b}^2\Bigr]
\end{align}
where $\|g\|_w^2=\sum_iw_ig_i^2$. For a lazy procedure that does not react to the losses within a block, as in standard mini-batching \citep{arora2012online}, the bracket is dominated by the second moment $\|\widehat G_b\|_{w_b}^2$, which yields the rate $T^{2/3}$. Our blocks instead react to the comparisons they observe: they play a lazy Markov chain
so that the loss saved keeps the bracket of the order of $N_b$.

\paragraph{Lazy chains and time scales.}
Within a block, steering is implemented through a lazy Markov chain $P_\tau$ that is reversible with respect to the sampling law $p$, chosen close to $w$. So a chain started from $p$ has law $p$ in every round, as the estimates require. Besides $p$, the chain depends on a rate $\tau$ and on a second distribution $x$, the \emph{refresh law}. Roughly, the transition matrix is
\begin{align}
\label{eq:refresh-law}
P_\tau(j\mid s)\approx\Bigl(1-\frac{\tau\lambda_s}2\Bigr)\one\{j=s\}+\frac{\tau\lambda_s}2\,x_j
\qquad\text{for}\qquad
\lambda_s=\frac{x_s}{p_s}
\end{align}
which is accurate when all $\tau\lambda_s$ are small. At $X_t=s$, the chain thus keeps its current action, except that with probability about $\tau\lambda_s/2$ it \emph{refreshes}, drawing the next action from the law $x$. Taking $x_s$ large relative to $p_s$ makes the chain leave $s$ quickly. The algorithm does this for the actions whose comparisons carry large importance weights, so that their contribution to the estimate errors does not persist. Between rounds, the chain also restarts from $p$ with probability $\tau$, cutting the block into \emph{chunks}.
The rate $\tau$ balances switching costs against the gains brought by steering: the chain switches about once every $1/\tau$ rounds, and steering, which acts only through these transitions, gains from comparisons that persist between them. Each block draws its rate at random from the dyadic rates between a base rate $\kappa=2^{-m}$ and $1$, so that steering gains on every time scale, while the mean rate $\mu=\E\tau$ is only of order $m\kappa$. Blocks have order $1/\kappa$ rounds.

\paragraph{Steering.}
Steering lets the observed comparisons tilt the transitions of the chain, avoiding the $T^{2/3}$ regret of na\"{\i}ve mini-batching. At the current action $s$, the learner observes $\Delta_t(s,j)$ for every edge $s\to j$, and assigns to a transition to $j$ a score $d_t(s,j)=\Delta_t(s,j)$ if $s\to j$ and zero if $j\to s$. A compensating score for staying makes the scores average to zero under the law $P_\tau(\cdot\mid s)$. If $S$ is the sum of the scores of the transitions made so far in the chunk, the next action is drawn from
\begin{align*}
P_\tau(j\mid s)\,\frac{1+\beta\bigl(S+d_t(s,j)\bigr)}{1+\beta S}
\end{align*}
with a small \emph{steering strength} $\beta>0$. \cref{fig:steering}(b) illustrates one transition. The factors of the transitions telescope along a trajectory, so on each chunk $J$ the steered law is $1+\beta S_J$ times the law
of the unsteered $P_\tau$-chain started from $p$, which we call the \emph{reference chain}; here $S_J$ is the final score sum of the chunk. In the rare event that $|S|$ grows to order $1/\beta$, the chunk stops scoring and continues with $P_\tau$, so the likelihood ratio $1+\beta S_J$ always lies within a factor of one.
For two actions, only the transitions from the revealing action are scored: a move to the blind action scores its advantage $d_t$, and a stay scores a negative multiple of $d_t$. Steering thus softens the break-even rule of comparison probing: within a chunk (and before scoring stops), a stay at the revealing action has its reference probability multiplied by $1-c\sum_ud_u$ for some $c>0$, where $\sum_ud_u$ is its break-even statistic.
\begin{figure}[t]
\centering
\begin{tikzpicture}[font=\footnotesize,
  act/.style={circle,draw,thick,minimum size=15pt,inner sep=0pt},
  mv/.style={-{Stealth[length=2.5pt 1.6,width=2.5pt 1.6]},shorten >=1pt},
  title/.style={font=\footnotesize\itshape}]
\begin{scope}[xshift=-4.5cm,yshift=-1.3cm,
  obs/.style={-{Stealth[length=4.5pt]},thick},
  lp/.style={-{Stealth[length=4pt]},thick,looseness=8,min distance=8pt}]
\node[title] at (0,1.75) {(a) 3-cycle};
\node[act,minimum size=12pt] (c1) at (90:0.8) {$1$};
\node[act,minimum size=12pt] (c2) at (210:0.8) {$2$};
\node[act,minimum size=12pt] (c3) at (330:0.8) {$3$};
\draw[obs] (c1) -- (c2);
\draw[obs] (c2) -- (c3);
\draw[obs] (c3) -- (c1);
\draw[lp] (c1) to[out=60,in=120] (c1);
\draw[lp] (c2) to[out=150,in=210] (c2);
\draw[lp] (c3) to[out=-30,in=30] (c3);
\node[anchor=west,inner sep=2pt] at (c1.east) {$\{3,1\}$};
\node[anchor=north,inner sep=2pt] at (c2.south) {$\{1,2\}$};
\node[anchor=north,inner sep=2pt] at (c3.south) {$\{2,3\}$};
\end{scope}
\begin{scope}
\node[title] at (0,1.85) {(b) reference};
\node[act] (s) at (0,0) {$s$};
\node[act] (j) at (1.9,0.8) {$j$};
\node[act] (jp) at (1.9,-0.8) {$j'$};
\node[act] (k) at (-1.9,0) {$r$};
\draw[mv,line width=1.3pt] (s) -- (j);
\draw[mv,line width=1.3pt] (s) -- (jp);
\draw[mv,line width=1.3pt] (s) -- (k);
\draw[mv,line width=3pt] (s) to[out=125,in=55,looseness=8,min distance=17pt] (s);
\end{scope}
\begin{scope}[xshift=5.6cm]
\node[title] at (0,1.85) {steered};
\node[act] (s) at (0,0) {$s$};
\node[act] (j) at (1.9,0.8) {$j$};
\node[act] (jp) at (1.9,-0.8) {$j'$};
\node[act] (k) at (-1.9,0) {$r$};
\draw[mv,line width=2.5pt] (s) -- (j);
\draw[mv,line width=0.6pt] (s) -- (jp);
\draw[mv,line width=1.3pt] (s) -- (k);
\draw[mv,line width=2.5pt] (s) to[out=125,in=55,looseness=8,min distance=17pt] (s);
\node[anchor=west] at (j.east) {$\ell_t(j)<\ell_t(s)$};
\node[anchor=west] at (jp.east) {$\ell_t(j')>\ell_t(s)$};
\node[anchor=north] at (k.south) {not scored};
\end{scope}
\begin{scope}[yshift=-3.0cm,xshift=-1.2cm]
\node[title] at (1.2,1.4) {(c) tournament scores};
\node[act] (i) at (0,0) {$i$};
\node[act] (jj) at (2.4,0) {$j$};
\draw[mv,line width=1.2pt] (i) to[bend left=35] node[above] {$\Delta_t(i,j)$} (jj);
\draw[mv,line width=1.2pt] (jj) to[bend left=35] node[below] {$0$} (i);
\end{scope}
\begin{scope}[yshift=-3.0cm,xshift=4.4cm]
\node[title] at (1.2,1.4) {full-information scores};
\node[act] (i) at (0,0) {$i$};
\node[act] (jj) at (2.4,0) {$j$};
\draw[mv,line width=1.2pt] (i) to[bend left=35] node[above] {$\tfrac12\Delta_t(i,j)$} (jj);
\draw[mv,line width=1.2pt] (jj) to[bend left=35] node[below] {$\tfrac12\Delta_t(j,i)$} (i);
\end{scope}
\end{tikzpicture}
\caption{(a) The 3-cycle, in which each action observes itself and its successor. Each action is labeled by the pair whose probing rule rests there, so the three pairs ask the learner to rest at different actions. (b) Steering at the current action $s$, which observes $j$ and $j'$ but not $r$; arrow widths indicate probabilities, with the changes exaggerated. The move to $j$, which had a smaller loss than $s$ in the current round, becomes more likely, the move to $j'$ less likely, the unscored move to $r$ is unchanged, and staying compensates. (c) The tournament rule scores only the move from the observing endpoint $i$, by the full amount, whereas full information allows half of the amount in each direction. A stationary reversible chain makes both moves equally often, so the two rules have the same average effect.}
\label{fig:steering}
\end{figure}
Steering makes the law of the trajectory depend on the losses, so the estimates must be corrected for it. Under the reference chain $P_\tau$, every round has law $p$, and the chunk sum $G_J=\sum_{t\in J}C_wQ_{X_t}\ell_t/p_{X_t}$ has mean $C_wL_J$, where $L_J=\sum_{t\in J}\ell_t$. Dividing by $1+\beta S_J$ (i.e., the likelihood ratio) transfers this to the steered chain: conditional on the rate and the chunk boundaries,
$
\E\left[\frac{G_J}{1+\beta S_J}\right]=\E_0[G_J]=C_wL_J
$,
where $\E_0$ denotes expectation under $P_\tau$. Since the ratio lies within a factor of one, the corrected estimates have second moments within a constant factor of those under the reference chain. The block estimate $\widehat G_b$ is the sum of the corrected chunk estimates (discarded in the low-probability event that the sum is too large). The correction affects only the estimate, so the loss reduction provided by the tilt is retained.

\paragraph{The steering gain.}
The loss reduction of the tilt, which must compensate for the second moment, is $\beta$ times a quadratic form in the losses of the block.
Because the path law of a chunk $J$ is $(1+\beta S_J)\Prb_0$, where $\Prb_0$ is reference $P_\tau$-chain, the expectation of the chunk's loss $\sum_{t \in J} \ell_t(X_t)$ equals its reference loss $p^{\mathsf T} L_J$ plus $\beta$ times the correlation between $S_J$ and $\sum_{t \in J} \ell_t(X_t)$, up to the low-probability stopping event. Now, the score of the transition after round $t$ has mean zero conditioned on the current action, so it is uncorrelated with the losses up to round $t$. By the Markov property, its correlation with the loss of a later round $t'$ is $\E_0\big[d_t(X_t,X_{t+1})\ell_{t'}(X_{t'})\big] = \sum_{i,j}p_iP_{ij}d_t(i,j)g_j$, where $P=P_\tau$ and $g=P^{t'-t-1}\ell_{t'}$ is the predicted loss.
Note that each pair is scored only at its observing endpoint, so only one of its two moves is steered. A stationary chain might make that move rarely: for example, on the 3-cycle of \cref{fig:steering}(a), a chain that moves from each action only to its predecessor is never steered. On the other hand, a reversible chain makes both moves equally often, which yields
\begin{align}
\label{eq:one-sided}
\sum_{i,j}p_iP_{ij}d_t(i,j)g_j=-\sum_{i\to j}p_iP_{ij}\bigl(\ell_t(i)-\ell_t(j)\bigr)(g_i-g_j)=-\ip{\ell_t}{(I-P)g}_p
\end{align}
where $\ip fg_p=\sum_ip_if_ig_i$. For the first equality, the zero mean of the scores at $i$ allows replacing $g_j$ by $g_j-g_i$, which removes the stays. The detailed balance equations $p_iP_{ij}=p_jP_{ji}$ make each term in the middle sum symmetric in $i$ and $j$, which gives the second equality using the identity $\ip f{(I-P)g}_p=\tfrac12\sum_{i,j}p_iP_{ij}(f_i-f_j)(g_i-g_j)$. Note that the right-hand side does not depend on the orientation of the tournament, so scoring each pair only at its observing endpoint has the same average effect as scoring both moves by half as much (\cref{fig:steering}(c)). Summing these correlations over chunks and averaging over the restarts and the rate shows that steering lowers the expected loss of a block by the steering strength $\beta$ times its \emph{steering gain}
\begin{align}
\label{eq:probing-gain}
\Gamma=\E_\tau\sum_{t<t'}(1-\tau)^{t'-t}\ip{\ell_t}{(I-P_\tau)P_\tau^{t'-t-1}\ell_{t'}}_p
\end{align}
a quadratic form in the losses of the block, in which $(1-\tau)^{t'-t}$ is the probability that no restart separates the two rounds. Each bracket in \cref{eq:my-block-cost} is therefore at most $2\eta\E\|\widehat G_b\|_{w_b}^2-\beta\Gamma+O(1+\mu h_b)$ where $h_b$ is the block length and $O(1+\mu h_b)$ are the expected switching costs.

\subsection{Gaining at every time scale}
\label{sec:tech-energy}

It remains to show that the steering gain controls the second moment of the block estimate: for a block of length $h_b$, we have that
\begin{align}
\label{eq:second-moment}
\E\|\widehat G_b\|_{w_b}^2 = O\left(\frac{mL_0}{\kappa}\bigl(\Gamma+\mu h_b\bigr)\right)
\end{align}
where $L_0$ measures how strongly the estimates amplify losses into errors. $L_0$ is a constant in the 3-cycle and logarithmic for every tournament, see \cref{eq:L_0-def}. With a learning rate $\eta$ of order $\beta\kappa/(mL_0)$, the steering gain $\Gamma$ then absorbs the second moment of every block up to its switching cost $O(1+\mu h_b)$. We first show on the 3-cycle why steering at a single switching rate cannot pay for the second moment, and then how a random rate fixes that.

\paragraph{A single rate is not enough.}
The difficulty lies in the errors of the estimates, which persist while the chain holds the action that produced them, possibly for much shorter times than those between the moves through which steering provides gains. Consider the 3-cycle $1\to2\to3\to1$ and let the weights be split evenly between actions $1$ and $2$. Now let action $3$ have a weight $w_3$ that is small but much larger than the base rate $\kappa$, see \cref{eq:refresh-law}. The centered loss of action $3$ contains $w_1(\ell_t(3)-\ell_t(1))$, and only action $3$ observes this comparison, so estimating this comparison requires the importance weight $1/p_3\approx 1/w_3$. To keep these large errors from persisting, the refresh law \cref{eq:refresh-law} makes the chain leave action $3$ after order $w_3/\tau$ rounds, whereas it holds actions $1$ and $2$ for order $1/\tau$ rounds. Let $\ell_t(3)$ be the average of $\ell_t(1)$ and $\ell_t(2)$, and let their comparison change sign about every $w_3/\kappa$ rounds. At the base rate $\kappa$, the errors at action $3$ then accumulate over whole stays and add order $h_b w_3/\kappa$ to the second moment. Steering, however, acts on the comparison of actions $1$ and $2$ only through moves between them, about once every $1/\kappa$ rounds, and the comparison cancels between such moves, so steering gains little. At the rate $\kappa/w_3$, the chain moves between actions $1$ and $2$ about as often as their comparison changes sign, and steering gains order $h_b$. Since a chain at rate $\tau$ incurs an expected switching cost of $\tau$ per round, we assign to each block a random rate drawn from the dyadic law that gives each rate $\tau\in\{\kappa,2\kappa,4\kappa,\ldots,\tfrac12\}$ the probability $\kappa/(2\tau)$ and the rate $1$ the probability $\kappa$. Each rate below $1$ then adds $\kappa/2$ to the mean rate $\mu=\E\tau$, and in the example, the rate $\kappa/w_3$ has probability $w_3/2$, so its expected gain, of order $w_3 h_b$, controls the errors’ contribution to the second moment up to a factor $O(1/\kappa)$.

\paragraph{Separating the time scales.}
For arbitrary losses, the eigenbasis shared by all the chains separates their time scales. Fix a block and drop its index $b$. The exact definition of the transition matrix \cref{eq:refresh-law} gives $2P_\tau-I=(I+\tau A)^{-1}$ for the matrix $A=\diag(\lambda)-\lambda x^{\mathsf T}$, which is self-adjoint for $\ip\cdot\cdot_p$. Hence, every chain $P_\tau$ has the $p$-orthonormal eigenvectors $\phi_n$ of $A$ on $p$-centered vectors, with eigenvalues $\alpha_n$ between $\tfrac12$ and $1/(4\kappa)$. At rate $\tau$, the chain and the restarts retain the $n$-th spectral component from one round to the next with probability $\chi_n$, and cut it with probability $1-\chi_n$, which is of order $\min\{1,\tau\alpha_n\}$. Let $\cP_n$ be the random partition of the block into \emph{windows} that cuts each boundary between rounds independently with probability $1-\chi_n$, so that two rounds $t$ and $t'$ share a window with probability $\chi_n^{|t-t'|}$. In the eigenbasis, if we write $[f]_n=\ip{\phi_n}f_p$ the steering gain \cref{eq:probing-gain} then becomes
\begin{align}
\label{eq:energy-windows}
2\Gamma=\cQ_*-\cQ_{\rm diag}
\qquad\text{where}\qquad
\cQ_*=\E_\tau\sum_n\psi_n\,\E\sum_{J\in\cP_n}\Bigl(\sum_{t\in J}[\ell_t]_n\Bigr)^2
\end{align}
where $\psi_n=\tau\alpha_n/(2+\tau\alpha_n)$ is of the order of the cut probability, and the diagonal term $\cQ_{\rm diag}=\E_\tau\sum_n\psi_n\sum_t[\ell_t]_n^2$ lies between $0$ and $\mu h/8$. We call $\cQ_*$ the \emph{steering energy}. Since a move toward the currently better action saves loss until the next cut, steering provides gains from comparisons that persist over windows of about $1/\min\{1,\tau\alpha_n\}$ rounds, while losing at most $\mu h/16$ on comparisons that change sign within windows. Moreover, since $\psi_n$ is about $\tau\alpha_n/2$ when $\tau\alpha_n\le1$, the dyadic law gives every such rate the same weight $\kappa\alpha_n/4$ in the energy. The second moment has the same form: conditional on the rate and the restarts, the chunks are independent and their likelihood ratios lie within a factor of one, so the second moment is of the order of the sum of the squared block mean, the squared chunk means, and the second moment of the errors under the reference chain,
\begin{align}
\label{eq:three-parts}
\|C_wL\|_w^2+\E\sum_J\|C_wL_J\|_w^2+V_{\rm err}
\qquad\text{where}\qquad
V_{\rm err}=\E_\tau\sum_n\E\sum_{J\in\cP_n}\Bigl\|\sum_{t\in J}[E\ell_t]_n\Bigr\|_w^2
\end{align}
and $\E_\tau$ is over the block's random rate $\tau$. Here $E_s=C_wQ_s/p_s-C_w$ is the error of source $s$, and $[Ef]_n=\sum_sp_s\phi_n(s)E_sf\in\R^K$ is the $n$-th spectral component of the errors of $f$ as a function of the source, which the reference chain retains with the same probability $\chi_n$. For the means, the Cauchy-Schwarz inequality over the windows shows that each coordinate of the energy is at least of order $\kappa$ times its squared block total: its weight $\psi_n$ is of the order of the cut probability $1-\chi_n\ge\kappa$, and the block has about $1+(1-\chi_n)h$ windows, with $h=O(1/\kappa)$. The restarts cut less often, so each chunk can be realized as a union of windows, and the same inequality gives
\begin{align}
\label{eq:means}
\|C_wL\|_w^2+\E\sum_J\|C_wL_J\|_w^2 = O\left(\frac{\cQ_*}{\kappa}\right).
\end{align}

\paragraph{Combining the time scales.}
As shown by the 3-cycle example, a loss component in one spectral mode can generate estimation errors in modes with different correlation times. Two complementary bounds, with common constant $L_0$, control these errors over an entire stay at one action and over a single round, respectively. For every $f\in\R^K$,
\begin{align}
\label{eq:endpoint-bounds}
\sum_s\frac{p_s}{\lambda_s}\|E_sf\|_w^2\le L_0\Var_p(f)
\qquad\text{and}\qquad
\sum_sp_s\|E_sf\|_w^2\le L_0\Var_x(f)
\end{align}
where $\Var_p$ and $\Var_x$ denote variances under $p$ and $x$. In the eigenbasis, the left-hand sides are $\sum_n\|[Ef]_n\|_w^2/\alpha_n$ and $\sum_n\|[Ef]_n\|_w^2$, and the right-hand sides are $L_0\sum_n[f]_n^2$ and $L_0\sum_n\alpha_n[f]_n^2$.
For a loss mode with eigenvalue $\alpha$, the first bound controls errors in slower spectral modes, whose correlations persist longer, using steering energy at the same rate $\tau$. The second controls errors in faster modes using energy at a rate comparable to $\tau\alpha_n/\alpha$, as in the 3-cycle example. The dyadic probabilities balance the corresponding energy weights, so changing rates incurs only a constant factor.
Each of the $m$ dyadic bands $[2^k,2^{k+1})$ of eigenvalues thus costs $O(L_0/\kappa)$, and combining the bands by the Cauchy--Schwarz inequality costs the factor $m$,
\begin{align}
\label{eq:errors}
V_{\rm err} = O\left(m\frac{\cQ_* L_0}\kappa\right).
\end{align}
Combining these bounds with $\cQ_* = 2\Gamma + \cQ_{\rm diag} \le 2\Gamma + \mu h/8$ gives the second-moment bound $O\big((mL_0/\kappa)(\Gamma+\mu h)\big)$, as claimed in \cref{eq:second-moment}.

\subsection{Scaling to \texorpdfstring{$K$}{K} actions}
\label{sec:tech-scaling}

The construction so far applies to every tournament, and up to logarithmic factors the number of actions enters only through the constant $L_0$ of the estimates. In the 3-cycle, the only action that observes a comparison is the observing endpoint of its pair. Scaling this solution naively, by comparing every pair at its observing endpoint, estimates the loss of a typical action with constant weight from comparisons observed only while playing it. With uniform weights, that action is played only a $1/K$ fraction of the time, so the estimates behave as with bandit feedback: their second moment, and with it $L_0$, grows linearly in $K$. Tournaments allow much better estimates, since a few actions observe all the others, which a weighted greedy construction exploits to make $L_0$ \emph{logarithmic}.

\begin{figure}[t]
\centering
\begin{tikzpicture}[font=\footnotesize,
  act/.style={circle,draw,thick,minimum size=11pt,inner sep=0pt,fill=white},
  piv/.style={act,minimum size=14pt},
  mem/.style={circle,draw,thick,minimum size=7pt,inner sep=0pt},
  obs/.style={-{Stealth[length=4.5pt]},thick},
  oth/.style={-{Stealth[length=3pt]},gray!60},
  band/.style={line width=19pt,gray!20,line cap=round},
  cell/.style={draw,rounded corners=4pt,thin,gray},
  title/.style={font=\footnotesize\itshape}]
\begin{scope}
\def\R{1.3}
\node[title,align=right,anchor=east] at (-1.75,0) {(a) a tournament\\on seven actions};
\draw[band] (120:\R) arc[start angle=120,end angle=270,radius=\R];
\draw[band] (320:\R) arc[start angle=320,end angle=370,radius=\R];
\fill[gray!20] (60:\R) circle (9.5pt);
\node[piv,fill=gray!45] (A) at (120:\R) {$s_1$};
\node[act] (B) at (170:\R) {};
\node[act] (C) at (220:\R) {};
\node[act] (D) at (270:\R) {};
\node[act] (F) at (320:\R) {};
\node[piv] (E) at (10:\R) {$s_2$};
\node[piv] (G) at (60:\R) {$s_3$};
\foreach \u/\w in {B/C,C/D,D/B,B/E,E/C,D/E,F/B,C/F,F/D,G/B,C/G,G/D,F/G,F/A,G/E} \draw[oth] (\u) -- (\w);
\foreach \w in {B,C,D} \draw[obs] (A) -- (\w);
\draw[obs] (E) -- (F);
\draw[obs,line width=1.3pt] (E) -- (A);
\draw[obs,line width=1.3pt] (G) -- (A);
\end{scope}
\begin{scope}[xshift=3.2cm,yshift=-0.35cm]
\node[title] at (2.53,1.45) {(b) the assignment};
\draw[cell] (-0.05,-0.32) rectangle (2.2,0.32);
\draw[cell] (2.65,-0.32) rectangle (3.9,0.32);
\draw[cell] (4.35,-0.32) rectangle (5.1,0.32);
\node[act,minimum size=14pt,fill=gray!45] (p1) at (0.3,0) {$s_1$};
\foreach \x in {0.95,1.45,1.95} \node[mem] at (\x,0) {};
\node[act,minimum size=14pt] (p2) at (3.0,0) {$s_2$};
\node[mem] at (3.65,0) {};
\node[act,minimum size=14pt] (p3) at (4.72,0) {$s_3$};
\draw[obs,line width=1.3pt] (p2) to[bend right=32] (p1);
\draw[obs,line width=1.3pt] (p3) to[bend right=38] (p1);
\node at (1.07,-0.62) {$\tfrac47$};
\node at (3.27,-0.62) {$\tfrac27$};
\node at (4.72,-0.62) {$\tfrac17$};
\end{scope}
\end{tikzpicture}
\caption{The logarithmic assignment, for uniform weights. (a) One run of the pivot construction on a tournament with seven actions, with self-loops omitted. The first pivot $s_1$ removes the actions that it observes, which form its cell (shaded); the next pivot $s_2$ is drawn from the remaining actions, which all observe $s_1$, and so on. The black edges carry the assignment: each pivot observes the actions of its cell and $s_1$, so it compares them with $s_1$. (b) The resulting cells and their weights. In expectation, each cell holds at least half of the remaining weight, so the expected number of pivots, which is the total source weight, is logarithmic in $1/w_{\min}$.}
\label{fig:assignment}
\end{figure}

\paragraph{A logarithmic assignment.}
Recall that entry $(Q_s)_{ii}$ of the matrix $Q_s$ is the \emph{load} of source $s$ on action $i$, the total weight that $s$ contributes to it. An assignment comes with \emph{source weights} $w\le\rho\le\one$ that bound the loads: $(Q_s)_{ii}\le\rho_s$ and $w_i(Q_s)_{ii}\le w_s\rho_i$. The algorithm samples from $p=(1-4\kappa)w+4\kappa q$ and refreshes from $x=(p+q)/2$, where $q=\rho/\sum_j\rho_j$, so that heavily loaded sources are visited often enough and left quickly. With these laws, $(Q_s)_{ii}\le\rho_s$ controls the errors that persist while the chain holds a source, and $w_i(Q_s)_{ii}\le w_s\rho_i$ those of a single round, so the endpoint bounds \cref{eq:endpoint-bounds} hold with $L_0$ of the order of the total source weight $\sum_s\rho_s$. The algorithm solves a linear program for an assignment of minimum total source weight, which is logarithmic for every tournament.
To bound the minimum, we average a randomized, weighted version of the greedy construction of logarithmic dominating sets \citep{megiddo1988dominating}: draw a pivot from the remaining actions with probability proportional to $w$, remove it along with the actions it observes, and repeat (\cref{fig:assignment}). The tournament property enters twice. First, the survivors of a pivot $s$ in the remaining set $\cU$ are precisely the $j\in\cU$ with $j\to s$. Since each pair is oriented once, writing $w(\cU)=\sum_{i\in\cU}w_i$ gives
\begin{align*}
\E\bigl[w(\cU_{\rm next})\bigm|\cU\bigr]
=\frac1{w(\cU)}\sum_{s\in\cU}w_s
  \sum_{j\in\cU:\,j\to s}w_j
=\frac{w(\cU)^2-\sum_{i\in\cU}w_i^2}{2w(\cU)}
\le\frac{w(\cU)}2
\end{align*}
so the expected number of pivots is logarithmic in $1/w_{\min}$. Second, every later pivot observes the first, which would otherwise have removed it, so each pivot can compare its removed actions with this common reference, drawn from $w$. Averaging over the randomness of the construction gives a feasible point of the linear program, whose total source weight is the expected number of pivots. Since smoothing ensures $w_{\min}\ge1/(KT)$, the endpoint bounds \cref{eq:endpoint-bounds} therefore hold with
\begin{align}
\label{eq:L_0-def}
L_0=O\left(\sum_s\rho_s\right) = O\left(\log\frac e{w_{\min}}\right) = O\big(\log(KT)\big).
\end{align}

\paragraph{The dependence on \texorpdfstring{$K$}{K}.}
With $\eta$ of order $\beta\kappa/(mL_0)$, the bracket of every block in \cref{eq:my-block-cost} is $O(1+\mu h_b)$. Summing over the $O(\kappa T)$ blocks and adding the term $(\log K)/\eta$ from the exponential weights analysis bounds the regret by $O(mL_0(\log K)/(\beta\kappa)+\mu T)$, where $\mu=O(m\kappa)$. Balancing the two terms in $\kappa$ gives
\begin{align*}
\Reg_T = O\left(m\sqrt{\frac{L_0}\beta\,T\log K}\right) = O\Bigl(\sqrt{T\log K}\,\log^2(KT)\Bigr).
\end{align*}
The second bound holds because each of the factors $m$, $L_0$ and $1/\beta$ is $O(\log(KT))$, the last because accumulated scores and importance weights have exponential tails, so that thresholds of order $\log(KT)$ already make stopping and discarding polynomially rare. The tournament is visible in $L_0$: the naive assignment, with $L_0$ of order $K$, would only give regret of order $\sqrt{KT}$ up to logarithmic factors, as for bandits without switching costs.

\subsection{Lower bounds and the remaining regimes}
\label{sec:tech-lower}

The lower bounds show that the condition of the upper bound cannot be weakened: a single pair of actions without an observing endpoint forces regret of order $T^{2/3}/\log T$, whatever the other actions observe, and beyond the condition the domination number $\gamma$ governs the rate. The construction of Dekel, Ding, Koren, and Peres \citep{DekelDingKorenPeres2014} for bandits gives all actions one shared multiscale random walk, so that only switches reveal which action is better. However, a third action that observes both actions of the pair would reveal their loss difference in a single round. Our construction keeps loss differences noisy even when they are observed together, by adding to the walk a coordinate for each member of a set $\mathcal S$ of actions, such as the pair, in which every member observes at most one member. Information about which member is better then comes only from switches and from rounds at actions outside $\mathcal S$, which have loss one, and if every action observes at most $\mathsf m$ members, balancing these costs against the cost of not identifying the better member gives regret of order $(|\mathcal S|/\mathsf m)^{1/3}T^{2/3}/\log T$. When $\gamma$ is large, thinning and random sampling, as in \citep{AlonCDK15}, find such a set with $|\mathcal S|$ of order $\gamma/\log K$ and $\mathsf m=O(\log K)$, and lazy exploration on $\gamma$ actions, which together observe all losses, matches the resulting bound up to logarithmic factors. \Cref{sec:regimes} gives the details.\label{pg:end-overview}

\section{Setting and main result}
\label{sec:setting-result}

\subsection{Problem setting and notation}
\label{sec:setting}

Let $\mathcal G=([K],E)$ be a directed graph on the actions $[K]=\{1,\ldots,K\}$, known before play, which may contain self-loops and opposite directed edges. Hidden from the learner, an oblivious adversary chooses deterministic losses $\ell_1,\ldots,\ell_T\in[0,1]^K$. At round $t$, the learner selects $I_t$ based on its past observations and randomization, incurs $\ell_t(I_t)$, and observes $\ell_t(j)$ for every $(I_t,j)\in E$. Note that the incurred loss $\ell_t(I_t)$ is \emph{only} observed if the action played $I_t$ admits a self-loop.
We define the number of switches and the expected regret by, respectively,
\begin{align}
N_T=\sum_{t=2}^T\one\{I_t\ne I_{t-1}\}
\qquad\text{and}\qquad
\Reg_T=\E\left[\sum_{t=1}^T\ell_t(I_t)+N_T\right]
-\min_{i\in[K]}\sum_{t=1}^T\ell_t(i).
\label{eq:regret}
\end{align}
We remark that given the oblivious nature of the adversary, the expectation in \cref{eq:regret} is only with respect to the learner's internal randomization. For simplicity, we assume that the time horizon $T$ is known, and define the minimax regret $\Reg^\star_T(\mathcal G)$ as the infimum over learners of the supremum over loss sequences of $\Reg_T$. For actions $i,j$ and a round $t$, we write
\begin{align*}
\Delta_t(i,j)=\ell_t(i)-\ell_t(j)
\end{align*}
for the advantage of $j$ over $i$. This is the extra loss of playing $i$ rather than $j$ in round $t$ and provides a \emph{comparison} for the pair.

A probability vector has nonnegative entries summing to one, and we identify it with a distribution on $[K]$. Let $e_i$ be the $i$-th coordinate vector, and let $\one$ denote the all-ones column vector. All vectors are columns, and $I$ is the identity matrix of the required size. For a positive probability vector $v$, weighted sums are abbreviated as $\ip f g_v=\sum_iv_if_ig_i$ and $\|f\|_v^2=\ip ff_v$, while unsubscripted brackets $\ip\cdot\cdot$ denote the ordinary dot product. We subtract a weighted mean using the \emph{centering matrix} $C_v=I-\one v^{\mathsf T}$, so $C_vf=f-(v^{\mathsf T} f)\one$, and write
\begin{align*}
\Var_v(f)=\|C_vf\|_v^2=\min_{c\in\R}\sum_iv_i(f_i-c)^2.
\end{align*}
If $v\le c\,v'$ coordinatewise for positive probability vectors $v,v'$ and a constant $c$, then $\Var_v(f)\le c\Var_{v'}(f)$; to see this, evaluate the minimum defining $\Var_v(f)$ at the $v'$-weighted mean of $f$. Finally, all logarithms are natural unless a base is displayed.

\subsection{Observability and strongly revealing graphs}
\label{sec:revealing}

Following Alon, Cesa-Bianchi, Dekel, and Koren \citep{AlonCDK15}, an action $i$ is observable if $(j,i)\in E$ for some action $j$ (possibly $i$ itself) and strongly observable if it has a self-loop or $(j,i)\in E$ for every $j\ne i$. The graph is observable if every action is observable, strongly observable if every action is strongly observable, and weakly observable if it is observable but not strongly observable. A set $\mathcal D$ of actions is observing if every action is observed by some member of $\mathcal D$, that is, if for every $j\in[K]$ there is $s\in\mathcal D$ with $(s,j)\in E$. The graph is observable exactly when it has an observing set, and its domination number $\gamma(\mathcal G)$ is then the smallest size of an observing set. For instance, bandit feedback has domination number $K$, and a graph with a revealing action that observes all losses has domination number one. Our characterization shows that a crucial role is played by the following notion of strongly revealing graph.
\begin{definition}[Strongly revealing graph]
\label{def:revealing}
A graph $\mathcal G=([K],E)$ is strongly revealing if every pair of distinct actions has an endpoint that observes the losses of both:
\begin{align}
\text{for every }i\ne j,\quad
&[(i,i)\in E\text{ and }(i,j)\in E]
\quad\text{or}\quad
[(j,j)\in E\text{ and }(j,i)\in E].
\label{eq:pairwise-condition}
\end{align}
\end{definition}

Full information is strongly revealing, and so is apple tasting, which is the graph on two actions in which one action observes both losses and the other observes nothing \citep{HelmboldLL2000}. Every tournament with self-loops, that is, every graph with all self-loops and exactly one directed edge between each pair of distinct actions is also strongly revealing. There are two types of pair which may violate \cref{eq:pairwise-condition} in a strongly observable graph: a bandit pair (two distinct actions with self-loops and no edge between them) and a \emph{loopless pair\footnote{This pattern is also called \emph{cops-and-robbers}.}} (two distinct actions with no self-loops where each action observes the other one). The next lemma shows that these are the only possibilities.

\begin{lemma}[Structure of strongly revealing graphs]
\label{lem:structure}
The following two properties hold:

\begin{enumerate}
\item Every strongly revealing graph is strongly observable.
\item A strongly observable graph is strongly revealing if and only if it contains neither a bandit pair nor a loopless pair.
\end{enumerate}
\end{lemma}

\begin{proof}
For (i), let $i$ be an action without a self-loop. For every $j\ne i$, the pair $\{i,j\}$ has an endpoint that observes both losses; this endpoint cannot be $i$, so $j$ observes $i$. Hence $i$ is observed by every other action.

For (ii), a strongly revealing graph contains no bandit pair, because an endpoint that observes both losses is joined to the other endpoint by an edge, and no loopless pair, because such an endpoint has a self-loop. Conversely, let the graph be strongly observable without such pairs, and let $i\ne j$. If both have self-loops, they are joined by an edge, and its source observes both losses. If exactly one of them, say $j$, lacks a self-loop, then $j$ is observed by every other action, so $i$ observes both losses. If neither has a self-loop, both are observed by every other action, so they form a loopless pair, which is excluded.
\end{proof}

A tournament is a feedback graph with exactly one directed edge between each distinct pair of vertices, and we write $i\to j$ when this edge points from $i$ to $j$; we specify its self-loops separately. A spanning strongly observable tournament is a subgraph on all $K$ vertices that is both a tournament and strongly observable, using only its retained edges and self-loops. The next proposition shows that strongly revealing graphs are exactly the graphs that contain such a tournament.
\begin{proposition}[Reduction to observable comparisons]
\label{prop:graph}
The pairwise condition in \cref{eq:pairwise-condition} is equivalent to containing a spanning strongly observable tournament. When it holds, one can select a tournament orientation in $O(K^2)$ graph operations such that playing $s$ reveals $\Delta_t(s,j)$ for every edge $s\to j$.
\end{proposition}
\begin{proof}
We construct a tournament from the pairwise condition and then verify the converse edge by edge. Under \cref{eq:pairwise-condition}, orient each pair in any direction whose source observes both endpoint losses, and retain the available self-loops. A vertex without a self-loop cannot be the source of a retained edge, so it receives an edge from every other vertex. The selected tournament is therefore strongly observable. Its edge sources observe both losses and hence their differences. Selecting the directions inspects at most a constant number of graph entries per pair.

Conversely, consider a spanning strongly observable tournament contained in $\mathcal G$. If $i\to j$ is one of its edges and $i$ lacked a retained self-loop, strong observability at $i$ would require $j\to i$, contradicting the tournament property. Thus the source of each retained edge has a self-loop, giving the corresponding alternative in \cref{eq:pairwise-condition}. Finally, the self-difference is always zero, whether or not it is observed. This proves the equivalence and the comparison assertion.
\end{proof}
Whenever a strongly revealing graph is considered below, we fix a selected orientation and ignore additional feedback. Define the closed combinatorial neighborhood $\overline N^+(s)=\{s\}\cup\{j:s\to j\}$. Including $s$ in this set does not assert that its raw loss is necessarily observed. Playing $s$ reveals every difference $\ell_t(i)-\ell_t(r)$ with $i,r\in\overline N^+(s)$, since it equals $\Delta_t(s,r)-\Delta_t(s,i)$, where $\Delta_t(s,s)=0$ is known. Throughout, pairwise comparisons mean these loss differences, which are the only feedback used by the algorithms for strongly revealing graphs. We call a procedure causal if each of its actions depends only on past observations and private randomness, and observable if the only feedback it uses consists of these comparisons.

\subsection{Main result}
\label{sec:guarantee}

The following restatement of \cref{thm:intro-char} gives explicit bounds on the minimax regret with switching costs of every feedback graph.
\begin{restatedtheorem}[Characterization, restated]
\label{thm:characterization}
Let $\mathcal G$ be a directed graph on $K\ge2$ actions, and let $T\ge2$.
\begin{enumerate}
\item If $\mathcal G$ is strongly revealing, then
\begin{align*}
\tfrac14\sqrt T\le\Reg^\star_T(\mathcal G)\le2^{40}\sqrt{T\log K}\,\log^2(eKT).
\end{align*}
\item If $\mathcal G$ is observable but not strongly revealing, and $\gamma=\gamma(\mathcal G)$, then
\begin{align*}
\frac{T^{2/3}}{1024\log_2(2T)}\max\Bigl\{1,\frac{\min\{\gamma,T\}}{\log^2(eK)}\Bigr\}^{1/3}\le\Reg^\star_T(\mathcal G)\le\min\bigl\{5\,(\gamma\log K)^{1/3}\,T^{2/3},\,T\bigr\}.
\end{align*}
\item If $\mathcal G$ is not observable, then $\Reg^\star_T(\mathcal G)\ge T/4$.
\end{enumerate}
\end{restatedtheorem}
Without switching costs, every strongly observable graph has minimax regret of order $\sqrt T$, up to logarithmic factors and graph-dependent constants \citep{AlonCDK15}. With switching costs, the strongly observable graphs split by the pairwise condition: by \cref{lem:structure}, the rate increases from $\sqrt T$ to $T^{2/3}$ exactly for the strongly observable graphs with a bandit pair or a loopless pair. For strongly revealing graphs, the upper bound depends on the number of actions only through $\log K$ and logarithmic factors, and prediction with expert advice gives a lower bound of order $\min\{T,\sqrt{T\log K}\}$, so the factor $\sqrt{T\log K}$ cannot be improved in the non-saturated regime $\log K\le T$ \citep[Chapter~3]{cesa2006prediction}. In the middle regime, the domination number plays the role that the number of actions plays for bandits with switching costs, whose rate is of order $\min\{T,K^{1/3}T^{2/3}\}$ up to logarithmic factors: for every $T\ge2$, the two bounds in (ii) differ by at most a factor $5120\log_2(2T)\log(eK)$. The term $1$ in the lower bound comes from a single pair of actions without an observing endpoint, and the other term from a large set of actions that every action observes only a few at a time (\cref{sec:lower}).

The upper bound in (i) is \cref{thm:main}, stated in \cref{sec:algorithm} and proved in \cref{sec:analysis}. For two actions, \cref{app:breakeven} gives a simpler algorithm with regret $O(\sqrt T)$. The other bounds are proved in \cref{sec:regimes}.

\subsection{Background}
\label{sec:background}

\paragraph{Exponential weights.}
Exponential weights maintains a distribution over the actions where the probability of an action is proportional to a negative exponential of its estimated cumulative loss. Its standard second-order analysis bounds the regret against every action by a term $\log K/\eta$ plus $\eta$ times the second moment of the estimates under the current weights. In bandit problems, this second moment is the price of exploration; our algorithm for strongly revealing graphs is designed so that steering within each block pays for it. We use the following signed-vector version, which places no sign condition on the estimates \citep[Section~11.4]{lattimore2020bandit}. Note that the inequality holds pathwise, including for adaptively generated vectors.

\begin{lemma}[Second-order exponential weights]
\label{lem:ew}
Let $g_1,\ldots,g_B\in\R^K$, let $\eta>0$, and suppose $\eta\|g_b\|_\infty\le1$ for every $b$. Starting from $\pi_1=\one/K$, update
$\pi_{b+1,i}=\pi_{b,i}e^{-\eta g_{b,i}}/\sum_j\pi_{b,j}e^{-\eta g_{b,j}}$.
Then every action $i$ satisfies
\begin{align*}
\sum_{b=1}^B\ip{\pi_b-e_i}{g_b}
\le\frac{\log K}{\eta}+\eta\sum_{b=1}^B\|g_b\|_{\pi_b}^2.
\end{align*}
\end{lemma}
\begin{proof}
We compare the update normalizers with the weight of the comparison action. Applying the well known exponential inequality $e^{-z}\le1-z+z^2$ for $|z|\le1$, we have that $\mathcal Z_b=\sum_j\pi_{b,j}e^{-\eta g_{b,j}}$ satisfies
\begin{align*}
\log\mathcal Z_b\le-\eta\ip{\pi_b}{g_b}+\eta^2\|g_b\|_{\pi_b}^2,
\end{align*}
where we used $\log(1+y)\le y$ for $y>-1$. The update gives
\begin{align*}
0\ge\log\pi_{B+1,i}=-\log K-\eta\sum_{b=1}^Bg_{b,i}-\sum_{b=1}^B\log\mathcal Z_b.
\end{align*}
Combining the two displays and dividing by $\eta$ proves the claim.
\end{proof}

\paragraph{Markov chains.}
Our main algorithm uses Markov chains on the actions as reference processes, and we recall here some basic terminology \citep[Chapter~1]{levin2017markov}. A transition matrix $P$ on $[K]$ has nonnegative entries and rows summing to one. We write $P(j\mid i)=P_{ij}$ for the probability of moving from $i$ to $j$, and we call $P(i\mid i)$ the holding probability of $i$. A $P$-chain is a random sequence $X_1,X_2,\ldots$ in $[K]$ in which, given the past $X_1,\ldots,X_t$, the next state $X_{t+1}$ is drawn from $P(\cdot\mid X_t)$ (Markov property). The chain switches between rounds $t$ and $t+1$ if $X_{t+1}\ne X_t$. If $X_t$ has law $v$, then $X_{t+1}$ has law $P^{\mathsf T}v$. A generic function $g:[K]\to\R$ can be represented by a vector in $\R^K$ with entries $g(i)=g_i$, so that $(Pg)(i)=\sum_jP_{ij}g_j$ is the expected value of $g$ after one step from $i$. Iterating the Markov property gives $\E[g(X_{t'})\mid X_1,\ldots,X_t]=(P^{t'-t}g)(X_t)$ for $t\le t'$. A probability vector $p$ is stationary if $p^{\mathsf T} P=p^{\mathsf T}$. A chain started from a stationary $p$, which we call a stationary chain, has law $p$ in every round, but consecutive rounds are correlated. Two consecutive states of a stationary chain form a stationary transition $(X,X')$, with $X\sim p$ and $X'\sim P(\cdot\mid X)$. In the algorithm of \cref{sec:algorithm}, a stationary reversible chain serves as a reference: its law, the reference law, fixes the baseline loss $p^{\mathsf T}\ell_t$ and the mean of the comparison estimate. The action process actually played is tilted away from the reference law, and a likelihood correction preserves the estimate's mean (\cref{sec:steering}).

\paragraph{Reversible chains.}
The matrix $P$ is reversible with respect to a positive probability vector $p$ if the detailed-balance identity $p_iP_{ij}=p_jP_{ji}$ holds for all $i,j$. In words, a chain started from $p$ is as likely to step from $i$ to $j$ as from $j$ to $i$, so its law looks the same forward and backward in time. Summing the identity over $i$ shows that $p$ is stationary. All reference chains in this paper are reversible, and we use three consequences of detailed balance \citep[Chapters~12 and~13]{levin2017markov}; see also Aldous and Fill \citep[Chapter~3]{aldous2002reversible}.

First, for all real $f$ and $g$,
\begin{align}
\ip f{(I-P)g}_p=\frac12\sum_{i,j}p_iP_{ij}(f_i-f_j)(g_i-g_j).
\label{eq:dirichlet}
\end{align}
To verify it, expand the right side: by the unit row sums and the stationarity of $p$, the terms $f_ig_i$ and $f_jg_j$ each contribute $\tfrac12\ip fg_p$, and by detailed balance the two cross terms each contribute $-\tfrac12\ip f{Pg}_p$. With $g=f$, the left side is called the Dirichlet form of $P$. It measures how much one step of the chain changes the value of $f$, and for $g\ne f$ it measures how the changes of $f$ and $g$ along one step are correlated. The right side treats the two directions of each pair symmetrically, so comparisons observed from only one endpoint of each pair can reproduce it; this is how a tournament enters steering (\cref{lem:tournament-scores}).

Second, detailed balance says that $\diag(p)P$ is symmetric, that is, $\ip f{Pg}_p=\ip{Pf}g_p$ for all $f,g$. Hence $P$ has real eigenvalues and a basis of eigenvectors that is orthonormal for $\ip\cdot\cdot_p$, one of which is the constant vector $\one$, with eigenvalue one. In this basis, $P^k$ multiplies each coordinate by the $k$-th power of its eigenvalue.

Third, by the Markov property and stationarity, a chain started from $p$ has correlations $\E[f(X_t)g(X_{t'})]=\ip f{P^{t'-t}g}_p$ for $t\le t'$. In an eigenbasis, these correlations are geometric in the distance $t'-t$, one coordinate at a time. The chains of \cref{sec:holding} share one eigenbasis, and the analysis in \cref{sec:analysis} rests on this.

\section{The algorithm for strongly revealing graphs}
\label{sec:algorithm}

In this section, we present the algorithm \proc{TournamentSteering} for strongly revealing graphs and state the resulting guarantee, which is proved in \Cref{sec:analysis}. Starting from a strongly revealing graph, we restrict our attention to the spanning strongly observable tournament selected in \cref{prop:graph}, so that any edge $s\to j$ of the selected tournament defines a two-action problem of the kind treated in \cref{sec:tech-two,app:breakeven}, with $s$ revealing and $j$ blind, and a single trajectory of play must consider all of these problems at once. With more than two actions, the learner must also spread its time over actions as exponential weights prescribes, and exponential weights need estimates of the losses of the actions on which the learner rests. These actions need not observe one another, so the estimates must be built from the comparisons observed along the learner's own path, and the hard stopping decisions of the break-even rule used by \proc{ComparisonProbing} (see \cref{app:breakeven}) would bias them. The algorithm therefore keeps the two ingredients of \proc{ComparisonProbing}, lazy moves between actions and continuation decisions driven by the observed advantages, but makes the second one soft. The paths of a slowly switching Markov chain are reweighted by a bounded factor, so that the same comparisons can be corrected into unbiased estimates.

The algorithm uses each observed comparison twice: to steer the next move toward actions with smaller losses, and to estimate losses for the next exponential-weights update. It updates exponential weights once per block of order $1/\kappa$ rounds, where the base rate $\kappa=2^{-m}$, a power of two with an integer $m\ge5$, also sets the switching rate of the chains played within blocks, up to a logarithmic factor. We first describe the main loop, then the holding chains, the comparison estimates, the probing procedure \proc{ProbeChunk}, and the preparation of its inputs (\proc{PrepareBlock}), and we end with the parameters and the guarantee. The pseudocode of our algorithm, \proc{TournamentSteering}, is presented in \cref{alg:main}.

\subsection{The main loop}
\label{sec:main-algorithm}

    The main loop of \proc{TournamentSteering} partitions the horizon into blocks of length $H$ allowing a shorter final block (for this reason we use $h_b$ to refer to the generic length of block $b$), and maintains an exponential-weights distribution $\pi_b$ at \emph{block level}, so that it gets updated $\approx T/H$ times. At the start of each block $b$, our algorithm receives $\pi_b$ and smooths it to $w=w_b$ by adding a uniform exploration term. Then, it calls the deterministic routine \proc{PrepareBlock} described in \cref{sec:preparation}, which uses $w$ to compute a sampling law $p$ close to $w$, a refresh law $x$, and estimation matrices $\Psi_s$, $s\in[K]$. These estimation matrices are such that each product $\Psi_s\ell_t$ is computable from the comparisons observed when $s$ is played in round $t$, $\Psi_s\one=0$, and $\sum_sp_s\Psi_s=C_w$, where we recall that $C_w$ is the \emph{centering matrix} $C_w=I-\one w^{\mathsf T}$.

    At the beginning of each block, the algorithm draws a random rate $\tau\ge\kappa$, whose heavy-tailed law is detailed in \cref{sec:steering}. The block is partitioned in chunks according to this rate $\tau$: each boundary between
consecutive rounds within the block is independently selected as a
restart with probability $\tau$, using fresh coins independent
of the randomness used to generate the chunk trajectories.
A restart of this type draws a fresh action from $p$ and resets the steering score, which keeps the reweighting of each chunk bounded and lets each chunk's estimate be corrected separately. Each chunk is played by \proc{ProbeChunk}, which plays according to some holding chain $P_\tau$ of $(p,x)$ (defined in \cref{sec:holding}) tilted by the observed comparisons, and returns a corrected estimate of the chunk's centered loss (\cref{sec:steering}). More precisely, conditional on the sampled rate and chunk boundaries,
the path law within a chunk $J$ is the law of a stationary
$P_\tau$-chain reweighted by $1+\beta S_J$, where $S_J$ is the
final stopped score sum defined in \cref{sec:steering}.
The procedure \proc{ProbeChunk} samples this law sequentially. The sum of these estimates, discarded if it is too large, supplies the next exponential-weights update. All parameters, including the base rate, are specified in \cref{sec:parameters}.

\begin{algorithm}[htbp]
\caption{\proc{TournamentSteering}}
\label{alg:main}
\begin{algorithmic}[1]
\REQUIRE Base rate $\kappa=2^{-m}$; block length $H$; smoothing $\varepsilon$; learning rate $\eta$; score threshold $M$.
\STATE Split the horizon into blocks of $H$ rounds, the last of which may be shorter, and set $\pi_1\gets\one/K$.
\FOR{each block $b$, of length $h$}
\STATE Smooth $\pi_b$ to $w$ with smoothing $\varepsilon$: $w\gets(1-\varepsilon)\pi_b+\varepsilon\one/K$
\STATE Prepare the block: $(p,x,(\Psi_s)_s)\gets\proc{PrepareBlock}(w,\kappa)$.
\STATE Draw the restart rate $\tau$ according to \cref{eq:rate-law} and compute holding chain $P_\tau$ as in \cref{eq:holding-chain}. \label{line:draw-rate}
\STATE Partition $b$ in chunks $\cP_\tau$ by independently cutting each internal boundary with
probability $\tau$
 \label{line:rate}

\STATE Sum the corrected estimates $Z\gets\sum_{J\in\cP_\tau}\proc{ProbeChunk}(J,p,(\Psi_s)_s,P_\tau,M)$. \label{line:chunk}
\STATE Update the weights $\pi_{b+1,i}\propto\pi_{b,i}\exp(-\eta\widehat G_{b,i})$, where $\widehat G_b\gets\one\{\eta\|Z\|_\infty\le1\}\,Z$. \label{line:truncate}
\ENDFOR
\end{algorithmic}
\end{algorithm}

\subsection{Holding chains}
\label{sec:holding}

    We describe here the structure of the holding chains used within each block. Given two probability vectors $p$ and $x$, we construct a family of chains indexed by a parameter $\tau$, sharing $p$ as stationary law and whose switching frequency is governed by the refresh law $x$.
A chain started from $p$ has marginal law $p$ in every round, although it holds its action for long stretches and switches rarely. This law fixes the baseline loss $p^{\mathsf T} L$ of a block with loss sum $L$ and the mean of the estimate. Holding times may nevertheless differ across actions: an action whose comparisons receive a large importance weight should be left quickly, so that its contribution to the estimate does not persist. For simplicity, we say that $a \ge b$ for two generic vectors $a$ and $b$, if the inequality holds component-wise.
\begin{definition}[Holding family]
\label{def:holding-family}
Let $p,x$ be positive probability vectors with $x\ge p/2$, and put $\lambda_s=x_s/p_s$ and $\lmax=\max_s\lambda_s$. For a rate $0<\tau\le1$, let
\begin{align}
a_s^{(\tau)}=\frac{\tau\lambda_s}{2(1+\tau\lambda_s)}
\qquad\text{and}\qquad
P_\tau(j\mid s)=\bigl(1-a_s^{(\tau)}\bigr)\one\{s=j\}
+a_s^{(\tau)}\frac{p_ja_j^{(\tau)}}{\sum_rp_ra_r^{(\tau)}}.
\label{eq:holding-chain}
\end{align}
We call $(P_\tau)_{0<\tau\le1}$ the holding family of $(p,x)$, and $x$ its refresh law.
\end{definition}
At state $s$, the chain $P_\tau$ keeps its state with probability $1-a_s^{(\tau)}$. Otherwise it refreshes, drawing a destination with probability proportional to $p_ja_j^{(\tau)}$, the destination may be $s$ itself. These destination probabilities compensate for the different holding times, so that $P_\tau$ is reversible with respect to $p$ and $p$ remains stationary.
When $\tau\lmax$ is small, the refresh probability $a_s^{(\tau)}$ is about $\tau\lambda_s/2$ uniformly over $s$, and refreshes land approximately according to $x$.

The parameters play distinct roles. The stationary law $p$ fixes the baseline loss and the importance weights. The refresh law $x$ fixes how quickly each state is refreshed, through $\lambda_s=x_s/p_s$. The condition $x\ge p/2$, that is, $\lambda_s\ge\tfrac12$, keeps every state from being refreshed too slowly. The rate $\tau$ sets the overall speed: a chain started from $p$ switches at most $\tau/2$ times per round in expectation. The particular form of $a_s^{(\tau)}$ keeps $a_s^{(\tau)}\le\tfrac12$ at every rate, even when $\tau\lambda_s$ is large, so every holding probability is at least $\tfrac12$. \cref{sec:spectral} shows that it also gives all members of the family a common basis of eigenvectors. The next lemma collects the properties used below and its proof is in \cref{app:holding}.
\begin{lemma}[Holding chains]
\label{lem:holding-chains}
    Let $p,x$ be positive probability vectors with $x\ge p/2$, and  $(P_\tau)_{0<\tau\le1}$ their holding family.
\begin{enumerate}
\item $P_\tau$ is reversible with respect to $p$. Moreover, $a_s^{(\tau)}\le\tfrac12$, $\sum_sp_sa_s^{(\tau)}\le\tau/2$, $a_s^{(\tau)}\ge\tfrac13\sum_jp_ja_j^{(\tau)}$, and $P_\tau(s\mid s)\ge1-a_s^{(\tau)}\ge\tfrac12$ for every $s$.
\item A $P_\tau$-chain started from $p$ has marginal law $p$ at every round, and its expected number of switches in $h$ rounds is at most $(h-1)\tau/2$.
\end{enumerate}
\end{lemma}

\subsection{Estimates from comparisons}
\label{sec:comparison-estimates}

    While the holding family and the restart parameter $\tau$ guide  how to play within each block, we need to specify how the feedback observed along the way is used to produce good estimates of the actual losses. Since exponential weights are unaffected by adding a common constant to all coordinates of its input, it suffices to estimate centered losses.

    Fix a positive probability vector $v$, which we call the anchor distribution. An estimate of the centered loss $C_v\ell_t$ can share one observation among several coordinates, using the identity $f_i-v^{\mathsf T} f=\sum_rv_r(f_i-f_r)$. We call the played action the source of the comparisons it reveals, and the coordinate $i$ being estimated the target. We distribute the comparisons $f_i-f_r$ among the sources that observe both endpoints, and we control the weight that each source carries.
\begin{definition}[Comparison assignment]
\label{def:assignment}
For a positive probability vector $v$, let $\cC(v)$ be the set of pairs $((Q_s)_{s\in[K]},\rho)$ of $K\times K$ matrices and vectors $\rho\in\R^K$ such that
\begin{align}
Q_s\one=0,\qquad (Q_s)_{ir}\le0\ (i\ne r),\qquad
(Q_s)_{ir}=0\text{ unless }i,r\in\overline N^+(s),\qquad
\sum_sQ_s=C_v,
\label{eq:assignment}
\end{align}
and, for all $s,i\in[K]$,
\begin{align}
(Q_s)_{ii}\le\rho_s,\qquad v_i(Q_s)_{ii}\le v_s\rho_i,\qquad v\le\rho\le\one.
\label{eq:loads}
\end{align}
We call $\sum_s\rho_s$ the total source weight and $q=\rho/\sum_j\rho_j$ the normalized source weights.
\end{definition}
The first property, namely that each row sums to zero yields, for every $f$ the following equality:
\begin{align}
(Q_sf)_i=\sum_{r\ne i}\bigl[-(Q_s)_{ir}\bigr](f_i-f_r).
\label{eq:comparison-form}
\end{align}
Each nonzero term is observable at $s$ by the support constraint, so $Q_s\ell_t$ is computable from the comparisons observed when $s$ is played. The diagonal entry $(Q_s)_{ii}$ is the total weight that source $s$ contributes to target $i$; we call it the load of $s$ on $i$. The first load constraint in \cref{eq:loads} caps the contribution of each source. The second balances the contribution to a target against the anchor weights of source and target. The error of the estimate below is controlled by the total source weight, so we choose the assignment by the linear program
\begin{align}
\text{minimize }\sum_s\rho_s\quad\text{subject to }((Q_s)_{s\in[K]},\rho)\in\cC(v),
\label{eq:assignment-lp}
\end{align}
whose constraints \cref{eq:assignment,eq:loads} are linear in the entries of the matrices $Q_s$ and the vector $\rho$. \Cref{lem:rooted-assignment} shows that its value is at most $3\log(e/v_{\min})$.

Given an assignment and the sampling law $p$ of the block, the estimation matrices are
\begin{align}
\Psi_s=\frac{C_wQ_s}{p_s},
\qquad
G_J=\sum_{t\in J}\Psi_{X_t}\ell_t.
\label{eq:estimator}
\end{align}
Every summand is observable, including on rounds when the chain holds. Since $\sum_sQ_s=C_v$ and $C_wC_v=C_w$, we have $\sum_sp_s\Psi_s=C_w$. Under a stationary reference chain, $X_t\sim p$, so for every fixed chunk $J$, $G_J$ is an
unbiased estimator of the centered loss $C_wL_J$, where
$L_J=\sum_{t\in J}\ell_t$: $\E G_J=C_wL_J$. The played chain is tilted toward favorable paths.  \cref{sec:steering} describes this tilt and the correction that preserves the reference expectation.

\subsection{The probing procedure}
\label{sec:steering}

This section is devoted to explaining precisely how our algorithm plays in each chunk. The construction has five ingredients: observable scores, a linear tilt, a stopping rule, the corrected estimate, and a random rate with restarts. The first four define \proc{ProbeChunk}, and the last explains the restarts described in \cref{line:rate} of \cref{alg:main}. For a pseudocode of \proc{ProbeChunk} we refer to \cref{alg:probe-chunk}.

\paragraph{Scores.}
Let $P$ be a reference chain, reversible with respect to $p$, with positive holding probabilities. Of the two endpoints of a pair, only one observes their comparison. To account for this, we define the following notion of score:
\begin{align}
d_t(s,j)=
\begin{cases}
\Delta_t(s,j),&s\to j,\\
0,&j\ne s,\ j\to s,\\
-\displaystyle\frac{\sum_{i\ne s}P(i\mid s)d_t(s,i)}{P(s\mid s)},&j=s.
\end{cases}
\label{eq:scores}
\end{align}
Thus the score of an edge $s\to j$ is the advantage of its blind endpoint over its revealing endpoint, as in apple tasting (\cref{sec:tech-two}). The diagonal score makes the mean score under $P(\cdot\mid s)$ zero, and every score is computable from the comparisons observed at $s$. By detailed balance, a stationary transition crosses each pair equally often in both directions, so on average the one-sided scores act exactly like the full-information scores (\cref{lem:tournament-scores}). As an example, consider the two-action case, with $1\to2$ and $d_t=\Delta_t(1,2)$, the scores are
\begin{align}
d_t(1,2)=d_t,
\qquad
d_t(1,1)=-\frac{P(2\mid1)}{P(1\mid1)}\,d_t,
\qquad
d_t(2,1)=d_t(2,2)=0.
\label{eq:two-scores}
\end{align}

\paragraph{Reweighting.}
Consider a chunk $J$, with its rounds relabeled $1,\ldots,|J|$, started from $X_1\sim p$ with score sum $S=0$. It plays a path whose law is the law of a stationary $P$-chain multiplied by $1+\beta S$, where $S$ is the total score collected along the path and $\beta>0$ is a small steering strength, fixed below. Paths that move toward actions with smaller current losses are thus played slightly more often. This law can be sampled online: the next action is drawn from
\begin{align}
\widetilde P_t(j\mid s)=P(j\mid s)\,\frac{1+\beta\bigl(S+d_t(s,j)\bigr)}{1+\beta S},
\label{eq:tilt}
\end{align}
and the realized score is added to $S$. The probabilities in \cref{eq:tilt} sum to one over $j$, because the scores have mean zero under $P(\cdot\mid s)$. Along a path, the factors that multiply $P(j\mid s)$ are ratios of consecutive values of $1+\beta S$, so their product telescopes, and the probability of each path is its reference probability times the final value of $1+\beta S$. In probabilistic terms, $1+\beta S$ is a positive martingale under the reference chain, and the tilt is the change of measure that it defines \citep{williams1991probability}. We call this rule a linear tilt of $P$, and its main consequence is that the expected loss of the chunk differs from the stationary loss by exactly $\beta$ times the correlation between $S$ and the loss under the reference chain.

\paragraph{Stopping.} For our argument to work,
the weight $1+\beta S$ must remain positive. To this end, we use a threshold parameter $M$ fixed in the main loop. Once $|S|>M-1$, the chunk freezes $S$ and uses $P$ until it ends; we say that scoring stops in this case. Every score has magnitude at most one when $P(s\mid s)\ge\nicefrac12$, so $|S|\le M$. We set the steering strength to $\beta=1/(4M)$, which keeps $1+\beta S$ in $[\nicefrac34,\nicefrac54]$. Scoring stops only with small probability (\cref{lem:tails}). Note that freezing the score leaves the likelihood identity exact.

\paragraph{Correcting the estimate.}
Let $S_J$ be the final stopped score sum of a chunk $J$, and put $R_J=1+\beta S_J$. The likelihood identity gives
\begin{align}
\Prb_{\rm tilt}(X_J)=R_J\Prb_0(X_J),
\qquad
\E_{\rm tilt}\frac{G_J}{R_J}=\E_0G_J=C_wL_J,
\label{eq:chunk-correction}
\end{align}
where $\Prb_0,\E_0$ refer to the stationary reference law of the chunk and $\Prb_{\rm tilt},\E_{\rm tilt}$ to its tilted law. The correction is computable from the observed path, and its denominator lies in $[3/4,5/4]$. The correction affects only the estimate, not the actions, so the loss improvement of the tilt is retained.

\begin{algorithm}[htbp]
\caption{\proc{ProbeChunk}$(J,p,(\Psi_s)_s,P,M)$}
\label{alg:probe-chunk}
\begin{algorithmic}[1]
\REQUIRE Chunk $J$, with its rounds relabeled $1,\ldots,|J|$; sampling law $p$; estimation matrices $(\Psi_s)_s$; reference chain $P$; score threshold $M$, with steering strength $\beta=1/(4M)$.
\STATE Draw the first action $X_1\sim p$, and set the score sum $S\gets0$ and the chunk estimate $G\gets0$.
\FOR{$t=1,\ldots,|J|$}
\STATE Play $s=X_t$, observe its comparisons, and add its contribution $G\gets G+\Psi_s\ell_t$.
\IF{$t<|J|$}
    \IF{scoring has stopped, i.e., $|S|>M-1$}
    \STATE Set the scores $\hat d_j\gets0$ for all $j$
    \ELSE
    \STATE Set the scores $\hat d_j\gets d_t(s,j)$ from \cref{eq:scores}.
\ENDIF
\STATE Draw $X_{t+1}=j$ with probability $P(j\mid s)[1+\beta(S+\hat d_j)]/(1+\beta S)$
\STATE Update the realized score $S\gets S+\hat d_{X_{t+1}}$.
\ENDIF
\ENDFOR
\STATE \textbf{return} the corrected estimate $G/(1+\beta S)$.
\end{algorithmic}
\end{algorithm}

\paragraph{Rates and restarts.}
The law of the rate matches steering to estimation on every time scale. For intuition, consider the case $p=x$ and a temporal rate $\kappa\le u\le1$. A loss component that changes sign every $1/u$ rounds produces estimation error that accumulates for about $\min\{1/\kappa,1/u\}$ rounds, whereas a chain of rate $\tau$ can exploit it only if $\tau\gtrsim u$, so the base rate alone would leave fast components with too little gain. A rate fast enough for the time scale $1/u$ should therefore be drawn with probability about $\kappa/u$, so that the steering gain decays in $u$ like the estimation error; a heavier tail would switch too often, and a lighter one would leave fast components unpaid. Each block draws one rate from the dyadic law
\begin{align}
\Prb(\tau=y)=\frac\kappa{2y}\quad\bigl(y\in\{\kappa,2\kappa,4\kappa,\ldots,\tfrac12\}\bigr),
\qquad
\Prb(\tau=1)=\kappa,
\label{eq:rate-law}
\end{align}
and holds it throughout (\cref{line:draw-rate} of \cref{alg:main}). The $m+1$ probabilities sum to one, and $\Prb(\tau\ge y)\le\kappa/y$ for $\kappa\le y\le1$, with equality when $y$ is one of the rates. In this law, each rate below one has probability proportional to its inverse, which lets the analysis compare time scales that differ by powers of two at no cost (\cref{lem:rate-comparison}). Each of the $m$ rates below one contributes $\kappa/2$ to the mean rate, and the atom at $\tau=1$ contributes $\kappa$, so the mean rate is
\begin{align}
\mu=\E\tau=\kappa\Bigl(1+\frac m2\Bigr).
\label{eq:mean-rate}
\end{align}
Since restarts occur with probability $\tau$, faster chains also restart more often, which keeps the score accumulation controlled.

\Cref{line:chunk,line:truncate} of \cref{alg:main} sum the corrected chunk estimates and discard the output if it exceeds $1/\eta$ in maximum norm:
\begin{align}
Z=\sum_{J\in\cP_\tau}\frac{G_J}{R_J},
\qquad
\widehat G=\one\{\eta\|Z\|_\infty\le1\}Z,
\label{eq:block-raw}
\end{align}
where $\cP_\tau$ is the random partition into chunks. By \cref{eq:chunk-correction}, $Z$ is unbiased for $C_wL$, and truncation introduces only a small bias.

If $i_0$ is a sink of the selected tournament, then $\overline N^+(i_0)=\{i_0\}$. The support constraint and zero row sums force $Q_{i_0}=0$, and all scores at $i_0$ vanish. The algorithm therefore never needs feedback when playing a sink, which is why a sink may lack a self-loop.

\subsection{Block preparation}
\label{sec:preparation}

The last routine needed to specify our algorithm is \proc{PrepareBlock}, which takes as input the smoothed weights $w$ and a base rate $\kappa$ and returns the vectors $p$ and $x$ used in the construction of the holding family as well as the comparison matrices $\Psi_s$. \proc{PrepareBlock} aims at meeting three requirements. (i) The sampling law $p$ should be close to $w$, so that the baseline loss $p^{\mathsf T} L$ is close to the loss prescribed by exponential weights. (ii) The comparison matrices $\Psi_s$ should carry little total source weight, so that importance-weighted estimates remain small. (iii) Sources that carry substantial comparison weight should be sampled often enough and should refresh quickly.

\proc{PrepareBlock}, presented in \Cref{alg:prepare} meets these requirements in two steps. It first solves the linear program in \cref{eq:assignment-lp} with $w$ itself as the anchor distribution, finding an assignment in $\cC(w)$ of minimum total source weight. Let $q$ denote the normalized source weights associated with the LP solution (recall the definition of normalized source weights from \Cref{def:assignment}). The algorithm defines $p$ and $x$ by mixing the input weight $w$ and $q$: $p=(1-4\kappa)w+4\kappa q$, and $x=(p+q)/2$. These choices ensure that $p_s\ge4\kappa q_s$ for every source $s$, and that a heavily loaded source, with $q_s$ much larger than $p_s$, has a large relative refresh speed $\lambda_s=(1+q_s/p_s)/2$.

\begin{algorithm}[htbp]
\caption{\proc{PrepareBlock}$(w,\kappa)$}
\label{alg:prepare}
\begin{algorithmic}[1]
\REQUIRE Smoothed weights $w$; base rate $\kappa$.
\STATE Let $((Q_s)_s,\rho)$ be a solution of the LP in \cref{eq:assignment-lp} with anchor distribution $w$.
\STATE Compute the normalized source weights $q\gets\rho/\sum_j\rho_j$
\STATE Define the sampling law $p\gets(1-4\kappa)w+4\kappa q$, and the refresh law $x\gets(p+q)/2$.
\STATE \textbf{return} $p$, $x$, and $\Psi_s\gets C_wQ_s/p_s$ for $s\in[K]$.
\end{algorithmic}
\end{algorithm}

\subsection{Parameters and guarantee}
\label{sec:parameters}

The parameters of \proc{TournamentSteering} (\cref{alg:main}) balance the term of order $\log K/\eta$ from exponential weights against the switching cost, which is of order $\kappa T$ up to a logarithmic factor. To obtain the regret bound, play a fixed action throughout if $K=1$ or $T=1$. Otherwise, let $\Lambda=\log(eKT)$ and put $\kappa_*=\Lambda\sqrt{\log K/T}$. If $\kappa_*>1/32$, again play a fixed action. In the remaining case, let $m=\lfloor\log_2(1/\kappa_*)\rfloor$, so that $m\ge5$ and the base rate $\kappa=2^{-m}$ satisfies $\kappa_*\le\kappa<2\kappa_*$, and run \cref{alg:main} with
\begin{align}
H&=\frac{1}{8\kappa},
&\delta&=(KT)^{-20},
&M&=100\log\frac2\delta,
&\eta&=\frac{\kappa}{2^{25}Mm\Lambda},
&\varepsilon&=\frac1T.
\label{eq:parameters}
\end{align}
The block length is an integer with $H\kappa=1/8$, and smoothing gives $w_{b,i}\ge1/(KT)$. Estimates are truncated at $1/\eta$, the range allowed by \cref{lem:ew}. The threshold $M$ makes both score stopping and truncation occur with probability of order $\delta$, and it sets the steering strength $\beta=1/(4M)$. The learning rate is small enough for the steering gain to absorb the second moment charged by exponential weights; its factor $m$ is the number of dyadic bands of eigenvalues in the proof of \cref{lem:energy}. The following theorem gives a square-root regret guarantee whose dependence on $K$ is only through $\log K$ and logarithmic factors.
\begin{theorem}[Strongly revealing graphs]
\label{thm:main}
For every $K,T\ge1$ and every known, fixed strongly revealing graph, \proc{TournamentSteering} (\cref{alg:main}) satisfies, against every oblivious loss sequence,
\begin{align}
\Reg_T&\le2^{40}\min\{T,\sqrt{T\log K}\,\log^2(eKT)\}.\nonumber
\end{align}
\end{theorem}
The total running time of \proc{TournamentSteering} is polynomial in $K$ and $T$. This includes the solution of the linear program in each block (\cref{sec:preparation}), sampling of the actions, and updates of the exponential weights.
Prediction with expert advice is a special case, so its lower bound of order $\min\{T,\sqrt{T\log K}\}$ \citep[Chapter~3]{cesa2006prediction} shows that the bound is tight up to the factor $\log^2(eKT)$. For $K=2$, comparison probing (\cref{thm:breakeven}) removes this factor. The numerical constant is not optimized.

\section{Analysis of \proc{TournamentSteering}}
\label{sec:analysis}
This section proves \cref{thm:main}. The proof separates the learning update from the block procedure. Exponential weights run on block estimates tolerates their second moments as long as each block pays for its own: the block's expected cost must fall below the learner's expected loss by twice the learning rate times the second moment of its estimate. The blocks of \cref{alg:main} pay in this way, up to terms of the order of the switching cost, and the theorem then follows by balancing the remaining terms in the base rate $\kappa$. \Cref{sec:reduction,sec:self-financing-proof} carry out this reduction to a single block.

The rest of the section proves the guarantee for a single block. It rests on two bounds stated through the steering gain of the block: steering lowers the expected loss in proportion to the gain, and the gain also bounds the second moment of the corrected estimate, up to a factor of order $m\Lambda/\kappa$, where $m=\log_2(1/\kappa)$, and terms of the order of the switching cost. The first bound follows from an exact identity for the linear tilt. The second is the main technical step, in which the eigenvalues of the reference chain are grouped into dyadic bands and the quadratic forms in the losses are compared through random partitions of the block. Throughout, the played chain is analyzed through its reference chain, whose path law differs from the played one by a bounded factor, and the feedback enters only through two properties, which tournaments provide.

\subsection{Reduction to a single block}
\label{sec:reduction}

The reduction uses two propositions. The first concerns a generic blocked version of exponential weights. Consider blocks $b=1,\ldots,B$ of lengths $h_b$ summing to $T$, and write $L_b$ for the sum of the loss vectors of block $b$. Let $\pi_b$ be the exponential-weights distribution updated with learning rate $\eta$ on vectors $\widehat G_b\in\R^K$, as in \cref{lem:ew}, and let $w_b=(1-\varepsilon)\pi_b+\varepsilon\one/K$ with $0<\varepsilon\le1/2$ be its smoothed version. In block $b$ the learner runs a randomized procedure that may depend on the past. Let $\cF_b$ be the history before block $b$. The block cost $C_b$ is the loss incurred in block $b$, plus the number of switches between consecutive rounds of block $b$, plus one.

\begin{proposition}[Self-financing exponential weights]
\label{prop:self-financing}
Suppose that for every block $b$, almost surely,
\begin{align*}
\E[\widehat G_b\mid\cF_b]=C_{w_b}L_b-\xi_b,
\qquad
\eta\|\widehat G_b\|_\infty\le1,
\qquad
\E[C_b\mid\cF_b]\le\ip{w_b}{L_b}+\sigma_b-2\eta\,\E\bigl[\|\widehat G_b\|_{w_b}^2\bigm|\cF_b\bigr],
\end{align*}
where the vector $\xi_b$ and the number $\sigma_b$ are $\cF_b$-measurable. Then
\begin{align*}
\Reg_T\le\frac{\log K}{\eta}+\sum_{b=1}^B\E\bigl[\sigma_b+2\|\xi_b\|_\infty\bigr]+\varepsilon T.
\end{align*}
\end{proposition}

The centering $C_{w_b}$ in the mean does not matter, because the vectors $\pi_b-e_i$ in \cref{lem:ew} have coordinates summing to zero. The proposition does not involve the feedback graph: the whole difficulty is moved into the block condition, which asks each block to pay for the full second moment of its own estimate, including its squared mean. This is the counterpart of the accounting of comparison probing (\cref{sec:tech-two}), where each probe pays for itself pathwise: here each block pays, in expectation, for the information that its estimate supplies to exponential weights.

This payment is what separates the rates $\sqrt T$ and $T^{2/3}$. A block procedure that merely plays a lazy chain with marginal law close to $w_b$ has expected cost about $\ip{w_b}{L_b}$ plus its switches, so the block condition would force $\sigma_b$ to contain $2\eta\E\|\widehat G_b\|_{w_b}^2$. For blocks of length about $1/\kappa$, where $\kappa$ is the block's base rate, a persistent loss difference already makes the squared mean $\|C_{w_b}L_b\|_{w_b}^2$ of order $1/\kappa^2$. Over about $\kappa T$ blocks, the second-moment term is then of order $\eta T/\kappa$, and balancing it against $\log K/\eta$ and the switching cost $\kappa T$ gives regret of order $T^{2/3}$. The algorithm meets the block condition differently: it steers its moves by the observed comparisons, and the resulting loss reduction pays for the second moment.

The second proposition verifies the block condition for \cref{alg:main}, with an additive cost of order $\mu h$ per block, where $\mu=\kappa(1+m/2)$ is the mean rate of \cref{eq:mean-rate}.

\begin{proposition}[Block guarantee]
\label{prop:block}
Let $K,T\ge2$ and $\kappa=2^{-m}$ with an integer $m\ge5$, and set $H,\delta,M,\eta$ as in \cref{eq:parameters}. Let $w$ be a distribution with $w_i\ge1/(KT)$, prepared by \proc{PrepareBlock}$(w,\kappa)$. Consider a deterministic loss block of length $1\le h\le H$ with sum $L$, played by one iteration of \cref{alg:main}, and put $\mu=\kappa(1+m/2)$. Its output $\widehat G$ and cost $C$ satisfy
\begin{align}
\E\widehat G&=C_wL-\xi,
\qquad
\eta\|\widehat G\|_\infty\le1,
\label{eq:block-estimate}\\
\|\xi\|_\infty&\le\frac{2^6\delta h\Lambda}{\kappa},
\label{eq:block-bias}\\
\E C&\le\ip wL+1+6\mu h+\delta h^3-\frac{\kappa}{2^{24}Mm\Lambda}\,\E\|\widehat G\|_w^2.
\label{eq:block-cost}
\end{align}
The expected number of switches inside the block, plus one, is at most $1+2\mu h$.
\end{proposition}

The negative term in \cref{eq:block-cost} is the point of the proposition: each block pays for the second moment of its own output, as \cref{prop:self-financing} requires. The additive terms collect the switches, the difference between the sampling law $p$ and $w$, and rare events, and the vector $\xi$ accounts only for discarding a rare large estimate. The proposition is proved in \cref{sec:block-proof}, from the lemmas of \cref{sec:spectral,sec:interfaces,sec:energy,sec:estimation-lemma}. Combining it with \cref{prop:self-financing} gives a regret bound for every fixed base rate.

\begin{corollary}[Regret at a fixed base rate]
\label{cor:fixed-rate}
Let $K,T\ge2$, let $\kappa=2^{-m}$ with an integer $m\ge5$ and $\kappa\ge1/T$, and run \cref{alg:main} with the parameters in \cref{eq:parameters} for this $\kappa$. With $\mu=\kappa(1+m/2)$,
\begin{align}
\Reg_T\le\frac{2^{25}Mm\Lambda\log K}{\kappa}+14\mu T+3,
\qquad
\E N_T\le10\mu T+1.
\label{eq:fixed-rate}
\end{align}
\end{corollary}

\begin{proof}
We apply \cref{prop:self-financing} with the block costs of \cref{prop:block} and sum the additive terms. Smoothing with $\varepsilon=1/T$ gives $w_{b,i}\ge1/(KT)$, and conditional on $\cF_b$ the losses of block $b$ are deterministic because the adversary is oblivious. \Cref{prop:block} therefore applies conditionally on $\cF_b$. It gives $\eta\|\widehat G_b\|_\infty\le1$ and verifies the hypotheses of \cref{prop:self-financing} with $2\eta=\kappa/(2^{24}Mm\Lambda)$ and $\sigma_b=1+6\mu h_b+\delta h_b^3$. The proposition yields
\begin{align*}
\Reg_T\le\frac{\log K}{\eta}+\sum_{b=1}^B\Bigl(1+6\mu h_b+\delta h_b^3+\frac{2^7\delta h_b\Lambda}{\kappa}\Bigr)+1.
\end{align*}
The block length is $H=1/(8\kappa)$, so the number of blocks is $B\le8\kappa T+1\le8\mu T+1$. Moreover, $\sum_b\delta h_b^3\le\delta TH^2\le\delta T^3$ and, since $1/\kappa\le T$, $\sum_b2^7\delta h_b\Lambda/\kappa\le2^7\delta\Lambda T^2$. Since $T\le KT$ and $\Lambda\le KT$, with $KT\ge4$, these two sums are at most $129\,(KT)^{-17}<1$. Since $\log K/\eta=2^{25}Mm\Lambda\log K/\kappa$, this proves the regret bound. For switches, \cref{prop:block} bounds the expected number of switches inside block $b$, plus one, by $1+2\mu h_b$. Charging each switch at a block boundary to this allowance of the next block gives $\E N_T\le\sum_b(1+2\mu h_b)\le10\mu T+1$.
\end{proof}

\paragraph{How the tradeoff balances.}
The proof of \cref{prop:block} measures both the loss improvement and the estimate collected along the played path by the steering gain $\Gamma$, in three steps.
\begin{itemize}
\item \emph{Steering.} The tilted actions change the expected loss by $-\beta$ times the gain, up to a rare stopping error. The reference rate has mean $\mu=\kappa(1+m/2)$, and switches within chunks and at restarts cost $O(\mu h)$ in a block of length $h$.
\item \emph{Estimation.} Correcting each chunk by its likelihood ratio gives an unbiased block estimate before truncation. Its second moment is at most $2^{14}mL_0/\kappa$ times the gain plus $\mu h/16$, including its squared mean. The factor $L_0$ is logarithmic in the smallest weight and is the only trace of the feedback structure in this bound, and the factor $m$ is the number of dyadic bands into which the analysis groups the eigenvalues of the reference chain.
\item \emph{Exponential weights.} The update charges $2\eta$ times this moment. Steering covers that charge when $2^{15}\eta mL_0/\kappa\le\beta$. The parameter choice satisfies this because $L_0\le2^8\Lambda$. The remaining charge is at most $\beta\mu h/16$, which is absorbed by switching costs.
\end{itemize}
The ratio $\beta\kappa/(mL_0)$ is thus the rate at which the steering gain pays for the estimate's second moment; it shrinks with the base rate $\kappa$, so slower chains force a smaller learning rate. What remains is the term $\log K/\eta$, of order $Mm\Lambda\log K/\kappa$, which is the first term of \cref{eq:fixed-rate}, and the switching cost, of order $\mu T$, which is the second. Since $\mu$ is of order $m\kappa$, the two terms are equal when $\kappa^2$ is of order $M\Lambda\log K/T$, and both are then of order $m\sqrt{T\log K\cdot M\Lambda}$. The logarithmic factors have separate sources. The factor $\Lambda$ comes from $L_0$, which is logarithmic by \cref{lem:rooted-assignment,lem:balanced-estimator}. The factor $M$ enters through the cap on $\beta$, and it is logarithmic because score accumulation has an exponential tail (\cref{lem:tails}). The factor $m=\log_2(1/\kappa)$ counts the dyadic rates below one and the bands of eigenvalues; it enters the term $\log K/\eta$ through the estimation bound and the switching cost through the mean rate. Each is at most a constant multiple of $\Lambda=\log(eKT)$, which gives the rate $\sqrt{T\log K}\,\Lambda^2$ of \cref{thm:main}.

We now prove the theorem from \cref{prop:graph,cor:fixed-rate}.

\begin{proof}[Proof of \cref{thm:main}]
We choose $\kappa$ to balance the two terms of \cref{eq:fixed-rate}, and then treat the trivial regimes. Assume $K,T\ge2$ and $\kappa_*\le1/32$, where $\kappa_*=\Lambda\sqrt{\log K/T}$, and let $\kappa=2^{-m}$ with $m=\lfloor\log_2(1/\kappa_*)\rfloor$, as in \cref{sec:parameters}. Then $m\ge5$ and $\kappa_*\le\kappa<2\kappa_*$. Since $\Lambda\sqrt{T\log K}\ge1$, we have $\kappa\ge\kappa_*\ge1/T$, so \cref{cor:fixed-rate} applies, and $m\le\log_2T\le2\Lambda$. Hence, using $\kappa<2\kappa_*$, we have $\mu=\kappa(1+m/2)\le\kappa m\le4\Lambda^2\sqrt{\log K/T}$. Moreover, $M=100(\log2+20\log(KT))\le2000\Lambda$. With this choice, the two terms of \cref{eq:fixed-rate} are both of order $\Lambda^2\sqrt{T\log K}$:
\begin{align*}
\frac{2^{25}Mm\Lambda\log K}{\kappa}\le\frac{2^{26}\cdot2000\,\Lambda^3\log K}{\kappa_*}=2^{26}\cdot2000\,\Lambda^2\sqrt{T\log K},
\qquad
14\mu T\le56\,\Lambda^2\sqrt{T\log K}.
\end{align*}
Hence $\Reg_T\le(2^{26}\cdot2000+56)\,\Lambda^2\sqrt{T\log K}+3\le2^{40}\sqrt{T\log K}\,\Lambda^2$, where the last step uses $\sqrt{T\log K}\,\Lambda^2>4$, which holds for $K,T\ge2$. Similarly, $\E N_T\le10\mu T+1\le41\sqrt{T\log K}\,\Lambda^2$.

If $\kappa_*>1/32$, a fixed action has regret at most $T$ and no switches, while $T<32\Lambda\sqrt{T\log K}\le32\sqrt{T\log K}\,\Lambda^2$. Every realized path of any algorithm has loss-plus-switching regret at most $2T$ and at most $T-1$ switches, which supplies the term $T$ in the minimum. For $K=1$ regret and switches are zero. For $T=1$ a fixed action has regret at most one and no switches, and the minimum in the statement equals one when $K\ge2$. \Cref{prop:graph} supplies the required comparisons. \end{proof}

\subsection{Proof of the self-financing proposition}
\label{sec:self-financing-proof}

The proposition follows from the second-order bound of \cref{lem:ew}, applied pathwise, once the smoothing and the centering are accounted for.

\begin{proof}[Proof of \cref{prop:self-financing}]
We apply \cref{lem:ew} pathwise, take conditional expectations block by block, and cancel the second moments. Fix an action $i$. Since $\varepsilon\le1/2$, we have $\pi_b\le2w_b$ coordinatewise, so \cref{lem:ew} gives
\begin{align*}
\sum_b\ip{\pi_b-e_i}{\widehat G_b}\le\frac{\log K}\eta+2\eta\sum_b\|\widehat G_b\|_{w_b}^2.
\end{align*}
The vectors $\pi_b$, $w_b$, $\xi_b$ and the losses $L_b$ are fixed given $\cF_b$. Because $\pi_b-e_i$ has coordinates summing to zero and $C_{w_b}L_b=L_b-(w_b^{\mathsf T} L_b)\one$,
\begin{align*}
\E\bigl[\ip{\pi_b-e_i}{\widehat G_b}\bigm|\cF_b\bigr]
=\ip{\pi_b}{L_b}-L_b(i)-\ip{\pi_b-e_i}{\xi_b}
\ge\ip{w_b}{L_b}-\varepsilon h_b-L_b(i)-2\|\xi_b\|_\infty,
\end{align*}
where we used $\ip{w_b-\pi_b}{L_b}=\varepsilon\ip{\one/K-\pi_b}{L_b}\le\varepsilon h_b$. Combining the two displays with the block condition gives
\begin{align*}
\E\sum_bC_b-\sum_bL_b(i)
&\le\E\sum_b\Bigl(\ip{w_b}{L_b}+\sigma_b-2\eta\|\widehat G_b\|_{w_b}^2-L_b(i)\Bigr)\\
&\le\frac{\log K}\eta+\E\sum_b\bigl(\sigma_b+2\|\xi_b\|_\infty\bigr)+\varepsilon T.
\end{align*}
Every switch between consecutive rounds either lies inside a block or occurs at a block boundary and is covered by the extra one in the cost of the next block. Hence $\sum_bC_b$ is at least the incurred loss plus $N_T$. Choosing $i$ as a best fixed action proves the proposition.
\end{proof}

\subsection{Spectral structure of the holding chains}
\label{sec:spectral}

The rest of the section proves \cref{prop:block}. Both the steering gain and the estimate's second moment depend on correlations along the reference chain. These are simple for the holding family: every member is a function of one matrix that does not depend on the rate. For $\alpha\ge0$ write $\nu_\tau(\alpha)=(2+\tau\alpha)/(2(1+\tau\alpha))$, and call a vector $f$ $p$-centered if $p^{\mathsf T} f=0$.

\begin{lemma}[Spectral structure]
\label{lem:spectral}
Let $(p,x)$ be as in \cref{def:holding-family}, $0<\tau\le1$, and
\begin{align}
A=\diag(\lambda)-\lambda x^{\mathsf T},
\qquad\text{that is,}\qquad
(Af)_s=\lambda_s(f_s-x^{\mathsf T} f).
\label{eq:generator}
\end{align}
\begin{enumerate}
\item The matrix $\diag(p)A=\diag(x)-xx^{\mathsf T}$ is symmetric, $A\one=0$, and $\ip f{Af}_p=\Var_x(f)$ for every $f$. There are $p$-centered vectors $\phi_1,\ldots,\phi_{K-1}$, orthonormal for $\ip\cdot\cdot_p$, with $A\phi_n=\alpha_n\phi_n$ and $\tfrac12\le\alpha_n\le\lmax$. Together with $\one$ they form a basis of $\R^K$.
\item $(I+\tau A)(2P_\tau-I)=I$. Consequently, $P_\tau\one=\one$ and $P_\tau\phi_n=\nu_\tau(\alpha_n)\phi_n$ for every $n$.
\end{enumerate}
\end{lemma}

The proof is in \cref{app:holding}. By (ii), $I-P_\tau=\tfrac\tau2A(I+\tau A)^{-1}$, which is close to $\tfrac\tau2A$ when $\tau\lmax$ is small. Thus $A$ describes one step of the slow chains, and every member of the family is an explicit function of $A$. The basis in (i) is the spectral decomposition of $A$, obtained as for every reversible transition matrix (\cref{sec:background}). For a vector $f$ we write $[f]_n=\ip{\phi_n}f_p$ for its coordinates in this basis, so that $f=(p^{\mathsf T} f)\one+\sum_n[f]_n\phi_n$. Orthonormality and \cref{lem:spectral} give
\begin{align}
\ip fg_p&=(p^{\mathsf T} f)(p^{\mathsf T} g)+\sum_n[f]_n[g]_n,
&[P_\tau^kf]_n&=\nu_\tau(\alpha_n)^k[f]_n,\nonumber\\
\Var_p(f)&=\sum_n[f]_n^2,
&\Var_x(f)&=\sum_n\alpha_n[f]_n^2.
\label{eq:coordinates}
\end{align}
Every correlation across time below is therefore geometric in one coordinate, with a ratio that is a scalar function of an eigenvalue $\alpha\in[\tfrac12,\lmax]$. For $\ell\in[0,1]^K$, evaluating the variance at $\tfrac12$ also gives $\sum_n\alpha_n[\ell]_n^2=\Var_x(\ell)\le\tfrac14$.

\subsection{Two properties of the feedback}
\label{sec:interfaces}

In the rest of this section, $(p,x)$ is a holding family, with the eigenvectors $\phi_n$, the eigenvalues $\alpha_n$, and the coordinates $[\cdot]_n$ of \cref{lem:spectral,eq:coordinates}, and $\kappa=2^{-m}$, with an integer $m\ge5$, is a base rate. The lemmas on steering and estimation in \cref{sec:energy,sec:estimation-lemma} use the feedback only through the following two definitions, and this subsection shows that tournaments provide both. The first definition describes observable scores whose correlations with a reversible chain reproduce its Dirichlet form.

\begin{definition}[Score rule]
\label{def:score-rule}
Let $P$ be reversible with respect to a positive $p$, with $P(s\mid s)\ge\tfrac12$ for every $s$. A score rule for $P$ assigns to each round $t$, current action $s$, and destination $j$ a real number $d_t(s,j)$ such that, whenever $\ell_t\in[0,1]^K$,
\begin{enumerate}
\item the scores $d_t(s,\cdot)$ are computable from the comparisons observed when $s$ is played in round $t$;
\item $\sum_jP(j\mid s)d_t(s,j)=0$ for every $s$;
\item $|d_t(s,j)|\le1$ for $j\ne s$, and $|d_t(s,s)|\le2(1-P(s\mid s))$;
\item $\sum_{i,j}p_iP_{ij}d_t(i,j)g_j=-\ip{\ell_t}{(I-P)g}_p$ for every $g\in\R^K$.
\end{enumerate}
\end{definition}

Condition (ii) compensates the scores on the diagonal so that they average to zero, which is what makes the tilt of \cref{eq:tilt} a change of measure. Condition (iv) says that the scores, averaged along one stationary transition, pair the current loss with any function of the destination through the Dirichlet form of $P$ in \cref{eq:dirichlet}. With full information, the antisymmetric scores $\tfrac12\Delta_t(i,j)$ on all pairs, compensated on the diagonal as in \cref{eq:scores}, satisfy these conditions. The next lemma shows that a tournament achieves the same with scores on only one side of each pair; it is where the tournament enters steering.

\begin{lemma}[One-sided scores reproduce the Dirichlet form]
\label{lem:tournament-scores}
Fix a tournament orientation. Let $P$ be reversible with respect to a positive $p$, with $P(s\mid s)\ge\tfrac12$ for every $s$. The scores in \cref{eq:scores} form a score rule for $P$.
\end{lemma}

Each unordered pair is scored only from the endpoint that observes it. Detailed balance makes the flows $p_iP_{ij}$ and $p_jP_{ji}$ equal, so the direction of the scored edge does not matter; this is where reversibility is used.

\begin{proof}
We verify the four conditions of \cref{def:score-rule}; only the last uses the tournament. The off-diagonal scores in \cref{eq:scores} use the comparisons $\Delta_t(s,j)$ for $s\to j$, which are observed at $s$, and the diagonal score is a combination of them. The diagonal score is chosen so that $\sum_jP(j\mid s)d_t(s,j)=0$. For $\ell_t\in[0,1]^K$, every off-diagonal score has magnitude at most one. Since $P(s\mid s)\ge\tfrac12$, the diagonal score satisfies $|d_t(s,s)|\le(1-P(s\mid s))/P(s\mid s)\le2(1-P(s\mid s))$.

For condition (iv), fix $g$. The zero mean lets us subtract $g_i$ inside the inner sum, after which the diagonal term vanishes:
\begin{align*}
\sum_{i,j}p_iP_{ij}d_t(i,j)g_j
=\sum_{i,j}p_iP_{ij}d_t(i,j)(g_j-g_i)
=-\sum_{i\to j}p_iP_{ij}\bigl(\ell_t(i)-\ell_t(j)\bigr)(g_i-g_j).
\end{align*}
Each unordered pair $\{i,j\}$ appears exactly once in the last sum, through its unique edge. By detailed balance, $p_iP_{ij}(\ell_t(i)-\ell_t(j))(g_i-g_j)$ is symmetric in $i$ and $j$, so the sum equals $-\tfrac12\sum_{i\ne j}p_iP_{ij}(\ell_t(i)-\ell_t(j))(g_i-g_j)$. By \cref{eq:dirichlet}, this is $-\ip{\ell_t}{(I-P)g}_p$.
\end{proof}

The second definition describes estimators whose errors obey two endpoint bounds. Let $w$ be a positive distribution and $L_0\ge1$.

\begin{definition}[Balanced estimator]
\label{def:balanced}
Matrices $\Psi_s$, $s\in[K]$, form an $L_0$-balanced estimator for $(w,p,x)$ if
\begin{enumerate}
\item each product $\Psi_s\ell_t$ is computable from the comparisons observed when $s$ is played in round $t$;
\item $\Psi_s\one=0$ for every $s$, and $\sum_sp_s\Psi_s=C_w$;
\item the error matrices $E_s=\Psi_s-C_w$ satisfy, for every $f\in\R^K$,
\begin{align}
\sum_s\frac{p_s}{\lambda_s}\|E_sf\|_w^2\le L_0\Var_p(f),
\qquad
\sum_sp_s\|E_sf\|_w^2\le L_0\Var_x(f);
\label{eq:endpoint}
\end{align}
\item $\|\Psi_s\ell\|_\infty\le L_0\lambda_s$ for every $s$ and every $\ell\in[0,1]^K$.
\end{enumerate}
\end{definition}

Condition (ii) makes $\sum_t\Psi_{X_t}\ell_t$ unbiased for $C_wL$ under any process with marginal law $p$, such as the reference chain; the played chain is tilted, and the correction in \cref{eq:chunk-correction} restores this mean. We call the two inequalities in \cref{eq:endpoint} the persistence bound and the one-round bound. The first weights each source by its stationary mass and inverse relative speed $1/\lambda_s$, accounting for errors that persist while the chain holds. The second controls errors in a single round. Applied band by band to the eigenvalues of the holding family, these bounds control the covariance across time. Condition (iv) is the hypothesis of \cref{lem:tails}(ii), which yields the truncation bound.

For tournaments, a balanced estimator is obtained in two steps. The first is a statement about tournaments alone, and it is where the tournament enters estimation.

\begin{lemma}[Logarithmic rooted assignment]
\label{lem:rooted-assignment}
For every tournament on $K\ge2$ actions and every positive probability vector $v$, the set $\cC(v)$ of \cref{def:assignment} contains a point with $\sum_s\rho_s\le3\log(e/v_{\min})$, where $v_{\min}=\min_iv_i$. The linear program in \cref{eq:assignment-lp} has an optimal solution, and its value lies in $[1,3\log(e/v_{\min})]$.
\end{lemma}

The proof, in \cref{app:tournament}, is a weighted, randomized version of the greedy argument that every tournament has a dominating set of logarithmic size. A pivot drawn proportionally to $v$ removes half of the remaining weight in expectation. Every later pivot points to the first pivot, which serves as a common reference for all comparisons.

The following lemma converts an assignment into a balanced estimator by choosing the sampling law and the refresh law as in \cref{alg:prepare}.

\begin{lemma}[Balanced estimator from an assignment]
\label{lem:balanced-estimator}
Let $v,w$ be positive probability vectors with $\tfrac12w\le v\le2w$, let $((Q_s)_s,\rho)\in\cC(v)$ satisfy $\sum_s\rho_s\le3\log(e/v_{\min})$, and let $0<\kappa\le1/32$. Put $q=\rho/\sum_j\rho_j$, $p=(1-4\kappa)w+4\kappa q$, $x=(p+q)/2$, and $\Psi_s=C_wQ_s/p_s$. Then $(p,x)$ is a holding family with $\kappa\lmax<1/4$, $w\le\tfrac87p$, and $\|p-w\|_1\le8\kappa$. Moreover, $(\Psi_s)_s$ is an $L_0$-balanced estimator for $(w,p,x)$ with $L_0=128\log(e/v_{\min})$.
\end{lemma}

The persistence bound uses the first load constraint in \cref{eq:loads} together with $x\ge q/2$. The one-round bound uses the second load constraint together with $p\ge\tfrac78w$. The proof is in \cref{app:tournament}.

\subsection{Chunks and the steering gain}
\label{sec:energy}

The following lemma gives both uses of a chunk's likelihood ratio: it converts observed comparisons into a loss improvement and corrects the estimate from the same trajectory.

\begin{lemma}[Steering and estimation on a chunk]
\label{lem:probing}
Let $P$ be reversible with respect to a positive $p$, with $P(s\mid s)\ge\tfrac12$ for every $s$, and let the scores in \cref{alg:probe-chunk} form a score rule for $P$. Let $M\ge1$, and let $J$ be a chunk, with its rounds relabeled $1,\ldots,|J|$ and deterministic losses $\ell_1,\ldots,\ell_{|J|}\in[0,1]^K$ with sum $L_J$. Then the action process of \cref{alg:probe-chunk} is causal and observable. Its path law is $(1+\beta S_J)\Prb_0$, where $S_J\in[-M,M]$ is the final score sum, so that $\E_{\rm tilt}\Phi\le\tfrac54\E_0\Phi$ for every nonnegative function $\Phi$ of the path. Its expected loss is
\begin{align}
\E_{\rm tilt}\sum_{t=1}^{|J|}\ell_t(X_t)
=\ip p{L_J}-\beta\sum_{1\le t<t'\le|J|}\ip{\ell_t}{(I-P)P^{t'-t-1}\ell_{t'}}_p+\beta\mathsf r,
\label{eq:probing-response}
\end{align}
where $\Omega_{\rm stop}$ is the event that scoring stops and $|\mathsf r|\le|J|^2\Prb_0(\Omega_{\rm stop})$. Moreover, for any positive probability vector $w$ and matrices $\Psi_s$ with $\sum_sp_s\Psi_s=C_w$, the reference sum $G_J=\sum_{t=1}^{|J|}\Psi_{X_t}\ell_t$ and likelihood ratio $R_J=1+\beta S_J$ satisfy
\begin{align}
\E_{\rm tilt}\frac{G_J}{R_J}=C_wL_J,
\qquad
\E_{\rm tilt}\left\|\frac{G_J}{R_J}\right\|_w^2
\le\frac43\E_0\|G_J\|_w^2.
\label{eq:chunk-moments}
\end{align}
\end{lemma}

The vector $P^{t'-t-1}\ell_{t'}$ gives the expected loss in round $t'$ as a function of the action in round $t+1$. Each term in \cref{eq:probing-response} therefore pairs the current loss with a later expected loss through the Dirichlet form of $P$: a score in round $t$ helps to the extent that the current comparisons predict later ones. The proof in \cref{app:probing} identifies the path law by multiplying the transition factors, and then combines the Markov property with condition (iv) of \cref{def:score-rule}.

The block procedure applies \cref{lem:probing} to each chunk, with $P=P_\tau$ at the random rate $\tau$ of \cref{eq:rate-law}. Two rounds $t<t'$ lie in the same chunk with probability $(1-\tau)^{t'-t}$. Averaging the double sum in \cref{eq:probing-response} over the rate and the restarts therefore gives the steering gain of a loss block $\ell_1,\ldots,\ell_h$,
\begin{align}
\Gamma(\ell_{1:h})=\E_\tau\sum_{1\le t<t'\le h}(1-\tau)^{t'-t}\,\ip{\ell_t}{(I-P_\tau)P_\tau^{t'-t-1}\ell_{t'}}_p.
\label{eq:gain}
\end{align}
By \cref{eq:dirichlet}, each inner product equals $\tfrac12\E[(\ell_t(X)-\ell_t(X'))(g(X)-g(X'))]$ for a stationary transition $(X,X')$ of $P_\tau$, where $g=P_\tau^{t'-t-1}\ell_{t'}$ is the chain's prediction of the loss in round $t'$. The gain therefore sums, over pairs of rounds, the correlation between the comparisons in round $t$ and the predicted comparisons in round $t'$, weighted by the probability that no restart separates the two rounds. It is large when losses persist over the time scale $1/\tau$ of some rate. It can be negative when losses alternate, but only slightly: \cref{app:energy} shows that $\Gamma(\ell_{1:h})\ge-\mu h/16$ for losses in $[0,1]^K$, where $\mu$ is the mean rate of \cref{eq:mean-rate}. The gain is a deterministic function of the losses and the holding family; the learner never computes it.

The next corollary sums the loss response over the chunks of a block. A tail bound for accumulated scores makes the stopping error small.

\begin{corollary}[Block loss and switches]
\label{cor:probing-block}
Let $(p,x)$ be a holding family, let $\kappa=2^{-m}$ with an integer $m\ge5$, suppose $\kappa\lmax\le1/4$, and let $h\ge1$ be an integer with $h\kappa\le1/8$. Let $0<\delta<1$ and $M\ge100\log(2/\delta)$. Suppose the scores form a score rule for every $P_\tau$ with $\tau\in[\kappa,1]$. Suppose also that the estimation matrices $(\Psi_s)_{s\in[K]}$ supplied to the block procedure are known at the start of the block and that, for every $s\in[K]$ and every $\ell\in[0,1]^K$, the product $\Psi_s\ell$ is computable from the comparisons revealed by action $s$. For deterministic $\ell_1,\ldots,\ell_h\in[0,1]^K$ with sum $L$, the block procedure in \crefrange{line:draw-rate}{line:truncate} of \cref{alg:main}, run with these inputs, is causal and observable. Its actions $X_1,\ldots,X_h$ and internal switch count $N_{\rm int}=\sum_{t=1}^{h-1}\one\{X_{t+1}\ne X_t\}$ satisfy, with $\mu=\kappa(1+m/2)$,
\begin{align*}
\E\sum_{t=1}^h\ell_t(X_t)\le p^{\mathsf T} L-\beta\,\Gamma(\ell_{1:h})+\beta\delta h^3,
\qquad
\E N_{\rm int}\le\frac{13}8\mu h.
\end{align*}
\end{corollary}

The switch count includes changes at restarts, and the mean rate $\mu$ exceeds the base rate only by the factor $1+m/2$. The proof in \cref{app:probing} sums \cref{lem:probing} over the chunks, and the stopping bound is part of \cref{lem:tails} below.

\subsection{The corrected estimate and its tails}
\label{sec:estimation-lemma}

By \cref{cor:probing-block}, steering lowers the expected loss of a block by $\beta$ times the steering gain $\Gamma(\ell_{1:h})$ of \cref{eq:gain}, up to the stopping error. The next lemma is the estimation bound announced at the start of this section. The second moment it controls includes the squared mean of the estimate, which is the full centered block loss, and errors that are correlated over long stays, so it is not a sum of per-round variances. The lemma uses the feedback only through the balanced-estimator inequalities and the score rule.

\begin{lemma}[Estimation is paid for by steering]
\label{lem:energy}
Let $(p,x)$ be a holding family, let $\kappa=2^{-m}$ with an integer $m\ge5$, suppose $\kappa\lmax\le1/4$, and let $h\ge1$ be an integer with $h\kappa\le1/8$. Let $w$ be a positive probability vector with $w\le2p$, let $(\Psi_s)_s$ satisfy conditions (ii) and (iii) of \cref{def:balanced} with $L_0\ge1$, and suppose the scores form a score rule for each $P_\tau$ with $\tau\in[\kappa,1]$. Run the block procedure in \crefrange{line:draw-rate}{line:truncate} of \cref{alg:main} with these inputs and $M\ge1$, forming the chunk estimates $G_J=\sum_{t\in J}\Psi_{X_t}\ell_t$. For deterministic $\ell_1,\ldots,\ell_h\in[0,1]^K$ with sum $L$, the raw corrected estimate $Z$ of \cref{eq:block-raw} satisfies, with $\mu=\kappa(1+m/2)$,
\begin{align*}
\E Z=C_wL,
\qquad
\E\|Z\|_w^2\le\frac{2^{14}mL_0}{\kappa}\Bigl(\Gamma(\ell_{1:h})+\frac{\mu h}{16}\Bigr).
\end{align*}
\end{lemma}

The bound includes the squared mean $\|C_wL\|_w^2$, which is the cost of updating only between blocks. Its proof in \cref{app:energy} conditions on the rate and restart boundaries and uses independence between chunks. The likelihood correction increases each chunk's second moment by at most a factor $4/3$, leaving three terms: the squared block mean, the reference estimation covariance, and a sum of squared chunk means. All three are bounded by one nonnegative quadratic form, the steering energy, which equals twice the gain plus a diagonal term of order $\mu h$. Its correlations are geometric in the distance between rounds, and a geometric correlation $\chi^{|t-t'|}$ is the probability that rounds $t$ and $t'$ lie in one interval of a random partition that cuts each boundary between rounds independently with probability $1-\chi$. The quadratic forms are therefore expected sums of squared interval sums. By the Cauchy--Schwarz inequality on each partition, the energy controls the squared block total at every rate. The restarts cut less often than the partitions of the energy, which bounds the chunk means.

The covariance comparison uses the common eigenbasis of the holding family, with the eigenvalues grouped into the $m$ dyadic bands $[2^k,2^{k+1})$, $-1\le k\le m-2$. For a loss component in one band, the persistence bound controls the errors that it causes in slower bands, and the one-round bound those in faster bands. Matching the time scale of an error in a faster band with that of the loss requires a faster rate, and the dyadic law, which gives each rate below one a probability proportional to its inverse, makes this shift free. This is the step that lets steering pay for estimation on every relevant time scale, and the step that needs the random rate. Combining the bands of a general loss by the Cauchy--Schwarz inequality costs the factor $m$.

One tail bound controls both uses of the trajectory. It is proved in \cref{app:rare} from an elementary moment bound for accumulated quantities (\cref{lem:occupation}), and it also justifies the truncation in \cref{eq:block-raw}.

\begin{lemma}[Score and estimate tails]
\label{lem:tails}
Let $(p,x)$ be a holding family, let $\kappa=2^{-m}$ with an integer $m\ge5$, suppose $\kappa\lmax\le1/4$, and let $h\ge1$ be an integer with $h\kappa\le1/8$. Let $0<\delta<1$ and $M\ge100\log(2/\delta)$, and suppose the scores form a score rule for each $P_\tau$ with $\tau\in[\kappa,1]$.
\begin{enumerate}
\item For every $\tau\in[\kappa,1]$, a reference chunk begun at any round and ended by fresh restart coins of probability $\tau$ or by the block's end stops scoring with probability at most $\delta$.
\item If $\|\Psi_s\ell\|_\infty\le L_0\lambda_s$ for every $s$ and every $\ell\in[0,1]^K$, with $L_0\ge1$, then the raw corrected estimate of the block procedure satisfies $\|Z\|_\infty\le hL_0/(3\kappa)$ with probability one, and
\begin{align*}
\Prb\left(\|Z\|_\infty>\frac{2ML_0}{\kappa}\right)\le\frac\delta2.
\end{align*}
\end{enumerate}
\end{lemma}

The first assertion averages over future restart coins; it does not condition on the eventual chunk length. The second concerns the entire block under the played, tilted law. Truncation changes only the estimate supplied to exponential weights and therefore preserves the loss improvement.

\subsection{Proof of the block guarantee}
\label{sec:block-proof}

We now combine the lemmas of this section with the data prepared by \proc{PrepareBlock}.

\begin{proof}[Proof of \cref{prop:block}]
We verify the prepared inputs, use \cref{lem:energy,lem:tails} to control the corrected output, and combine its second moment with the loss improvement and switching bound of \cref{cor:probing-block}. Condition throughout on $w$, its preparation, and the deterministic loss block.

\emph{Prepared inputs.} The assignment has anchor distribution $v=w$, so $\tfrac12w\le v\le2w$ and $\log(e/v_{\min})\le\Lambda$. By \cref{lem:rooted-assignment}, its total source weight is at most $3\log(e/v_{\min})$. By \cref{lem:balanced-estimator}, $\kappa\lmax<1/4$, $w\le2p$, $\|p-w\|_1\le8\kappa$, and $\Psi_s=C_wQ_s/p_s$ is an $L_0$-balanced estimator with $L_0=128\log(e/v_{\min})\le2^8\Lambda$. The tournament scores form a score rule for every $P_\tau$ by \cref{lem:tournament-scores,lem:holding-chains}. Finally, $h\kappa\le1/8$.

\emph{Corrected output.} By \cref{lem:energy}, $\E Z=C_wL$. Since $1/\eta=2^{25}Mm\Lambda/\kappa\ge2ML_0/\kappa$, \cref{lem:tails} gives
\begin{align*}
\xi=\E[\one\{\eta\|Z\|_\infty>1\}Z],
\qquad
\|\xi\|_\infty\le\frac{\delta hL_0}{6\kappa}\le\frac{2^6\delta h\Lambda}{\kappa}.
\end{align*}
Thus $\E\widehat G=C_wL-\xi$ and $\eta\|\widehat G\|_\infty\le1$, which proves \cref{eq:block-estimate,eq:block-bias}. Discarding large outputs can only decrease the second moment, so
\begin{align}
\E\|\widehat G\|_w^2\le\E\|Z\|_w^2
\le\frac{2^{14}mL_0}{\kappa}\Bigl(\Gamma+\frac{\mu h}{16}\Bigr).
\label{eq:returned-second-moment}
\end{align}

\emph{Loss and switches.} By \cref{cor:probing-block}, the expected loss is at most $p^{\mathsf T} L-\beta\Gamma+\beta\delta h^3$, and the expected number of internal switches is at most $13\mu h/8$. Since $p-w$ has coordinates summing to zero and $L\in[0,h]^K$, subtracting $(h/2)\one$ gives $\ip{p-w}L\le(h/2)\|p-w\|_1\le4\kappa h$. Consequently
\begin{align*}
\E C\le\ip wL+1+4\kappa h+\frac{13}8\mu h-\beta\Gamma+\beta\delta h^3.
\end{align*}
The switch count, including the unit allowance for block entry, is at most $1+13\mu h/8\le1+2\mu h$.

\emph{Payment for the update.} By \cref{eq:returned-second-moment}, $\beta=1/(4M)$, and $L_0\le2^8\Lambda$,
\begin{align*}
\beta\Gamma
\ge\frac{\beta\kappa}{2^{14}mL_0}\E\|\widehat G\|_w^2-\frac{\beta\mu h}{16}
\ge\frac{\kappa}{2^{24}Mm\Lambda}\E\|\widehat G\|_w^2-\frac{\beta\mu h}{16}.
\end{align*}
Substitute this into the cost bound and use $\kappa\le\mu$, $\beta\le1$, and $4+13/8+1/16<6$. This gives \cref{eq:block-cost} and completes the proof.
\end{proof}

\section{Graphs that are not strongly revealing}
\label{sec:regimes}

This section completes the characterization. For every observable graph,  we can combine a low-switching algorithm with uniform exploration on an observing set $\mathcal D$ to obtain regret of order $(|\mathcal D|\log K)^{1/3}T^{2/3}$ (\cref{sec:explore}). Conversely, every graph that is not strongly revealing forces regret of order $\tilde\Omega(T^{2/3})$, and for observable graphs the lower bound grows with the domination number as the upper bound does, up to logarithmic factors (\cref{sec:lower}). \Cref{sec:char-proof} assembles the proof of the complete characterization as in \cref{thm:characterization}.

\subsection{Observable graphs: lazy exploration}
\label{sec:explore}

An observable graph need not let the learner observe anything useful at the actions it rests on, but every action is observed by some action. The learner therefore explores in separate rounds, and in the remaining rounds it follows exponential weights with few switches by the shrinking dartboard method of Geulen, V\"ocking, and Winkler \citep{GeulenVW10}, which we recall first.

\paragraph{Shrinking dartboard.}
Let $g_1,\ldots,g_T\in[0,\infty)^K$ be nonnegative loss vectors, which may be random. Let $\pi_t$ be the exponential-weights distribution on $[K]$ with learning rate $\eta$, starting from the uniform distribution and updated after every round by $\pi_{t+1,i}\propto\pi_{t,i}e^{-\eta g_t(i)}$. The method keeps an arm $\widetilde I_t\in[K]$, with $\widetilde I_1\sim\pi_1$. After round $t$ it keeps the arm with probability $e^{-\eta g_t(\widetilde I_t)}$ and otherwise draws $\widetilde I_{t+1}$ afresh from $\pi_{t+1}$. This is a lazy Markov chain whose transitions change from round to round, and its marginal law is $\pi_t$ in every round: if $\widetilde I_t\sim\pi_t$, then $\widetilde I_{t+1}=j$ with probability
\begin{align*}
\pi_{t,j}e^{-\eta g_t(j)}+\pi_{t+1,j}\sum_i\pi_{t,i}\bigl(1-e^{-\eta g_t(i)}\bigr)=\pi_{t+1,j},
\end{align*}
because $\pi_{t+1,j}$ is proportional to $\pi_{t,j}e^{-\eta g_t(j)}$ with normalizing constant $\sum_i\pi_{t,i}e^{-\eta g_t(i)}$. The arm therefore follows exponential weights in law while it switches with probability at most $\eta g_t(\widetilde I_t)$ per round. The calculation above is justified under the conditional independence
assumption in the following lemma.

\begin{lemma}[Shrinking dartboard]
\label{lem:sdb}
Let $\eta>0$, and let $g_1,\ldots,g_T\in[0,\infty)^K$ be random loss vectors such that each $g_t$ is independent of $\widetilde I_t$ given $g_1,\ldots,g_{t-1}$. For each $t<T$, conditional on $(g_1,\ldots,g_t,\widetilde I_t)$, the keeping coin and the resampling draw are independent, with respective distributions $\operatorname{Bernoulli}(e^{-\eta g_t(\widetilde I_t)})$ and $\pi_{t+1}$. Assume additionally that $\E g_t(i)<\infty$ for every $t\in[T]$ and $i\in[K]$. Then, given $g_1,\ldots,g_t$, the arm $\widetilde I_t$ has law $\pi_t$. Moreover, for every $i\in[K]$,
\begin{align*}
\E\sum_{t=1}^Tg_t(\widetilde I_t)-\E\sum_{t=1}^Tg_t(i)\le\frac{\log K}\eta+\frac\eta2\,\E\sum_{t=1}^T\ip{\pi_t}{g_t^2},
\qquad
\E\sum_{t=2}^T\one\{\widetilde I_t\ne\widetilde I_{t-1}\}\le\eta\,\E\sum_{t=1}^T\ip{\pi_t}{g_t},
\end{align*}
where $g_t^2$ is the coordinatewise square.
\end{lemma}

\begin{proof}
The proof of this result is a simple adaptation of Geulen~et~al.~\citep{GeulenVW10}, and  consists of three steps: (i) we verify the marginal identity by induction, (ii) we then apply the potential argument of exponential weights pathwise, and finally (iii) bound the probability of a switch. Recall that $\pi_t$ is a function of $g_1,\ldots,g_{t-1}$.

\emph{(i) Marginal identity.} We show that, given $g_1,\ldots,g_t$, the arm $\widetilde I_t$ has law $\pi_t$. For $t=1$, the arm is drawn from $\pi_1$, and $g_1$ is independent of it. Suppose the claim holds for $t$. Given $g_1,\ldots,g_t$ and $\widetilde I_t=i$, the method keeps $i$ with probability $e^{-\eta g_t(i)}$ and otherwise draws from $\pi_{t+1}$, using fresh randomness. Hence, given $g_1,\ldots,g_t$, the arm $\widetilde I_{t+1}$ equals $j$ with probability
\begin{align*}
\pi_{t,j}e^{-\eta g_t(j)}+\pi_{t+1,j}\sum_i\pi_{t,i}\bigl(1-e^{-\eta g_t(i)}\bigr)=\pi_{t+1,j},
\end{align*}
by the identity displayed in \cref{sec:explore}. Since $g_{t+1}$ is independent of $\widetilde I_{t+1}$ given $g_1,\ldots,g_t$, the same law holds given $g_1,\ldots,g_{t+1}$.

\emph{Regret.} By the marginal identity, $\E g_t(\widetilde I_t)=\E\ip{\pi_t}{g_t}$. Put $\mathcal Z_t=\sum_j\pi_{t,j}e^{-\eta g_t(j)}$. For $z\ge0$, $e^{-z}\le1-z+z^2/2$, and $\log(1+y)\le y$ for $y>-1$, so
\begin{align*}
\log\mathcal Z_t\le-\eta\ip{\pi_t}{g_t}+\frac{\eta^2}2\ip{\pi_t}{g_t^2}.
\end{align*}
The update gives, for every $i\in[K]$,
\begin{align*}
0\ge\log\pi_{T+1,i}=-\log K-\eta\sum_{t=1}^Tg_t(i)-\sum_{t=1}^T\log\mathcal Z_t.
\end{align*}
Combining the two displays and dividing by $\eta$ gives the pathwise bound $\sum_t\ip{\pi_t}{g_t}-\sum_tg_t(i)\le\log K/\eta+(\eta/2)\sum_t\ip{\pi_t}{g_t^2}$. Taking expectations proves the first inequality.

\emph{Switches.} Given $g_1,\ldots,g_t$ and $\widetilde I_t$, a switch between rounds $t$ and $t+1$ requires the keeping coin to fail, which has probability $1-e^{-\eta g_t(\widetilde I_t)}\le\eta g_t(\widetilde I_t)$. Taking expectations and using the marginal identity bounds the probability of a switch by $\eta\E\ip{\pi_t}{g_t}$, and summing over $t$ proves the second inequality.
\end{proof}

\paragraph{Exploration.}
Let $\mathcal D$ be an observing set, so that every action is an out-neighbor of some member of $\mathcal D$ (\cref{sec:revealing}). For $i\in[K]$, let $\mathcal D_i=\{s\in\mathcal D:(s,i)\in E\}$, which is nonempty. In each round, the learner tosses an independent coin $\mathsf c_t$ with success probability $\kappa$. On success it plays an action $I^\circ_t$ drawn uniformly from $\mathcal D$, independently of everything else; otherwise it plays the arm $\widetilde I_t$ of a shrinking dartboard run on all actions with the estimates
\begin{align}
\hat\ell_t(i)=\frac{|\mathcal D|}{\kappa\,|\mathcal D_i|}\,\mathsf c_t\one\{I^\circ_t\in\mathcal D_i\}\,\ell_t(i),
\qquad i\in[K].
\label{eq:explore-estimate}
\end{align}
These estimates are computable from the observations, nonnegative, and unbiased. They depend only on the coin and the draw of round $t$, and not on the dartboard's arm, which is what the marginal identity of the shrinking dartboard requires.

\begin{algorithm}[htbp]
\caption{\proc{LazyExploration}}
\label{alg:explore}
\begin{algorithmic}[1]
\REQUIRE Observing set $\mathcal D$; exploration rate $\kappa$; learning rate $\eta$.
\STATE Set $\pi_1$ uniform on $[K]$ and draw $\widetilde I_1\sim\pi_1$.
\FOR{$t=1,\ldots,T$}
\STATE Toss an independent coin $\mathsf c_t$ with success probability $\kappa$. \IF{$\mathsf c_t=1$}
    \STATE play $I^\circ_t$ drawn uniformly from $\mathcal D$
\ELSE
    \STATE play $\widetilde I_t$.
\ENDIF
\STATE Form $\hat\ell_t$ as in \cref{eq:explore-estimate} and update $\pi_{t+1,i}\propto\pi_{t,i}e^{-\eta\hat\ell_t(i)}$.
\STATE Keep $\widetilde I_{t+1}=\widetilde I_t$ with probability $e^{-\eta\hat\ell_t(\widetilde I_t)}$, and otherwise draw $\widetilde I_{t+1}\sim\pi_{t+1}$.
\ENDFOR
\end{algorithmic}
\end{algorithm}

The following theorem gives the upper bound of the middle regime. Its dependence on the graph is through the size of an observing set, which can be taken equal to the domination number, and it is of the familiar order $T^{2/3}$ of bandits with switching costs.

\begin{theorem}[Observable graphs]
\label{thm:explore}
Let $\mathcal G$ be an observable graph on $K\ge2$ actions with an observing set $\mathcal D$, and let $T\ge1$. With $\kappa=(|\mathcal D|\log K/(9T))^{1/3}$ and $\eta=\sqrt{\kappa\log K/(T|\mathcal D|)}$, if $\kappa\le\tfrac12$, \cref{alg:explore} satisfies
\begin{align*}
\Reg_T\le3^{4/3}\,(|\mathcal D|\log K)^{1/3}\,T^{2/3}.
\end{align*}
If $\kappa>\tfrac12$, a fixed action has regret at most $(|\mathcal D|\log K)^{1/3}T^{2/3}$.
\end{theorem}

Exploration costs at most one unit of loss and two switches per exploration round, and the variance of the estimates, of order $|\mathcal D|/\kappa$ per round, enters the regret of exponential weights. Balancing the exploration cost $\kappa T$ against the resulting regret of exponential weights, of order $\sqrt{T|\mathcal D|\log K/\kappa}$, gives the rate.

\begin{proof}
We decompose the learner's cost into that of the dartboard routine and that of exploration, and apply \cref{lem:sdb} to the estimates. In every round, $\ell_t(I_t)\le\ell_t(\widetilde I_t)+\mathsf c_t$, and the learner's switches are switches of the dartboard's arm or occur in and out of exploration rounds, so $N_T\le\widetilde N_T+2\sum_t\mathsf c_t$, where $\widetilde N_T$ is the number of switches of $\widetilde I_t$.

The estimate $\hat\ell_t$ is a function of $\mathsf c_t$, $I^\circ_t$, and the deterministic $\ell_t$, which are independent of $\widetilde I_t$ and of $\hat\ell_1,\ldots,\hat\ell_{t-1}$. It also satisfies $0\le\hat\ell_t(i)\le|\mathcal D|/\kappa$, so the integrability hypothesis of \cref{lem:sdb} holds. Hence \cref{lem:sdb} applies with $g_t=\hat\ell_t$: given $\hat\ell_1,\ldots,\hat\ell_t$, the arm $\widetilde I_t$ has law $\pi_t$, which is a function of $\hat\ell_1,\ldots,\hat\ell_{t-1}$. Since $\Prb(\mathsf c_t=1,I^\circ_t\in\mathcal D_i)=\kappa|\mathcal D_i|/|\mathcal D|$, each estimate satisfies
\[
    \E\hat\ell_t(i)=\ell_t(i), \quad \E\hat\ell_t(i)^2=\frac{|\mathcal D|\ell_t(i)^2}{\kappa|\mathcal D_i|}\le\frac{|\mathcal D|}{\kappa},
\]
and both remain true conditionally on $\hat\ell_1,\ldots,\hat\ell_{t-1}$. Consequently $\E\ell_t(\widetilde I_t)=\E\ip{\pi_t}{\ell_t}=\E\ip{\pi_t}{\hat\ell_t}=\E\hat\ell_t(\widetilde I_t)$, and \cref{lem:sdb} gives, for every action $i$,
\begin{align*}
\E\sum_t\ell_t(\widetilde I_t)-\sum_t\ell_t(i)\le\frac{\log K}\eta+\frac{\eta T|\mathcal D|}{2\kappa},
\qquad
\E\widetilde N_T\le\eta\,\E\sum_t\ip{\pi_t}{\hat\ell_t}\le\eta T.
\end{align*}
Adding the exploration cost $\E\sum_t3\mathsf c_t=3\kappa T$, we obtain
\begin{align*}
\Reg_T\le\frac{\log K}\eta+\eta T\Bigl(1+\frac{|\mathcal D|}{2\kappa}\Bigr)+3\kappa T\le\frac{\log K}\eta+\frac{\eta T|\mathcal D|}\kappa+3\kappa T,
\end{align*}
where the second step uses $\kappa\le\tfrac12\le|\mathcal D|/2$. The stated $\eta$ makes the first two terms equal to $\sqrt{T|\mathcal D|\log K/\kappa}$ each, and the stated $\kappa$ gives $T|\mathcal D|\log K=9\kappa^3T^2$. The bound is therefore $6\kappa T+3\kappa T=9\kappa T=3^{4/3}(|\mathcal D|\log K)^{1/3}T^{2/3}$.

If $\kappa>\tfrac12$, then $T<\tfrac89|\mathcal D|\log K$, and a fixed action has regret at most $T=T^{1/3}T^{2/3}\le(|\mathcal D|\log K)^{1/3}T^{2/3}$.
\end{proof}

\subsection{Lower bounds}
\label{sec:lower}

The lower bounds of observable graphs which do not enjoy strong revealabity must capture two effects. A single pair of actions without an observing endpoint forces regret of order $T^{2/3}$, whatever the other actions observe, and many actions that can be observed only a few at a time force a rate that grows with their number, as for bandits with switching costs \citep{DekelDingKorenPeres2014}. One construction captures both effects.

By \cref{lem:structure}, a strongly observable graph that is not strongly revealing contains a bandit pair or a loopless pair. For undirected graphs with self-loops, a bandit pair is a pair of nonadjacent actions, and Arora, Marinov, and Mohri \citep{AroraMM19} design a hard instance by giving every other action loss one and letting the pair face the construction of Dekel, Ding, Koren, and Peres \citep{DekelDingKorenPeres2014}. In that construction the two members share one random walk, so an action that observes both members reveals their loss difference in a single round. Moreover, in a loopless pair each member observes only the other. The following theorem applies to every set of actions in which every member observes at most one member, with no assumption on what the other actions observe beyond the number of members that each of them observes. The proof is deferred to \cref{app:lower}.

\begin{proposition}\label{thm:lower}
Let $\mathcal G$ be a directed graph on $K\ge2$ actions, and let $\mathcal S$ be a set of at least two actions in which every member observes at most one member of $\mathcal S$. Let $\mathsf m\ge1$ be such that every action observes at most $\mathsf m$ members of $\mathcal S$, and let $T\ge2$. For every learner there is an oblivious loss sequence on which
\begin{align*}
\Reg_T\ge\frac{T^{2/3}}{256\log_2(2T)}\Bigl(\frac{\min\{|\mathcal S|,T\}}{\mathsf m}\Bigr)^{1/3}.
\end{align*}
\end{proposition}
\begin{proof}[Proof Sketch]
    In the lower bound construction, we give every action outside $\mathcal S$ loss one, which costs at least a quarter of a unit per round against the members of $S$. In round $t$, we start from the raw loss $\tfrac12+\Xi_{0,t}+\iota_t\Xi_{v,t}$ of a member $v$, subtract a gap $\vartheta$ of order $(|\mathcal S|/(\mathsf mT))^{1/3}/\log T$ if $v$ is an unknown better member $\mathsf j$, and then clip to $[\tfrac14,\tfrac34]$. Here $\Xi_t$ is a version of the multiscale random walk of Dekel~et~al.~\cite{DekelDingKorenPeres2014} with one coordinate common to all members of $S$ and one coordinate for each member, and $\iota_t$ is a fresh random sign. Before clipping, the current independent sign makes the walk contribution to the expected loss difference between the chosen member and $\mathsf j$ vanish. The raw-loss comparison therefore contributes $\vartheta$ per round spent on a member other than $\mathsf j$. In \Cref{app:lower} we control the correction caused by clipping via \cref{lem:clip,lem:lower-regret}.

    For the information, the observations of every round can be reconstructed from the sign and a shifted copy of the walk: a round at a member shifts the common coordinate if the member that it observes is $\mathsf j$, and a round outside $\mathcal S$ shifts the coordinate of $\mathsf j$ if it observes $\mathsf j$. Along the parent edges of the multiscale tree, the environment of $\mathsf j$ then differs from a reference environment without a better member only where the shift changes, which happens $O(\log T)$ times per switch or per round outside $\mathcal S$, and each round outside $\mathcal S$ concerns at most $\mathsf m$ candidates for $\mathsf j$. Hence either the learner pays for $\Omega(\vartheta T)$ switches or rounds outside $\mathcal S$, or it cannot tell most environments from the reference environment and spends a constant fraction of the rounds on members other than $\mathsf j$.
\end{proof}

    To get a better understanding of this result, we present two special cases. First, a pair of distinct actions which has no endpoint that observes the losses of both exactly when each of the two observes at most one of them. Every graph that is not strongly revealing therefore contains a pair to which the theorem applies with $\mathsf m=2$, and this pair forces regret at least $T^{2/3}/(256\log_2(2T))$, whatever the other actions observe. This shows that the condition of \cref{thm:main} cannot be relaxed. In strongly observable graphs these pairs are the bandit pairs and the loopless pairs, as the proof of \cref{lem:structure} shows, and in weakly observable graphs they include every weakly observable action together with an action that does not observe it. Second, for bandit feedback, $\mathcal S=[K]$ and $\mathsf m=1$ give the lower bound of order $K^{1/3}T^{2/3}/\log T$ of \citep{DekelDingKorenPeres2014} when $K\le T$. In general, the bound grows with the ratio of the size of $\mathcal S$ to the number of its members that a single action observes.

Large sets that satisfy the assumption of \cref{thm:lower} with a small $\mathsf m$ exist in every graph with a large domination number. Alon, Cesa-Bianchi, Dekel, and Koren \citep[Lemma~8]{AlonCDK15} find such sets among the weakly observable actions, and their proof uses that these actions have no self-loops: it bounds the average number of sampled actions that a sampled action observes, which is at least one when every action observes itself. Since the domination number can be large even when every action observes its own loss, as for bandit feedback, we need a variant for all actions, which disregards self-loops when it removes edges between sampled actions. Its proof, in \cref{app:few-observers}, removes actions until no action observes a large fraction of the remaining ones, and then samples a small random subset of the remaining actions.

\begin{lemma}[Sets with few observers]
\label{lem:few-observers}
Let $\mathcal G$ be an observable graph on $K\ge2$ actions with domination number $\gamma(\mathcal G)\ge32\log K$. There is a set of at least $\gamma(\mathcal G)/(16\log K)$ actions, with no edges between distinct members, such that every action observes at most $1+\log K$ of them.
\end{lemma}

A set with no edges between distinct members satisfies the assumption of \cref{thm:lower}, because each member observes at most itself. With the lemma, the two cases of \cref{thm:lower} combine into a lower bound in terms of the domination number.

\begin{theorem}[Domination lower bound]
\label{cor:domination}
Let $\mathcal G$ be an observable graph on $K\ge2$ actions that is not strongly revealing, let $\gamma=\gamma(\mathcal G)$, and let $T\ge2$. Then
\begin{align*}
\Reg^\star_T(\mathcal G)\ge\frac{T^{2/3}}{1024\log_2(2T)}\max\Bigl\{1,\frac{\min\{\gamma,T\}}{\log^2(eK)}\Bigr\}^{1/3}.
\end{align*}
\end{theorem}

\begin{proof}
Since $\mathcal G$ is not strongly revealing, \cref{thm:lower} applied to a pair without an observing endpoint gives $\Reg^\star_T(\mathcal G)\ge T^{2/3}/(256\log_2(2T))$. This proves the claim if $\gamma<32\log K$. Indeed, with $x=\log(eK)$, we then have $\min\{\gamma,T\}/\log^2(eK)<32(x-1)/x^2\le8$, because $(x-2)^2\ge0$, so the right side of the claim is at most $2T^{2/3}/(1024\log_2(2T))$, which is less than the bound for the pair. If $\gamma\ge32\log K$, \cref{lem:few-observers} gives a set $\mathcal S$ of at least $\gamma/(16\log K)\ge2$ actions, with no edges between distinct members, such that every action observes at most $\log(eK)$ of its members. Since $\min\{|\mathcal S|,T\}\ge\min\{\gamma,T\}/(16\log(eK))$, \cref{thm:lower} with $\mathsf m=\log(eK)$ gives
\begin{align*}
\Reg^\star_T(\mathcal G)\ge\frac{T^{2/3}}{256\log_2(2T)}\Bigl(\frac{\min\{|\mathcal S|,T\}}{\log(eK)}\Bigr)^{1/3}\ge\frac{T^{2/3}}{256\log_2(2T)}\Bigl(\frac{\min\{\gamma,T\}}{16\log^2(eK)}\Bigr)^{1/3}.
\end{align*}
As $16^{1/3}<4$, this bound and the bound for the pair prove the claim.
\end{proof}

Together with \cref{thm:explore}, applied to an observing set of size $\gamma$, and with the regret at most $T$ of a learner that never switches, the corollary determines the minimax regret up to a factor of order $\log T\cdot\log K$. Arora, Marinov, and Mohri \citep{AroraMM19} prove a lower bound of order $\gamma^{1/3}T^{2/3}/\log T$ for disjoint unions of star graphs; the corollary applies to every observable graph that is not strongly revealing.

\subsection{Proof of the characterization}
\label{sec:char-proof}

We then have all the ingredients to prove our characterization. The upper bounds come from \cref{thm:main,thm:explore}, and the lower bounds from \cref{cor:domination} and two elementary constructions, which we give in full.

\begin{proof}[Proof of \cref{thm:characterization}]
We treat the three regimes in turn.

\emph{Strongly revealing graphs.} The upper bound is \cref{thm:main}. For the lower bound, which holds for every graph, let $\mathsf b_1,\ldots,\mathsf b_T$ be independent and uniform on $\{0,1\}$, give action $1$ loss $\mathsf b_t$, action $2$ loss $1-\mathsf b_t$, and every other action loss one. Every action has expected loss at least $\tfrac12$ in round $t$ given the past, so every learner has expected loss at least $T/2$, while the better of actions $1$ and $2$ has loss $T/2-|\sum_t(2\mathsf b_t-1)|/2$. Switching costs are nonnegative, so the expected regret, averaged over the losses, is at least half the first absolute moment of $\sum_t(2\mathsf b_t-1)$. This sum has second moment $T$ and fourth moment $3T^2-2T\le3T^2$, and Hölder's inequality bounds the second moment by the first absolute moment to the power $2/3$ times the fourth moment to the power $1/3$. Hence the first absolute moment is at least $\sqrt{T/3}$, and some loss sequence in the support gives regret at least $\sqrt{T/3}/2\ge\sqrt T/4$.

\emph{Observable graphs that are not strongly revealing.} The upper bound $5(\gamma\log K)^{1/3}T^{2/3}$ is \cref{thm:explore} applied to an observing set of size $\gamma=\gamma(\mathcal G)$, since $3^{4/3}<5$, and the upper bound $T$ holds because a learner that never switches has regret at most $T$. The lower bound is \cref{cor:domination}.

\emph{Graphs that are not observable.} Some action $i$ is observed by no action, so its losses never enter the feedback. Fix another action $j$ and consider two loss sequences that are constant over time: in both, $j$ has loss $\tfrac12$ and every action other than $i$ and $j$ has loss one, while $i$ has loss $0$ in the first and $1$ in the second. The feedback, and hence the law of the learner's actions, is the same under both sequences. In the first, every round not spent at $i$ costs at least $\tfrac12$ more than $i$; in the second, every round not spent at $j$ costs at least $\tfrac12$ more than $j$. The two regrets therefore sum to at least $T/2$, and one of them is at least $T/4$.
\end{proof}

\subsection*{AI methodology}
We used OpenAI's GPT-6 Astra and Anthropic's Claude Opus 5.5 to brainstorm and refine algorithmic ideas, develop preliminary proof arguments, assist with literature searches, and help prepare figures.
We also used these tools to assist with \LaTeX{} formatting and improve grammar, punctuation, and phrasing throughout the manuscript.
All AI-assisted material incorporated into the manuscript was subsequently checked by the authors and revised where appropriate.

OpenAI's GPT-6 Astra was also used to produce the accompanying Lean formalization.
Its correspondence with the manuscript was assessed using separate instances of GPT-6 Astra, distinct from those used to produce the formalization.
The authors thoroughly reviewed the resulting reports.

The authors take full responsibility for the paper and its accompanying materials.

\phantomsection
\addcontentsline{toc}{section}{References}
\bibliographystyle{alpha}
\bibliography{references,references-additional}

@article{CesaBianchiCesariDellaVecchia2024,
  author  = {Cesa-Bianchi, Nicol{\`o} and Cesari, Tommaso and Della Vecchia, Riccardo},
  title   = {Cooperative Online Learning with Feedback Graphs},
  journal = {Transactions on Machine Learning Research},
  year    = {2024},
  url     = {https://openreview.net/forum?id=PtNyIboDIG}
}

@article{BartokFPRS14,
  author  = {G{\'a}bor Bart{\'o}k and Dean P. Foster and D{\'a}vid P{\'a}l and Alexander Rakhlin and Csaba Szepesv{\'a}ri},
  title   = {Partial Monitoring---Classification, Regret Bounds, and Algorithms},
  journal = {Mathematics of Operations Research},
  volume  = {39},
  number  = {4},
  pages   = {967--997},
  year    = {2014}
}

@article{AntosBPS13,
  author  = {Andr{\'a}s Antos and G{\'a}bor Bart{\'o}k and D{\'a}vid P{\'a}l and Csaba Szepesv{\'a}ri},
  title   = {Toward a Classification of Finite Partial-Monitoring Games},
  journal = {Theoretical Computer Science},
  volume  = {473},
  pages   = {77--99},
  year    = {2013}
}

@inproceedings{CohenHK16,
  author    = {Alon Cohen and Tamir Hazan and Tomer Koren},
  title     = {Online Learning with Feedback Graphs Without the Graphs},
  booktitle = {Proceedings of the 33rd International Conference on Machine Learning},
  series    = {Proceedings of Machine Learning Research},
  volume    = {48},
  pages     = {811--819},
  year      = {2016}
}

@inproceedings{EspositoFHC22,
  author    = {Emmanuel Esposito and Federico Fusco and Dirk van der Hoeven and Nicol{\`o} Cesa-Bianchi},
  title     = {Learning on the Edge: Online Learning with Stochastic Feedback Graphs},
  booktitle = {Advances in Neural Information Processing Systems 35},
  pages     = {34776--34788},
  year      = {2022}
}

@inproceedings{HarrisPW23,
  author    = {Keegan Harris and Chara Podimata and Zhiwei Steven Wu},
  title     = {Strategic Apple Tasting},
  booktitle = {Advances in Neural Information Processing Systems 36},
  pages     = {79918--79945},
  year      = {2023}
}

@inproceedings{AnavaHM15,
  author    = {Oren Anava and Elad Hazan and Shie Mannor},
  title     = {Online Learning for Adversaries with Memory: Price of Past Mistakes},
  booktitle = {Advances in Neural Information Processing Systems 28},
  pages     = {784--792},
  year      = {2015}
}

@article{YueBKJ12,
  author  = {Yisong Yue and Josef Broder and Robert Kleinberg and Thorsten Joachims},
  title   = {The {$K$}-armed Dueling Bandits Problem},
  journal = {Journal of Computer and System Sciences},
  volume  = {78},
  number  = {5},
  pages   = {1538--1556},
  year    = {2012}
}

@inproceedings{SahaKM21,
  author    = {Aadirupa Saha and Tomer Koren and Yishay Mansour},
  title     = {Adversarial Dueling Bandits},
  booktitle = {Proceedings of the 38th International Conference on Machine Learning},
  series    = {Proceedings of Machine Learning Research},
  volume    = {139},
  pages     = {9235--9244},
  year      = {2021}
}

@article{BlumB00,
  author  = {Avrim Blum and Carl Burch},
  title   = {On-line Learning and the Metrical Task System Problem},
  journal = {Machine Learning},
  volume  = {39},
  number  = {1},
  pages   = {35--58},
  year    = {2000}
}

@InProceedings{koren2017bandits,
  title = 	 {Bandits with Movement Costs and Adaptive Pricing},
  author = 	 {Koren, Tomer and Livni, Roi and Mansour, Yishay},
  booktitle = 	 {Proceedings of the 2017 Conference on Learning Theory},
  pages = 	 {1242--1268},
  year = 	 {2017},
  editor = 	 {Kale, Satyen and Shamir, Ohad},
  volume = 	 {65},
  series = 	 {Proceedings of Machine Learning Research},
  month = 	 {07--10 Jul},
  publisher =    {PMLR}
}

@article{CesaBianchiLugosiStoltz2005,
  author  = {Nicol{\`o} Cesa-Bianchi and G{\'a}bor Lugosi and Gilles Stoltz},
  title   = {Minimizing Regret with Label Efficient Prediction},
  journal = {IEEE Transactions on Information Theory},
  volume  = {51},
  number  = {6},
  pages   = {2152--2162},
  year    = {2005},
  doi     = {10.1109/TIT.2005.847729}
}

@inproceedings{HelmboldLL1992,
  author       = {David P. Helmbold and
                  Nick Littlestone and
                  Philip M. Long},
  title        = {Apple Tasting and Nearly One-sided Learning},
  booktitle    = {Proceedings of the 33rd Annual Symposium on Foundations of Computer Science},
  publisher    = {IEEE Computer Society},
  pages        = {493--502},
  year         = {1992}
}

@article{HelmboldLL2000,
  author       = {David P. Helmbold and
                  Nick Littlestone and
                  Philip M. Long},
  title        = {Apple Tasting},
  journal      = {Inf. Comput.},
  volume       = {161},
  number       = {2},
  pages        = {85--139},
  year         = {2000}
}

@inproceedings{Cesa-BianchiDS13,
  author       = {Nicol{\`{o}} Cesa{-}Bianchi and
                  Ofer Dekel and
                  Ohad Shamir},
  title        = {Online Learning with Switching Costs and Other Adaptive Adversaries},
  booktitle    = {{NIPS}},
  pages        = {1160--1168},
  year         = {2013}
}

@inproceedings{RamanSubediRamanTewari2024,
  title={Apple tasting: Combinatorial dimensions and minimax rates},
  author={Raman, Vinod and Subedi, Unique and Raman, Ananth and Tewari, Ambuj},
  booktitle={The Thirty Seventh Annual Conference on Learning Theory},
  pages={4358--4380},
  year={2024},
  organization={PMLR}
}

@article{CesaBianchiLugosiStoltz2006,
  title={Regret minimization under partial monitoring},
  author={Cesa-Bianchi, Nicol{\`o} and Lugosi, G{\'a}bor and Stoltz, Gilles},
  journal={Mathematics of Operations Research},
  volume={31},
  number={3},
  pages={562--580},
  year={2006},
  publisher={INFORMS}
}

@inproceedings{ChenHuangLiZhang2021,
 title = {Understanding Bandits with Graph Feedback},
 author = {Chen, Houshuang and Huang, Zengfeng and Li, Shuai and Zhang, Chihao},
 booktitle = {Advances in Neural Information Processing Systems},
 pages = {24659--24669},
 volume = {34},
 year = {2021}
}

@InProceedings{DannWeiZimmert2023,
  title = 	 {A Blackbox Approach to Best of Both Worlds in Bandits and Beyond},
  author =       {Dann, Chris and Wei, Chen-Yu and Zimmert, Julian},
  booktitle = 	 {Proceedings of Thirty Sixth Conference on Learning Theory},
  pages = 	 {5503--5570},
  year = 	 {2023},
  volume = 	 {195},
}

@InProceedings{KocakCarpentier2023,
  title = 	 {Online Learning with Feedback Graphs: The True Shape of Regret},
  author =       {Koc\'{a}k, Tom\'{a}\v{s} and Carpentier, Alexandra},
  booktitle = 	 {Proceedings of the 40th International Conference on Machine Learning},
  pages = 	 {17260--17282},
  year = 	 {2023},
  volume = 	 {202},
}

@InProceedings{ChenHeZhang2024,
  title = 	 {On Interpolating Experts and Multi-Armed Bandits},
  author =       {Chen, Houshuang and He, Yuchen and Zhang, Chihao},
  booktitle = 	 {Proceedings of the 41st International Conference on Machine Learning},
  pages = 	 {6776--6802},
  year = 	 {2024},
  volume = 	 {235},
}

@inproceedings{DekelDingKorenPeres2014,
  author       = {Ofer Dekel and
                  Jian Ding and
                  Tomer Koren and
                  Yuval Peres},
  title        = {Bandits with switching costs: $T^{2/3}$ regret},
  booktitle    = {{STOC}},
  pages        = {459--467},
  publisher    = {{ACM}},
  year         = {2014}
}

@inproceedings{DevroyeLN13,
  author       = {Luc Devroye and
                  G{\'{a}}bor Lugosi and
                  Gergely Neu},
  title        = {Prediction by random-walk perturbation},
  booktitle    = {{COLT}},
  series       = {{JMLR} Workshop and Conference Proceedings},
  volume       = {30},
  pages        = {460--473},
  publisher    = {JMLR.org},
  year         = {2013}
}

@article{KalaiV05,
  author       = {Adam Tauman Kalai and
                  Santosh S. Vempala},
  title        = {Efficient algorithms for online decision problems},
  journal      = {J. Comput. Syst. Sci.},
  volume       = {71},
  number       = {3},
  pages        = {291--307},
  year         = {2005}
}

@inproceedings{GeulenVW10,
  author       = {Sascha Geulen and Berthold V{\"{o}}cking and Melanie Winkler},
  title        = {Regret Minimization for Online Buffering Problems Using the Weighted
                  Majority Algorithm},
  booktitle    = {{COLT}},
  pages        = {132--143},
  publisher    = {Omnipress},
  year         = {2010}
}

@article{auer2002nonstochastic,
  title={The nonstochastic multiarmed bandit problem},
  author={Auer, Peter and Cesa-Bianchi, Nicolò and Freund, Yoav and Schapire, Robert E.},
  journal={SIAM Journal on Computing},
  volume={32},
  number={1},
  pages={48--77},
  year={2002},
  publisher={SIAM}
}

@article{freund1997decision,
  title={A decision-theoretic generalization of on-line learning and an application to boosting},
  author={Freund, Yoav and Schapire, Robert E.},
  journal={Journal of Computer and System Sciences},
  volume={55},
  number={1},
  pages={119--139},
  year={1997},
  publisher={Elsevier}
}

@inproceedings{AlonCDK15,
  title = {Online Learning with Feedback Graphs: Beyond Bandits},
  author = {Alon, Noga and Cesa-Bianchi, Nicol{\`o} and Dekel, Ofer and Koren, Tomer},
  booktitle = {Proceedings of the 28th Conference on Learning Theory},
  series = {Proceedings of Machine Learning Research},
  volume = {40},
  pages = {23--35},
  year = {2015}
}

@book{lattimore2020bandit,
  title = {Bandit Algorithms},
  author = {Lattimore, Tor and Szepesv{\'a}ri, Csaba},
  publisher = {Cambridge University Press},
  year = {2020}
}

@book{levin2017markov,
  title = {Markov Chains and Mixing Times},
  author = {Levin, David A. and Peres, Yuval},
  edition = {Second},
  publisher = {American Mathematical Society},
  note = {With contributions by Elizabeth L. Wilmer},
  year = {2017}
}

@misc{aldous2002reversible,
  title = {Reversible {Markov} Chains and Random Walks on Graphs},
  author = {Aldous, David and Fill, James Allen},
  year = {2002},
  note = {Unfinished monograph, recompiled 2014}
}

@book{williams1991probability,
  title = {Probability with Martingales},
  author = {Williams, David},
  publisher = {Cambridge University Press},
  year = {1991}
}

@article{megiddo1988dominating,
  title = {On Finding a Minimum Dominating Set in a Tournament},
  author = {Megiddo, Nimrod and Vishkin, Uzi},
  journal = {Theoretical Computer Science},
  volume = {61},
  number = {2--3},
  pages = {307--316},
  year = {1988}
}

@book{cesa2006prediction,
  title = {Prediction, Learning, and Games},
  author = {Cesa-Bianchi, Nicol{\`o} and Lugosi, G{\'a}bor},
  publisher = {Cambridge University Press},
  year = {2006}
}

@inproceedings{rangi2019online,
  title = {Online Learning with Feedback Graphs and Switching Costs},
  author = {Rangi, Anshuka and Franceschetti, Massimo},
  booktitle = {Proceedings of the 22nd International Conference on Artificial Intelligence and Statistics},
  series = {Proceedings of Machine Learning Research},
  volume = {89},
  pages = {2435--2444},
  year = {2019}
}

@inproceedings{AroraMM19,
  title = {Bandits with Feedback Graphs and Switching Costs},
  author = {Arora, Raman and Marinov, Teodor V. and Mohri, Mehryar},
  booktitle = {Advances in Neural Information Processing Systems 32},
  pages = {10397--10407},
  year = {2019}
}

@inproceedings{cheng2023understanding,
  title = {Understanding the Role of Feedback in Online Learning with Switching Costs},
  author = {Cheng, Duo and Zhou, Xingyu and Ji, Bo},
  booktitle = {Proceedings of the 40th International Conference on Machine Learning},
  series = {Proceedings of Machine Learning Research},
  volume = {202},
  pages = {5521--5543},
  year = {2023}
}

@inproceedings{MannorS11,
  title = {From Bandits to Experts: On the Value of Side-Observations},
  author = {Mannor, Shie and Shamir, Ohad},
  booktitle = {Advances in Neural Information Processing Systems 24},
  pages = {684--692},
  year = {2011}
}

@inproceedings{alon2013bandits,
  title = {From Bandits to Experts: A Tale of Domination and Independence},
  author = {Alon, Noga and Cesa-Bianchi, Nicol{\`o} and Gentile, Claudio and Mansour, Yishay},
  booktitle = {Advances in Neural Information Processing Systems 26},
  pages = {1610--1618},
  year = {2013}
}

@inproceedings{kocak2014efficient,
  title = {Efficient Learning by Implicit Exploration in Bandit Problems with Side Observations},
  author = {Koc{\'a}k, Tom{\'a}{\v{s}} and Neu, Gergely and Valko, Michal and Munos, R{\'e}mi},
  booktitle = {Advances in Neural Information Processing Systems 27},
  pages = {613--621},
  year = {2014}
}

@article{alon2017nonstochastic,
  title = {Nonstochastic Multi-Armed Bandits with Graph-Structured Feedback},
  author = {Alon, Noga and Cesa-Bianchi, Nicol{\`o} and Gentile, Claudio and Mannor, Shie and Mansour, Yishay and Shamir, Ohad},
  journal = {SIAM Journal on Computing},
  volume = {46},
  number = {6},
  pages = {1785--1826},
  year = {2017}
}

@inproceedings{ZimmertL19,
  title = {Connections Between Mirror Descent, {Thompson} Sampling and the Information Ratio},
  author = {Zimmert, Julian and Lattimore, Tor},
  booktitle = {Advances in Neural Information Processing Systems 32},
  pages = {11973--11982},
  year = {2019}
}

@inproceedings{EldowaECC23,
  title = {On the Minimax Regret for Online Learning with Feedback Graphs},
  author = {Eldowa, Khaled and Esposito, Emmanuel and Cesari, Tommaso and Cesa-Bianchi, Nicol{\`o}},
  booktitle = {Advances in Neural Information Processing Systems 36},
  pages = {46122--46133},
  year = {2023}
}

@inproceedings{arora2012online,
  title = {Online Bandit Learning against an Adaptive Adversary: from Regret to Policy Regret},
  author = {Arora, Raman and Dekel, Ofer and Tewari, Ambuj},
  booktitle = {Proceedings of the 29th International Conference on Machine Learning},
  pages = {1503--1510},
  year = {2012}
}

@article{altschuler2021online,
  title = {Online Learning over a Finite Action Set with Limited Switching},
  author = {Altschuler, Jason M. and Talwar, Kunal},
  journal = {Mathematics of Operations Research},
  volume = {46},
  number = {1},
  pages = {179--203},
  year = {2021}
}

@article{simchilevi2023phase,
  title = {Phase Transitions in Bandits with Switching Constraints},
  author = {Simchi-Levi, David and Xu, Yunzong},
  journal = {Management Science},
  volume = {69},
  number = {12},
  pages = {7182--7201},
  year = {2023}
}

@inproceedings{koren2017multi,
  title = {Multi-Armed Bandits with Metric Movement Costs},
  author = {Koren, Tomer and Livni, Roi and Mansour, Yishay},
  booktitle = {Advances in Neural Information Processing Systems 30},
  pages = {4119--4128},
  year = {2017}
}

@inproceedings{rouyer2021algorithm,
  title = {An Algorithm for Stochastic and Adversarial Bandits with Switching Costs},
  author = {Rouyer, Chlo{\'e} and Seldin, Yevgeny and Cesa-Bianchi, Nicol{\`o}},
  booktitle = {Proceedings of the 38th International Conference on Machine Learning},
  series = {Proceedings of Machine Learning Research},
  volume = {139},
  pages = {9127--9135},
  year = {2021}
}

@inproceedings{amir2022better,
  title = {Better Best of Both Worlds Bounds for Bandits with Switching Costs},
  author = {Amir, Idan and Azov, Guy and Koren, Tomer and Livni, Roi},
  booktitle = {Advances in Neural Information Processing Systems 35},
  pages = {15800--15810},
  year = {2022}
}

@inproceedings{shi2022power,
  title = {Power-of-2-Arms for Bandit Learning with Switching Costs},
  author = {Shi, Ming and Lin, Xiaojun and Jiao, Lei},
  booktitle = {Proceedings of the Twenty-Third International Symposium on Theory, Algorithmic Foundations, and Protocol Design for Mobile Networks and Mobile Computing},
  pages = {131--140},
  year = {2022}
}

@inproceedings{FengL18,
  title = {Online Learning with Graph-Structured Feedback against Adaptive Adversaries},
  author = {Feng, Zhili and Loh, Po-Ling},
  booktitle = {Proceedings of the IEEE International Symposium on Information Theory},
  pages = {931--935},
  year = {2018}
}

\appendix

\section{Two actions: comparison probing}
\label{app:breakeven}

This appendix gives the complete statement and proof for the smallest strongly revealing graphs, with two actions: apple tasting \citep{HelmboldLL2000}, which is only sketched in \cref{sec:tech-two}. In apple tasting, action $1$, which we call revealing, observes both losses, and the other, the blind action $2$, yields no feedback. The comparison probing algorithm below achieves regret of order $\sqrt T$, without the logarithmic factors of \cref{thm:main}. Its principle, that a probe is paid for by what it reveals, guides the general algorithm of \cref{sec:algorithm}.

\subsection{Setting and algorithm}
\label{app:breakeven-alg}

In apple tasting, only the advantage $d_t=\Delta_t(1,2)=\ell_t(1)-\ell_t(2)\in[-1,1]$ of the blind action over the revealing action matters. Indeed, subtracting $\ell_t(2)$ from both losses of round $t$ changes the loss of the learner and of every fixed action by the same amount, so it leaves the regret unchanged, and the shifted losses are $(d_t,0)$. Write $D_t=\sum_{u=1}^td_u$ with $D_0=0$. The regret of a learner that plays $I_t\in\{1,2\}$ is then
\begin{align*}
\Reg_T=\E\left[\sum_{t=1}^T\one\{I_t=1\}d_t+N_T\right]-\min\{0,D_T\},
\end{align*}
where $N_T$ is the number of switches, as in \cref{eq:regret}.

Our algorithm, \proc{ComparisonProbing}, rests at the blind action and probes the revealing action at random times. In each round spent resting it tosses an independent coin $\mathsf c_t$ with success probability $\kappa$; on success it starts a probe. A probe that starts at round $t_0$ plays action $1$, observes the advantages, and returns to resting after the first round $t$ with $D_t-D_{t_0-1}\ge0$, that is, as soon as the blind action has done at least as well as the revealing action since the probe began. We call this rule the break-even rule, and $D_t-D_{t_0-1}$ the break-even statistic of the probe. In terms of Markov chains, the learner holds the blind action with probability $1-\kappa$ in each round, so the waiting times between probes are geometric, and it holds the revealing action as long as the break-even statistic is negative. The pseudocode of \proc{ComparisonProbing} is provided in \Cref{alg:breakeven}.

\begin{algorithm}[htbp]
\caption{\proc{ComparisonProbing} }
\label{alg:breakeven}
\begin{algorithmic}[1]
\REQUIRE Probe rate $\kappa$.
\STATE Start resting.
\FOR{$t=1,\ldots,T$}
\STATE If resting, toss an independent coin $\mathsf c_t$ with success probability $\kappa$; if $\mathsf c_t=1$, set $t_0\gets t$ and start a probe.
\STATE If probing, play action $1$, observe $d_t$, and return to resting if $D_t-D_{t_0-1}\ge0$; otherwise play action $2$.
\ENDFOR
\end{algorithmic}
\end{algorithm}
We note that the behavior of the algorithm is intrinsically asymmetric, as evidence comes only from the revealing action: the learner stays at the revealing action only while what it observes there shows that staying pays. We would like to highlight three features of \proc{ComparisonProbing} which  reappear in the algorithm of \cref{sec:algorithm}. First, the learner explores lazily: it rests for geometric waiting times between probes, so it switches only about $\kappa T$ times, and the holding chains of \cref{sec:holding} extend these waiting times to many actions. Second, every edge $s\to j$ of a strongly revealing graph defines a two-action problem of this kind, with $s$ revealing and $j$ blind, and the scores of \cref{sec:steering} are the advantages $\Delta_t(s,j)$ observed at the revealing endpoints, which lets one chain probe all pairs at once. Third, each probe ``pays for itself'': a completed probe costs at most one unit of loss more than resting. The general algorithm keeps this accounting in expectation, with the loss reduction of steering paying for the second moment of its estimates (\cref{prop:self-financing}).

Before analyzing the algorithm, we introduce some useful notation and a key lemma.
A \emph{probe} starts at a round $t_0$ when a coin tossed while resting succeeds. It ends at the first round $t\in\{t_0,\ldots,T\}$ such that $D_t-D_{t_0-1}\ge0$, in which case it is completed. If no such round exists, it continues through round $T$ and is unfinished. A new probe may start in the round immediately after the previous probe ends.

\begin{lemma}[Missed advantage]
\label{lem:missed-advantage}
Fix a deterministic sequence $d_1,\ldots,d_T\in[-1,1]$. Let $\mathcal R=\{t:I_t=2\}$ be the set of resting rounds and $N^{\rm probe}$ the number of probes. Then
\begin{align*}
\E\left[-\sum_{t\in\mathcal R}d_t\right]\le\frac1\kappa+\E N^{\rm probe}.
\end{align*}
\end{lemma}
\begin{proof}
We identify the rounds at which a probe would never end, bound the hidden advantage pathwise by the number of such rounds missed before the final probe plus the number of probes, and bound the expected number of missed rounds by a geometric waiting time. All randomness is due to the coins, and \cref{fig:geometric-probing}(b) illustrates the argument. For a candidate start time $t_0$, the probe that would start at $t_0$ is the interval $[t_0,t_1]$ with
\begin{align*}
t_1=\min\{u\in\{t_0,\ldots,T\}:D_u-D_{t_0-1}\ge0\},
\end{align*}
if this minimum exists, in which case the probe is completed; otherwise it is the unfinished interval $[t_0,T]$. The actual probes of the algorithm are exactly the intervals whose start times are selected by successful coins in resting rounds.

Call a round $t$ terminal if a probe started at $t$ would not be completed by round $T$. The set of terminal rounds is
\begin{align}
\mathcal B=\Bigl\{t:D_{t-1}>\max_{u=t,\ldots,T}D_u\Bigr\}.
\label{eq:terminal-times}
\end{align}
Indeed, a probe started at $t$ has cumulative advantage $D_u-D_{t-1}$ through round $u\ge t$, so it is completed by round $T$ if and only if $\max_{u\ge t}D_u\ge D_{t-1}$.

Let $t^\star$ be the first terminal round at which a probe starts, with $t^\star=T+1$ if there is none. An actual probe is unfinished if and only if its start is terminal, so $t^\star$ is the start of the first unfinished probe. Consequently, all probes inside $\{1,\ldots,t^\star-1\}$ are completed, and all resting rounds lie in $\{1,\ldots,t^\star-1\}$. Let $N^{\rm miss}=|\mathcal B\cap\{1,\ldots,t^\star-1\}|$ be the number of terminal rounds missed before the final probe.

We first relate the hidden advantage to $N^{\rm miss}$. Let $D^{\rm cl}$ be the total advantage over the probes contained in $\{1,\ldots,t^\star-1\}$. Each of them is completed and has cumulative advantage in $[0,1]$, so $0\le D^{\rm cl}\le N^{\rm probe}$. Up to round $t^\star-1$, the timeline consists of resting rounds and completed probes, and after it there are no resting rounds. Hence
\begin{align}
D_{t^\star-1}=\sum_{t\in\mathcal R}d_t+D^{\rm cl}\le\sum_{t\in\mathcal R}d_t+N^{\rm probe}.
\label{eq:hidden-advantage}
\end{align}

Next, we bound $-D_{t^\star-1}$ by $N^{\rm miss}$. Let $\overline D_t=\max_{u=t,\ldots,T}D_u$ for $t=0,\ldots,T$. For $t\ge1$, the drop $\overline D_{t-1}-\overline D_t$ is positive exactly when $t\in\mathcal B$, by \cref{eq:terminal-times}, and then
\begin{align*}
\overline D_{t-1}-\overline D_t=D_{t-1}-\overline D_t\le D_{t-1}-D_t=-d_t\le1.
\end{align*}
Therefore $\overline D_0-\overline D_{t^\star-1}\le N^{\rm miss}$. Moreover $\overline D_{t^\star-1}=D_{t^\star-1}$: if $t^\star\le T$, this follows from $t^\star\in\mathcal B$, and if $t^\star=T+1$, then $\overline D_T=D_T$. Since $\overline D_0\ge D_0=0$, we get $-D_{t^\star-1}\le\overline D_0-\overline D_{t^\star-1}\le N^{\rm miss}$. Combined with \cref{eq:hidden-advantage}, this gives the pathwise inequality
\begin{align*}
-\sum_{t\in\mathcal R}d_t\le-D_{t^\star-1}+N^{\rm probe}\le N^{\rm miss}+N^{\rm probe}.
\end{align*}

It remains to bound $\E N^{\rm miss}$. A terminal round cannot lie inside a completed probe $[t_0,t_1]$. Indeed, $D_{t_1}\ge D_{t_0-1}$; for $t=t_0$ this already shows $t\notin\mathcal B$, and for $t_0<t\le t_1$ the probe has not yet been completed at round $t-1$, so $D_{t-1}<D_{t_0-1}\le D_{t_1}$, and again $t\notin\mathcal B$. Since all probes before $t^\star$ are completed, every terminal round before $t^\star$ is a resting round. At each such round the algorithm tosses an independent coin with success probability $\kappa$, and $t^\star$ is the first terminal round whose coin succeeds. Therefore $N^{\rm miss}$ is stochastically dominated by the number of failures before the first success in independent trials with success probability $\kappa$, and $\E N^{\rm miss}\le(1-\kappa)/\kappa\le1/\kappa$.
\end{proof}

\begin{theorem}[Comparison probing]
\label{thm:breakeven}
For every deterministic sequence $d_1,\ldots,d_T\in[-1,1]$ and every $\kappa\in(0,1]$, \proc{ComparisonProbing} satisfies
\begin{align*}
\Reg_T\le\frac1\kappa+3\kappa T.
\end{align*}
In particular, setting $\kappa=1/\sqrt{3T}$ yields $\Reg_T\le2\sqrt{3T}$.
\end{theorem}

\begin{proof}
We bound the number of probes and then compare the learner with each action. A probe can start only in a resting round, and only when that round's coin succeeds. Equivalently, every round has an independent coin, which is used only if the algorithm rests at the start of the round. The resting event is determined before the round's coin is revealed, so $\E N^{\rm probe}$ is $\kappa$ times the expected number of rounds at whose start the learner rests, which is at most $\kappa T$. Each probe causes at most two switches, so $\E N_T\le2\E N^{\rm probe}\le2\kappa T$.

Against the blind action, every completed probe has cumulative advantage in $[0,1]$, and an unfinished probe has a negative one. The learner's total advantage $\sum_t\one\{I_t=1\}d_t$ is therefore at most $N^{\rm probe}$, and its regret against the blind action is at most $\E[N^{\rm probe}+N_T]\le3\kappa T$.

Against the revealing action, whose cumulative loss in the shifted problem is $D_T$, the learner misses exactly the advantages of the resting rounds:
\begin{align*}
\sum_{t:I_t=1}d_t-D_T=-\sum_{t\in\mathcal R}d_t.
\end{align*}
By \cref{lem:missed-advantage} and $N_T\le2N^{\rm probe}$, the regret against the revealing action is at most $1/\kappa+3\E N^{\rm probe}\le1/\kappa+3\kappa T$. The regret is upper bounded by the larger of the two bounds, which proves the theorem.
\end{proof}

\section{Proofs for the holding chains}
\label{app:holding}

Both proofs use the destination law of a refresh. Put $z_\tau=\sum_jx_j/(1+\tau\lambda_j)\le1$. Since $p_ja_j^{(\tau)}=\tau x_j/(2(1+\tau\lambda_j))$ and $\sum_rp_ra_r^{(\tau)}=\tau z_\tau/2$, the destination probabilities in \cref{eq:holding-chain} form the probability vector $r$ with $r_j=x_j/((1+\tau\lambda_j)z_\tau)$, and $P_\tau(j\mid s)=(1-a_s^{(\tau)})\one\{s=j\}+a_s^{(\tau)}r_j$.

\begin{proof}[Proof of \cref{lem:holding-chains}]
We read off the properties from the explicit chain. For (i), detailed balance holds because $p_sP_\tau(j\mid s)=p_sa_s^{(\tau)}p_ja_j^{(\tau)}/\sum_rp_ra_r^{(\tau)}$ is symmetric in $s\ne j$. Next, $a_s^{(\tau)}<\tfrac12$ because $\tau\lambda_s/(1+\tau\lambda_s)<1$, and $\sum_sp_sa_s^{(\tau)}\le\tfrac\tau2\sum_sp_s\lambda_s=\tfrac\tau2$. Since $z_\tau\le1$, $\lambda_s\ge\tfrac12$, and $y\mapsto y/(1+\tau y)$ is increasing,
\begin{align*}
\frac{a_s^{(\tau)}}{\sum_jp_ja_j^{(\tau)}}=\frac{\lambda_s}{(1+\tau\lambda_s)z_\tau}\ge\frac{1/2}{1+\tau/2}=\frac1{2+\tau}\ge\frac13.
\end{align*}
Finally, $P_\tau(s\mid s)=1-a_s^{(\tau)}(1-r_s)\ge1-a_s^{(\tau)}\ge\tfrac12$.

For (ii), reversibility makes $p$ stationary, so the chain has marginal law $p$ at every round. Conditional on $X_t=s$, a switch has probability $1-P_\tau(s\mid s)\le a_s^{(\tau)}$. Averaging under $p$ and summing over the $h-1$ transitions gives the bound $(h-1)\tau/2$.
\end{proof}

\begin{proof}[Proof of \cref{lem:spectral}]
We diagonalize $A$ after a symmetric rescaling and verify the identity in (ii) by direct multiplication. For (i), the definition \cref{eq:generator} and $p_s\lambda_s=x_s$ give $\diag(p)A=\diag(x)-xx^{\mathsf T}$. This matrix is symmetric, $f^{\mathsf T}(\diag(x)-xx^{\mathsf T})f=\Var_x(f)$, and $A\one=0$. Put $\widetilde A=\diag(\sqrt p)A\diag(1/\sqrt p)$, which equals $\diag(1/\sqrt p)\,\diag(p)A\,\diag(1/\sqrt p)$ and is therefore symmetric, and note that $\widetilde A\sqrt p=\diag(\sqrt p)A\one=0$. By the spectral theorem for symmetric matrices, $\widetilde A$ has an orthonormal basis of eigenvectors that contains $\sqrt p$. Let $\widetilde\phi_1,\ldots,\widetilde\phi_{K-1}$ be the other basis vectors, and put $\phi_n=\diag(1/\sqrt p)\widetilde\phi_n$. These are eigenvectors of $A$, orthonormal for $\ip\cdot\cdot_p$, and $p$-centered because $p^{\mathsf T}\phi_n=\sqrt p^{\mathsf T}\widetilde\phi_n=0$. Their eigenvalues satisfy
\begin{align*}
\alpha_n=\ip{\phi_n}{A\phi_n}_p=\Var_x(\phi_n)=\min_c\sum_sx_s(\phi_{n,s}-c)^2.
\end{align*}
Since $x\ge p/2$ and $\phi_n$ is a $p$-centered unit vector, the minimum is at least $\tfrac12\min_c\sum_sp_s(\phi_{n,s}-c)^2=\tfrac12$. Evaluating it at $c=0$ and using $x_s\le\lmax p_s$ shows that it is at most $\lmax$.

For (ii), put $\Theta=\diag((1+\tau\lambda_s)_s)$. Using $1-2a_s^{(\tau)}=1/(1+\tau\lambda_s)$ and $2a_s^{(\tau)}=\tau\lambda_s/(1+\tau\lambda_s)$, we can write
\begin{align*}
2P_\tau-I=\Theta^{-1}+\tau\Theta^{-1}\lambda\,r^{\mathsf T},
\qquad
I+\tau A=\Theta-\tau\lambda x^{\mathsf T}.
\end{align*}
Multiplying, and using $x^{\mathsf T}\Theta^{-1}=z_\tau r^{\mathsf T}$ and $\tau x^{\mathsf T}\Theta^{-1}\lambda=\sum_jx_j\tau\lambda_j/(1+\tau\lambda_j)=1-z_\tau$, gives
\begin{align*}
(I+\tau A)(2P_\tau-I)=I+\tau\lambda r^{\mathsf T}-\tau z_\tau\lambda r^{\mathsf T}-\tau(1-z_\tau)\lambda r^{\mathsf T}=I.
\end{align*}
Hence $2P_\tau-I$ is the inverse of $I+\tau A$. Since $(I+\tau A)\phi_n=(1+\tau\alpha_n)\phi_n$, we obtain $(2P_\tau-I)\phi_n=\phi_n/(1+\tau\alpha_n)$, that is, $P_\tau\phi_n=\nu_\tau(\alpha_n)\phi_n$. The rows of $P_\tau$ sum to one, so $P_\tau\one=\one$.
\end{proof}

\section{Proofs for the steering gain}
\label{app:probing}

\begin{proof}[Proof of \cref{lem:probing}]
We identify the path law by multiplying its transition factors. This gives the corrected estimate identities directly; the loss response then follows from the Markov property and the score identity.

Relabel the rounds of the chunk as $1,\ldots,|J|$. Under $\Prb_0$ the actions form a stationary $P$-chain. For $1\le t\le|J|-1$ let
\begin{align*}
\zeta_t=d_t(X_t,X_{t+1})\,\one\{|S_{t-1}|\le M-1\},
\qquad S_0=0,\qquad S_t=S_{t-1}+\zeta_t.
\end{align*}
These are the effective scores of \cref{alg:probe-chunk}, which depend only on $X_{1:t+1}$ because the losses are deterministic. The indicator depends only on $X_{1:t}$, and condition (ii) of \cref{def:score-rule} gives $\sum_jP(j\mid X_t)d_t(X_t,j)=0$, so $\zeta_t$ has conditional mean zero given $X_{1:t}$. Since $P(s\mid s)\ge\tfrac12$, condition (iii) gives $|d_t|\le1$ for all arguments. Hence $|S_t|\le M$: an update starts from $|S_{t-1}|\le M-1$, and a frozen sum does not change. Consequently every factor in \cref{eq:tilt} is positive, since $\beta M=1/4$, and the factors at $X_t$ average to one under $P(\cdot\mid X_t)$, because the scores have mean zero there. The transition rule therefore defines probability distributions. The final score sum is $S_J=S_{|J|-1}$. Along a path $x_{1:|J|}$, the product of the factors telescopes to $1+\beta S_J$, so $\Prb_{\rm tilt}(x_{1:|J|})=(1+\beta S_J)\Prb_0(x_{1:|J|})$. Summing a nonnegative $\Phi$ against both sides and using $1+\beta S_J\le1+\beta M\le\tfrac54$ gives $\E_{\rm tilt}\Phi\le\tfrac54\E_0\Phi$. The scores use only comparisons observed at the current action in the current round, together with the past score sum, which proves causality and observability.

The same path identity gives
\begin{align*}
\E_{\rm tilt}\frac{G_J}{R_J}=\E_0G_J=\sum_t\sum_sp_s\Psi_s\ell_t=C_wL_J,
\qquad
\E_{\rm tilt}\left\|\frac{G_J}{R_J}\right\|_w^2
=\E_0\frac{\|G_J\|_w^2}{R_J}\le\frac43\E_0\|G_J\|_w^2.
\end{align*}
These prove \cref{eq:chunk-moments}. Whenever $\Psi_s\ell_t$ is observable, the corrected output is observable as well.

Write $Y=\sum_{t=1}^{|J|}\ell_t(X_t)$. Summing $Y$ against the path law gives $\E_{\rm tilt}Y=\E_0Y+\beta\E_0[S_JY]$, and $\E_0Y=\ip p{L_J}$ by stationarity. Let $S^{\rm raw}=\sum_{t=1}^{|J|-1}d_t(X_t,X_{t+1})$ be the score sum without freezing. The two sums agree unless scoring stops, in which case they differ by at most $|J|-1$, because each omitted score has magnitude at most one. Since $0\le Y\le|J|$, the remainder $\mathsf r=\E_0[(S_J-S^{\rm raw})Y]$ satisfies $|\mathsf r|\le|J|^2\Prb_0(\Omega_{\rm stop})$.

It remains to compute $\E_0[S^{\rm raw}Y]$. Fix $t\le|J|-1$ and split $Y=Y_{\le t}+Y_{>t}$ into the losses of rounds up to $t$ and after $t$. The first part is a function of $X_{1:t}$, while $\E_0[d_t(X_t,X_{t+1})\mid X_{1:t}]=0$; hence $\E_0[d_t(X_t,X_{t+1})Y_{\le t}]=0$. By the Markov property, $\E_0[Y_{>t}\mid X_{1:t+1}]=W_{t+1}(X_{t+1})$ with $W_{t+1}=\sum_{t'>t}P^{t'-t-1}\ell_{t'}$. Since $d_t(X_t,X_{t+1})$ is a function of $X_{1:t+1}$ and $X_t\sim p$, condition (iv) of \cref{def:score-rule} gives
\begin{align*}
\E_0[d_t(X_t,X_{t+1})Y]
=\sum_{i,j}p_iP_{ij}d_t(i,j)W_{t+1}(j)
=-\ip{\ell_t}{(I-P)W_{t+1}}_p.
\end{align*}
Summing over $t$ and substituting $W_{t+1}$ gives $\E_0[S^{\rm raw}Y]=-\sum_{t<t'}\ip{\ell_t}{(I-P)P^{t'-t-1}\ell_{t'}}_p$. Substituting into $\E_{\rm tilt}Y=\ip p{L_J}+\beta\E_0[S^{\rm raw}Y]+\beta\mathsf r$ proves \cref{eq:probing-response}.
\end{proof}

\begin{proof}[Proof of \cref{cor:probing-block}]
We apply \cref{lem:probing} to each chunk for a fixed rate and partition, average, and bound the stopping error. The rate and the restarts are drawn before play, independently of the losses; each chunk starts from a fresh draw from $p$ and is causal by \cref{lem:probing}. This proves causality and observability.

Condition on the rate $\tau$ and on the partition $\cP_\tau$. By \cref{lem:holding-chains}(i), $P_\tau$ is reversible with respect to $p$ and $P_\tau(s\mid s)\ge\tfrac12$, so \cref{lem:probing} applies to every chunk $J$ and gives
\begin{align*}
\E\Bigl[\sum_{t\in J}\ell_t(X_t)\Bigm|\tau,\cP_\tau\Bigr]
=\ip p{L_J}-\beta\sum_{t<t'\in J}\ip{\ell_t}{(I-P_\tau)P_\tau^{t'-t-1}\ell_{t'}}_p+\beta\mathsf r_J,
\end{align*}
with $|\mathsf r_J|\le|J|^2\,\Prb_0(\Omega_J\mid\tau,\cP_\tau)$, where $\Omega_J$ is the event that scoring stops in chunk $J$, and $\Prb_0$ refers to the reference chain $P_\tau$ started from $p$. Summing over chunks, a pair $t<t'$ contributes exactly when no restart separates it, which has probability $(1-\tau)^{t'-t}$ given $\tau$. Averaging over the restarts and the rate therefore turns the double sums into $\Gamma(\ell_{1:h})$ of \cref{eq:gain}, and
\begin{align*}
\E\sum_{t=1}^h\ell_t(X_t)&\le p^{\mathsf T} L-\beta\,\Gamma(\ell_{1:h})+\beta\mathsf r_{\rm stop},\\
\mathsf r_{\rm stop}&=\E_{\tau,\cP_\tau}\sum_{J\in\cP_\tau}|J|^2\,\Prb_0(\Omega_J\mid\tau,\cP_\tau).
\end{align*}
To bound the stopping error, group chunks by their starting round. Let $\Omega^{(t)}$ be the reference event that a chunk starts at $t$ and its scoring stops. Conditional on the rate and on a start at $t$, the initial action has law $p$ and future restarts are fresh, so \cref{lem:tails}(i) gives $\Prb_0(\Omega^{(t)})\le\delta$. Since every chunk has at most $h$ rounds,
\begin{align*}
\mathsf r_{\rm stop}
\le h^2\E_{\tau,\cP_\tau}\sum_{J\in\cP_\tau}\Prb_0(\Omega_J\mid\tau,\cP_\tau)
=h^2\sum_{t=1}^h\Prb_0(\Omega^{(t)})\le\delta h^3.
\end{align*}
This proves the loss bound without conditioning the tail estimate on the eventual chunk length.

For switches, again condition on $\tau$ and $\cP_\tau$. Applying the bound $\E_{\rm tilt}\Phi\le\tfrac54\E_0\Phi$ of \cref{lem:probing} to the number of changes inside a chunk $J$, and then \cref{lem:holding-chains}(ii), bounds the expected changes inside $J$ by $\tfrac54(|J|-1)\tau/2$. Summing over chunks gives at most $5\tau h/8$. A change at a chunk boundary requires a restart, and the expected number of restarts is at most $\tau(h-1)$. Hence the conditional expectation of $N_{\rm int}$ given $\tau$ is at most $13\tau h/8$, and averaging over the rate gives $\tfrac{13}8\mu h$, by \cref{eq:mean-rate}.
\end{proof}

\section{Proof of the score and estimate tails}
\label{app:rare}

We begin with an elementary moment bound for sums of nonnegative random variables, from which both tail bounds in \cref{lem:tails} will follow. Here a filtration $(\mathcal H_t)$ records the information revealed through step $t$, and adaptation means that $Y_t$ is determined by the information in $\mathcal H_t$.

\begin{lemma}[Accumulated quantities]
\label{lem:occupation}
Let $Y_1,\ldots,Y_n$ be adapted to $(\mathcal H_t)_{t=0}^n$, with $0\le Y_t\le b_*$ and
$\E[\sum_{u=t+1}^nY_u\mid\mathcal H_t]\le v_*$ for $0\le t\le n$.
Let $V\coloneqq\sum_{t=1}^nY_t$ and $B_*\coloneqq b_*+v_*>0$. For every $0\le\theta<B_*^{-1}$,
\begin{align*}
\E[e^{\theta V}]\le1+\frac{\theta\E[V]}{1-\theta B_*}.
\end{align*}
\end{lemma}

\begin{proof}
We first bound the positive integer moments of $V$ by ordering the time indices in their expansions, then sum the resulting bounds in the exponential series. For $k\ge1$, each nondecreasing tuple represents at most $k!$ terms of the expansion, and all terms are nonnegative, giving
\begin{align*}
V^k\le k!\sum_{1\le t_1\le\cdots\le t_k\le n}Y_{t_1}\cdots Y_{t_k}.
\end{align*}
For $k\ge2$, fix the first $k-1$ indices and condition on $\mathcal H_{t_{k-1}}$, at which point the factors with those indices are already known. The remaining sum $\sum_{t_k=t_{k-1}}^nY_{t_k}$ includes the current term, bounded by $b_*$, and the future terms, whose conditional expected sum is bounded by $v_*$, so its conditional expectation is at most $b_*+v_*=B_*$. Applying this bound successively to the last index in each remaining ordered sum gives $\E[V^k]\le k!B_*^{k-1}\E[V]$, with equality for $k=1$. For $\theta B_*<1$, we can therefore sum the exponential series, interchanging expectation and summation because its terms are nonnegative:
\[
\E[e^{\theta V}]
\le1+\E[V]\sum_{k=1}^{\infty}\theta^kB_*^{k-1}
=1+\frac{\theta\E[V]}{1-\theta B_*}.
\]
This is the asserted bound.
\end{proof}

The same bound applies to a process stopped before $n$ using only the information revealed so far, since setting all later nonnegative terms to zero preserves the required upper bounds on the summands and on their conditional expected remainder. For the applications below, we need these conditional bounds for the accumulated absolute scores and for the quantities that bound the estimator, which we obtain from two explicit identities for the holding chains, valid from any initial state.

\begin{lemma}[Expected accumulated quantities]
\label{lem:occupation-potentials}
Let $(p,x)$ be a holding family, let $0<\kappa\le1/32$ satisfy $\kappa\lmax\le1/4$, and let $h\ge1$ be an integer with $h\kappa\le1/8$. Fix $\kappa\le\tau\le1$.
\begin{enumerate}
\item A $P_\tau$-chain from any state satisfies $\E\bigl[\sum_{t=1}^n\lambda_{X_t}\bigr]\le5/\kappa$ for every $n\le h$.
\item Let $N$ be the length of a reference chunk terminated by independent restarts of probability $\tau$, possibly shortened at a deterministic horizon. For any score rule and initial state,
\begin{align*}
\E\left[\sum_{t=1}^{N-1}|d_t(X_t,X_{t+1})|\right]\le7.
\end{align*}
\end{enumerate}
\end{lemma}

\begin{proof}
We first derive a summation identity that will be used for both assertions. If $(I-P_\tau)g=f-(p^{\mathsf T} f)\one$, the conditional expected decrease in $g$ over a transition equals the difference between $f$ and its stationary mean, and after taking expectations the intermediate values cancel when summed. Writing $\E_s$ for expectation with initial state $X_1=s$, this gives
\begin{align}
\E_s\left[\sum_{t=1}^nf(X_t)\right]=n\,p^{\mathsf T} f+g_s-\E_s[g_{X_{n+1}}].
\label{eq:occupation-telescope}
\end{align}

We now prove (i) by taking $f_s\coloneqq\lambda_s$ and $g_s\coloneqq2\lambda_s-2/(\tau\lambda_s)$. The refresh destination law in \cref{eq:holding-chain} is $r_j\coloneqq p_ja_j^{(\tau)}/\sum_{i=1}^Kp_ia_i^{(\tau)}$, and $\sum_{s=1}^Kp_s\lambda_s=1$ follows from the definition of $\lambda_s$ and the fact that $x$ is a probability vector. Substituting these expressions gives $r^{\mathsf T} g=2-2/\tau$ and $((I-P_\tau)g)_s=a_s^{(\tau)}(g_s-r^{\mathsf T} g)=\lambda_s-1$, which is the identity required to apply \cref{eq:occupation-telescope}.

It remains to bound the difference between the initial value of $g$ and its expected final value in that identity, which is at most the difference between its largest and smallest coordinates. This difference is at most $2\lmax+4/\tau\le9/(2\kappa)$, using $\lambda_s\ge1/2$ together with the assumptions on $\kappa\lmax$ and $\tau$. Since $n\le h$ and $h\kappa\le1/8$, \cref{eq:occupation-telescope} bounds the expected sum by $h+9/(2\kappa)\le37/(8\kappa)<5/\kappa$, proving (i).

For (ii), abbreviate $a_s\coloneqq a_s^{(\tau)}$ and $\bar a\coloneqq\sum_{s=1}^Kp_sa_s$, so $\bar a$ is the stationary mean refresh probability, and set $g_s\coloneqq-\bar a/a_s$. By \cref{lem:holding-chains}(i), $-3\le g_s\le0$, hence the difference between two coordinates of $g$ is at most $3$. The same refresh law satisfies $r^{\mathsf T} g=-1$, so $(I-P_\tau)g=a-\bar a\one$. We can therefore apply \cref{eq:occupation-telescope} with $f=a$ and average over the independent length $N-1$, with a zero sum when $N=1$:
\begin{align*}
\E_s\left[\sum_{t=1}^{N-1}a_{X_t}\right]
\le\bar a\,\E[N-1]+3
\le\frac{\bar a}{\tau}+3\le\frac72.
\end{align*}
Here the expected number of transitions before a restart is at most $1/\tau$, deterministic truncation can only decrease this number, and the final bound uses $\bar a\le\tau/2$ from \cref{lem:holding-chains}(i). To pass from refresh probabilities to absolute scores, use the zero row mean in condition (ii) of \cref{def:score-rule}: the diagonal contribution in absolute value is at most the sum of the absolute off-diagonal contributions, and condition (iii) bounds each off-diagonal score by one, giving
\[
\sum_{j=1}^KP_\tau(j\mid s)|d_t(s,j)|
\le2\sum_{j\in[K]\setminus\{s\}}P_\tau(j\mid s)|d_t(s,j)|
\le2a_s.
\]
The restart coins are independent of the reference chain, so we may sum this conditional bound over the transitions before termination and then take expectations. The preceding bound of $7/2$ for the sum of refresh probabilities now gives the claimed bound of $7$ for the sum of absolute scores, proving (ii).
\end{proof}

\begin{proof}[Proof of \cref{lem:tails}]
We apply \cref{lem:occupation} first to absolute scores in one reference chunk, then to the quantities that bound the estimator in each fixed chunk. For the latter application, we use each chunk's likelihood ratio to pass from the reference law to the law of the played actions, then combine the chunks using their conditional independence.

\emph{Score stopping.} We first prove (i), fixing the rate and the starting round of a reference chunk. Write $n$ for the number of possible internal transitions from that starting round to the block's end, and relabel the remaining rounds $1,\ldots,n+1$. For each $t\in\{1,\ldots,n\}$, let $Y_t\coloneqq|d_t(X_t,X_{t+1})|$ if the chunk continues across that boundary and let it be zero otherwise, using the raw scores without freezing and setting all terms after termination to zero. Then $Y_t\le1$ by condition (iii) of \cref{def:score-rule} and the holding-probability bound in \cref{lem:holding-chains}(i).

For this argument, reveal the restart coins only as their boundaries are reached, so that future coins remain independent of the revealed history. By the Markov property and \cref{lem:occupation-potentials}(ii), the conditional expected sum of future terms is at most $7$, since a surviving chunk continues from its current state with fresh restart coins, whereas a terminated chunk contributes only zeros. This bound holds at the start as well as after each transition, so \cref{lem:occupation} applies with $b_*=1$, $v_*=7$, and $\theta=1/16$, giving
\[
\E_0\left[\exp\left(\frac1{16}\sum_{t=1}^nY_t\right)\right]
\le1+\frac{(1/16)7}{1-8/16}<2.
\]
If scoring stops, the raw and stopped scores agree up to that time, so some raw partial score sum exceeds $M-1$ in absolute value and therefore $\sum_{t=1}^nY_t>M-1$. Applying Markov's inequality to the preceding exponential bound gives a stopping probability at most $2e^{-(M-1)/16}\le\delta$, where the last inequality follows from $M\ge100\log(2/\delta)$, proving (i).

\emph{Estimator size.} We now prove (ii), conditioning on the rate and the partition $\cP_\tau$, so the chunk lengths are fixed. In a reference chunk $J$ of length $|J|$, let $V_J\coloneqq\sum_{t\in J}\lambda_{X_t}$. Each summand is at most $1/(4\kappa)$, and the conditional expected sum of future terms is at most $5/\kappa$ by \cref{lem:occupation-potentials}(i), which applies from every current state. With $\theta\coloneqq\kappa/12$, \cref{lem:occupation} gives $\E_0[e^{\theta V_J}]\le1+2\theta|J|$, since $\theta(b_*+v_*)\le21/48<1/2$ and stationarity gives $\E_0[V_J]=|J|$. The latter identity uses the fresh initial draw from $p$ in each reference chunk.

We next transfer this bound to the played law using the likelihood ratio in \cref{lem:probing}. Since $\E_0[R_J]=1$, $R_J\le5/4$, and $e^{\theta V_J}-1\ge0$, we can handle the constant term exactly and apply the upper bound on $R_J$ only to the nonnegative remainder, obtaining
\begin{align*}
\E_{\rm tilt}[e^{\theta V_J}]
=1+\E_0[R_J(e^{\theta V_J}-1)]
\le1+\frac52\theta|J|\le e^{(5/2)\theta|J|}.
\end{align*}
Each chunk starts with a fresh draw from $p$ and uses independent randomness, so the chunk paths are independent conditional on the rate and partition. For $V\coloneqq\sum_{J\in\cP_\tau}V_J$, multiplying the preceding bounds combines their exponents, and the chunk lengths sum to $h$. Averaging over the rate and partition therefore gives $\E[e^{\theta V}]\le e^{(5/2)\theta h}\le e^{1/32}$ under the played law, where the last inequality uses the bound on $h\kappa$.

We can now bound the raw corrected estimate $Z$ of \cref{eq:block-raw}. The hypothesis of (ii) bounds each estimator contribution, and $R_J\ge3/4$ bounds the correction factor, so the triangle inequality gives
\begin{align*}
\|Z\|_\infty\le\frac43L_0V\le\frac{hL_0}{3\kappa}.
\end{align*}
The second inequality uses the bound on each summand of $V$ and proves the asserted almost-sure bound. For the tail bound, the event $\|Z\|_\infty>2ML_0/\kappa$ implies $V>3M/(2\kappa)$, so applying Markov's inequality to $e^{\theta V}$ with the value of $\theta$ chosen above yields
\[
\Prb\left[\|Z\|_\infty>\frac{2ML_0}{\kappa}\right]
\le e^{1/32-M/8}\le\frac\delta2.
\]
The last inequality again follows from $M\ge100\log(2/\delta)$, which completes (ii) and the proof.
\end{proof}

\section{Proofs of the tournament lemmas}
\label{app:tournament}

\begin{proof}[Proof of \cref{lem:rooted-assignment}]
We construct a random sequence of pivots and average the resulting comparisons to obtain matrices and source weights satisfying the assignment constraints, then bound their total source weight by the expected number of pivots. The first pivot serves as a common reference, and the later pivots supply observable comparisons with it, as we verify below. The construction is a weighted and randomized version of the greedy argument that every tournament on $K$ vertices has a dominating set of at most $\lceil\log_2K\rceil$ vertices \citep{megiddo1988dominating}.

For $\cU\subseteq[K]$, let $v(\cU)\coloneqq\sum_{i\in\cU}v_i$ denote its total weight, and start with the remaining set $\cU\coloneqq[K]$. While this set is nonempty, choose a pivot $s\in\cU$ with probability $v_s/v(\cU)$, let $\mathcal A_s\coloneqq\cU\cap\overline N^+(s)$ be its assigned cell, remove that cell, and repeat. Each step removes the pivot itself, so the procedure terminates after at most $K$ steps, and sources that are never selected have empty cells. Let $r_0$ be the first pivot, which we call the root, and let $\cV$ be the random set of selected pivots, so each source compares the targets in its cell with this root. Define the rows of $Q_s$ and the source weights by
\begin{align*}
e_i^{\mathsf T} Q_s\coloneqq\E\bigl[\one\{i\in\mathcal A_s\}(e_i-e_{r_0})^{\mathsf T}\bigr],
\qquad\rho_s\coloneqq\Prb[s\in\cV].
\end{align*}

\emph{Feasibility.} We first check the matrix constraints in \cref{eq:assignment}. Each row sums to zero and has nonpositive off-diagonal entries, because these properties hold for every realization of the row before taking its expectation. Every later pivot points to every earlier pivot, since an edge in the other direction would have caused the later pivot to be removed when the earlier one was selected, hence the root and every target in the cell of a selected source $s$ belong to $\overline N^+(s)$. This gives the support restriction, including for sources that are never selected and therefore contribute zero. The cells partition $[K]$ and $r_0$ has law $v$, so summing the row definitions gives $\sum_{s=1}^Ke_i^{\mathsf T} Q_s=e_i^{\mathsf T}-v^{\mathsf T}$ for every $i$, which proves \cref{eq:assignment}.

We next check the source-weight constraints. Being the first pivot implies being selected, so $v_s\le\rho_s\le1$, and the diagonal entry satisfies $(Q_s)_{ii}=\Prb[i\in\mathcal A_s,\ r_0\ne i]\le\Prb[i\in\mathcal A_s]\le\rho_s$, since a target can belong to a source's cell only if that source is selected. For the remaining load constraint, let $\cU_k$ be the remaining set before the $k$-th pivot, taking it to be empty after termination. In the sums below, every term involving an empty remaining set is defined to be zero, so a fraction is evaluated only when its denominator is positive. A vertex is selected at most once, and summing its conditional selection probabilities over the possible steps gives
\begin{align*}
\rho_i&=\sum_{k=1}^K\E\Bigl[\one\{i\in\cU_k\}\frac{v_i}{v(\cU_k)}\Bigr],\\
(Q_s)_{ii}\le\Prb[i\in\mathcal A_s]
&=\one\{i\in\overline N^+(s)\}\sum_{k=1}^K\E\Bigl[\one\{s,i\in\cU_k\}\frac{v_s}{v(\cU_k)}\Bigr]\le\frac{v_s}{v_i}\rho_i.
\end{align*}
For the last inequality, drop the requirement that the source also remain at that step and use the preceding formula for the target's selection probability. This proves the remaining load constraint, and therefore $((Q_s)_s,\rho)\in\cC(v)$.

\emph{Total source weight.} We now bound the number of pivots by showing that each step removes at least half the remaining weight in expectation. Conditional on a nonempty remaining set $\cU$, the vertices that survive a pivot $s$ are precisely those $j\in\cU$ with $j\to s$, because the pivot and all of its remaining out-neighbors are removed. Each unordered pair of distinct vertices in $\cU$ contributes its product of weights exactly once, according to its unique edge, giving
\begin{align*}
\E[v(\cU_{\rm next})\mid\cU]=\frac1{v(\cU)}\sum_{s\in\cU}v_s\sum_{j\in\cU:\,j\to s}v_j
=\frac{v(\cU)^2-\sum_{i\in\cU}v_i^2}{2v(\cU)}\le\frac{v(\cU)}2.
\end{align*}
Starting from total weight one and iterating this bound shows that after $k$ pivots the expected remaining mass is at most $2^{-k}$, with the remaining mass kept at zero after termination. A nonempty remaining set has mass at least $v_{\min}$, so Markov's inequality bounds the probability of needing another pivot by $\min\{1,2^{-k}/v_{\min}\}$. Summing these survival probabilities over $k\ge0$ gives the expected number of pivots, which equals $\sum_{s=1}^K\rho_s$, and bounds it by $\lceil\log_2(1/v_{\min})\rceil+2\le3\log(e/v_{\min})$. For the first bound, split the sum at the ceiling in this expression, bound each earlier term by one, and sum the remaining geometric tail, whose sum is at most two. The last inequality follows from $\lceil y\rceil\le y+1$ and $1/\log2<3$, proving the required upper bound on total source weight.

\emph{Attainment of the minimum.} The constraint $\rho\ge v$ gives $\sum_{s=1}^K\rho_s\ge1$ for every feasible point, because the coordinates of the probability vector sum to one. The feasible set is closed, since its defining constraints are linear, and it is bounded: $v\le\rho\le\one$, while the sign and sum constraints give $-v_r\le(Q_s)_{ir}\le0$ for $i\ne r$ and $0\le(Q_s)_{ii}\le1-v_i$. Indeed, the off-diagonal entries are nonpositive and sum over sources to the corresponding entry of the centering matrix, while the diagonal entries are nonnegative and have the same sum property. The feasible set is nonempty by the construction and is compact in a finite-dimensional space, so the continuous objective of the linear program in \cref{eq:assignment-lp} attains its minimum. Its value is at most the objective of the constructed point, which together with the lower bound proves the lemma.
\end{proof}

\begin{proof}[Proof of \cref{lem:balanced-estimator}]
We first verify the properties of the holding family, then conditions (i), (ii), and (iv) of \cref{def:balanced}, and finally the two endpoint bounds in condition (iii). For these last two bounds, we first prove two estimates for weighted sums of squared assigned comparisons, then account for centering and for the choice of sampling law.

\emph{Holding family.} Write $\|\rho\|_1=\sum_{s=1}^K\rho_s$, which is at least one because $\rho\ge v$, and note that the normalized source weights form a positive probability vector. The mixture defining the sampling law is therefore a positive probability vector as well, and gives $p\ge(1-4\kappa)w\ge\tfrac78w$ and $p_s\ge4\kappa q_s$. It also gives $\|p-w\|_1=4\kappa\|q-w\|_1\le8\kappa$, since the distance in this norm between two probability vectors is at most two. Next, $x=(p+q)/2$ is a positive probability vector with $x\ge p/2$ and $x\ge q/2$, and $\lambda_s=x_s/p_s=(1+q_s/p_s)/2$. The bound $p_s\ge4\kappa q_s$ gives $\kappa\lambda_s\le\kappa/2+1/8<1/4$, where the last inequality uses the assumed upper bound on the base rate, completing the holding-family claims.

\emph{Conditions (i), (ii), (iv).} By \cref{eq:comparison-form} and the support constraint, $Q_s\ell_t$ is computable from the comparisons observed at $s$, because every nonzero coefficient multiplies a difference whose two endpoints belong to that source's closed neighborhood. The matrix $C_w$ subtracts a known weighted mean of the resulting vector, and division by the known sampling probability therefore preserves observability, proving (i). Next, $Q_s\one=0$ gives $\Psi_s\one=0$, and $\sum_{s=1}^Kp_s\Psi_s=C_w\sum_{s=1}^KQ_s=C_wC_v=C_w$, because the constant vector subtracted by the inner centering is removed by the outer centering, proving (ii).

For $\ell\in[0,1]^K$, the weights $-(Q_s)_{ir}$ in \cref{eq:comparison-form} are nonnegative and sum to $(Q_s)_{ii}$, so $|(Q_s\ell)_i|\le(Q_s)_{ii}\le\rho_s$, since every loss difference has absolute value at most one. Subtracting a weighted mean at most doubles the maximum norm, because that mean lies between the smallest and largest coordinates. Using $q_s/p_s=2\lambda_s-1\le2\lambda_s$, we obtain
\begin{align*}
\|\Psi_s\ell\|_\infty\le\frac{2\rho_s}{p_s}=2\|\rho\|_1\frac{q_s}{p_s}\le4\|\rho\|_1\lambda_s\le12\log(e/v_{\min})\lambda_s\le L_0\lambda_s.
\end{align*}
Here the final bounds use the assumed total source weight and the value of the balancing constant in the lemma, proving (iv).

\emph{Two bounds for the assignment.} To prove the remaining condition, we first show that every $f\in\R^K$ satisfies
\begin{align}
\sum_{s,i=1}^K\frac{v_i}{q_s}(Q_sf)_i^2\le2\|\rho\|_1\Var_v(f),
\qquad
\sum_{s,i=1}^K\frac{v_i}{v_s}(Q_sf)_i^2\le4\|\rho\|_1\Var_{(v+q)/2}(f).
\label{eq:assignment-moments}
\end{align}
Apply the weighted Cauchy-Schwarz inequality to \cref{eq:comparison-form}, using the nonnegative coefficients whose sum is the diagonal entry, to obtain $(Q_sf)_i^2\le(Q_s)_{ii}\sum_{r\in[K]\setminus\{i\}}[-(Q_s)_{ir}](f_i-f_r)^2$. For $i\ne r$, the assignment identity gives $\sum_{s=1}^K[-(Q_s)_{ir}]=v_r$, which allows us to sum the bound over sources. For the first inequality, use $(Q_s)_{ii}\le\rho_s=\|\rho\|_1q_s$, multiply by $v_i/q_s$, and sum over sources and targets to obtain
\begin{align*}
\sum_{s,i=1}^K\frac{v_i}{q_s}(Q_sf)_i^2\le\|\rho\|_1\sum_{i,r=1}^Kv_iv_r(f_i-f_r)^2=2\|\rho\|_1\Var_v(f).
\end{align*}
The terms with identical indices are zero, so including them does not change the sum, and expanding the square gives the final variance identity. This proves the first inequality in \cref{eq:assignment-moments}.

For the second bound, use the other load constraint, $v_i(Q_s)_{ii}/v_s\le\rho_i$, in the same Cauchy-Schwarz estimate and sum over sources to obtain
\begin{align*}
\sum_{s,i=1}^K\frac{v_i}{v_s}(Q_sf)_i^2\le\sum_{i,r=1}^K\rho_iv_r(f_i-f_r)^2=\|\rho\|_1\sum_{i,r=1}^Kq_iv_r(f_i-f_r)^2.
\end{align*}
Subtract the mean of $f$ under $(v+q)/2$ from both entries in each difference and apply $(a-b)^2\le2a^2+2b^2$. The resulting squared deviations are averaged under the two probability vectors, whose average is the distribution used for centering, so the last sum is bounded by $4\Var_{(v+q)/2}(f)$, proving the second inequality in \cref{eq:assignment-moments}.

\emph{Persistence bound.} Let $F_s\coloneqq C_wQ_s=p_s\Psi_s$. Centering at the $w$-mean does not increase the $w$-weighted norm, because it subtracts the squared mean from the squared norm, and $w\le2v$, so $\|F_sf\|_w^2\le\|Q_sf\|_w^2\le2\|Q_sf\|_v^2$. Recall that the error matrix subtracts the centering matrix from the estimation matrix. Since $p_s/\lambda_s=p_s^2/x_s$ and $p_sE_sf=F_sf-p_sC_wf$, apply $(a+b)^2\le2a^2+2b^2$ to each coordinate of this difference, then use $x\ge q/2$ and $\sum_{s=1}^Kp_s^2/x_s\le2$. For this last bound, $x\ge p/2$ gives $p_s^2/x_s\le2p_s$, and summing uses the fact that $p$ is a probability vector. These inequalities give
\begin{align*}
\sum_{s=1}^K\frac{p_s}{\lambda_s}\|E_sf\|_w^2
\le2\sum_{s=1}^K\frac{\|F_sf\|_w^2}{x_s}+2\sum_{s=1}^K\frac{p_s^2}{x_s}\|C_wf\|_w^2
\le8\sum_{s=1}^K\frac{\|Q_sf\|_v^2}{q_s}+4\Var_w(f).
\end{align*}
By the first inequality in \cref{eq:assignment-moments}, $v\le2w$, and $\|\rho\|_1\ge1$, the right side is at most $16\|\rho\|_1\Var_v(f)+4\Var_w(f)\le36\|\rho\|_1\Var_w(f)$. For the variance comparison, evaluate the squared deviations under $v$ at the $w$-mean, then use the fact that variance minimizes the average squared deviation over all choices of center. Since $w\le\tfrac87p$, the same argument bounds this expression by $\tfrac{288}7\|\rho\|_1\Var_p(f)\le\tfrac{864}7\log(e/v_{\min})\Var_p(f)\le L_0\Var_p(f)$, proving the persistence bound.

\emph{One-round bound.} We now prove the second endpoint bound. Since $\sum_{s=1}^Kp_s\Psi_sf=C_wf$, the vectors $\Psi_s f$ have this mean when the source is drawn with law $p$. Subtracting this mean does not increase their second moment, giving $\sum_{s=1}^Kp_s\|E_sf\|_w^2\le\sum_{s=1}^Kp_s\|\Psi_sf\|_w^2=\sum_{s=1}^K\|F_sf\|_w^2/p_s$. Now $p_s\ge\tfrac78w_s\ge\tfrac7{16}v_s$, so the earlier norm comparison and the second inequality in \cref{eq:assignment-moments} give
\begin{align*}
\sum_{s=1}^K\frac{\|F_sf\|_w^2}{p_s}\le\frac{32}7\sum_{s=1}^K\frac{\|Q_sf\|_v^2}{v_s}\le\frac{128}7\|\rho\|_1\Var_{(v+q)/2}(f).
\end{align*}
Finally, $v\le2w\le\tfrac{16}7p$ gives $(v+q)/2\le\tfrac{16}7x$, hence $\Var_{(v+q)/2}(f)\le\tfrac{16}7\Var_x(f)$ by the same comparison of variances. Multiplying these bounds shows that the one-round error is at most $\tfrac{2048}{49}\|\rho\|_1\Var_x(f)\le\tfrac{6144}{49}\log(e/v_{\min})\Var_x(f)\le L_0\Var_x(f)$, completing condition (iii) and the proof of the lemma.
\end{proof}

\section{Proof of the estimation lemma}
\label{app:energy}

Throughout this appendix, we work under the hypotheses of \cref{lem:energy}.

We prove \cref{lem:energy} using the coordinates in the common eigenbasis of the holding family from \cref{eq:coordinates}. We first express twice the steering gain as a nonnegative quadratic form, which we call the steering energy, minus a small diagonal term, and then represent the geometric correlations in the quadratic forms by random partitions of the block (\cref{lem:partition}). The corrected second moment separates into three terms: the squared block mean, the squared chunk means, and the covariance of the reference errors, with the first two bounded directly in the proof of \cref{lem:energy}. To bound the covariance, we prove \cref{lem:reference-bound} by grouping the eigenvalues into dyadic bands, comparing the correlations in different bands through a change of rate (\cref{lem:rate-comparison}), and applying the two endpoint bounds to the errors caused by each band (\cref{lem:band-errors}).

\paragraph{The steering energy.}
We first express the steering gain in the common eigenbasis, so that the contribution of each coordinate has a geometric weight in the distance between rounds. For $\alpha\ge0$ and $0<\tau\le1$, let $\psi_\tau(\alpha)\coloneqq1/\nu_\tau(\alpha)-1=\tau\alpha/(2+\tau\alpha)$ and $\chi_\tau(\alpha)\coloneqq(1-\tau)\nu_\tau(\alpha)$, and let $\chi_n\coloneqq\chi_\tau(\alpha_n)$. By \cref{lem:spectral}(ii), for rounds $t<t'$ the matrix $(I-P_\tau)P_\tau^{t'-t-1}$ multiplies $\phi_n$ by $(1-\nu)\nu^{t'-t-1}=\psi_\tau(\alpha_n)\nu^{t'-t}$, where $\nu\coloneqq\nu_\tau(\alpha_n)$, while $p^{\mathsf T}(I-P_\tau)=0$ removes the constant coordinate. Thus \cref{eq:coordinates} gives
\[
\ip{\ell_t}{(I-P_\tau)P_\tau^{t'-t-1}\ell_{t'}}_p=\sum_{n=1}^{K-1}\psi_\tau(\alpha_n)\nu_\tau(\alpha_n)^{t'-t}[\ell_t]_n[\ell_{t'}]_n.
\]
Substituting this identity into \cref{eq:gain} incorporates the restart probability into the geometric weight, and the terms indexed by $(t,t')$ and $(t',t)$ agree in the resulting symmetric double sum. Adding the diagonal terms and then subtracting them therefore gives $2\Gamma(\ell_{1:h})=\cQ_*(\ell_{1:h})-\cQ_{\rm diag}(\ell_{1:h})$, where, for every real sequence $f_1,\ldots,f_h\in\R^K$, we define
\begin{align}
\cQ_*(f_{1:h})&\coloneqq\E_\tau\Bigl[\sum_{n=1}^{K-1}\psi_\tau(\alpha_n)\sum_{t,t'=1}^h\chi_n^{|t-t'|}[f_t]_n[f_{t'}]_n\Bigr],\nonumber\\
\cQ_{\rm diag}(f_{1:h})&\coloneqq\E_\tau\Bigl[\sum_{n=1}^{K-1}\psi_\tau(\alpha_n)\sum_{t=1}^h[f_t]_n^2\Bigr].
\label{eq:energy}
\end{align}
We call $\cQ_*$ the steering energy and first bound the diagonal term, which accounts for the difference between this energy and twice the gain. For losses in $[0,1]^K$, the bounds $\psi_\tau(\alpha)\le\tau\alpha/2$ and $\sum_{n=1}^{K-1}\alpha_n[\ell_t]_n^2=\Var_x(\ell_t)\le\tfrac14$ give $0\le\cQ_{\rm diag}(\ell_{1:h})\le h\E_\tau[\tau]/8=\mu h/8$, where the last equality uses \cref{eq:mean-rate}.

\paragraph{Geometric correlations and random partitions.}
Both the steering energy and the second moment of the estimate involve double sums with geometric weights, so we introduce matrices that let us compare these sums. For $0\le\chi<1$, let $\mathcal K_\chi$ be the $h\times h$ correlation matrix with entries $(\mathcal K_\chi)_{tt'}\coloneqq\chi^{|t-t'|}$, where $0^0=1$, so that $\mathcal K_0$ is the identity matrix $I$. For symmetric $h\times h$ matrices $U$ and $U'$, we write $U\preceq U'$ if $y^{\mathsf T} Uy\le y^{\mathsf T} U'y$ for every $y\in\R^h$, and for vectors $g_1,\ldots,g_h\in\R^K$ we let $\|g\|_U^2\coloneqq\sum_{t,t'=1}^hU_{tt'}\ip{g_t}{g_{t'}}_p$, with matrices on $\R^K$ acting on these sequences one round at a time. Expanding each $g_t$ in the $p$-orthonormal basis formed by $\one$ and the vectors $\phi_n$ expresses $\|g\|_U^2$ as a sum of quadratic forms $y^{\mathsf T} Uy$, one for each basis vector, so $U\preceq U'$ implies $\|g\|_U^2\le\|g\|_{U'}^2$. This also shows that the expression is nonnegative whenever $U$ is positive semidefinite.

The restart rule already gives one example of these geometric weights: two rounds $t\le t'$ lie in the same chunk exactly when none of the $t'-t$ boundaries between them is a restart, which has probability $(1-\tau)^{t'-t}$. We now use the same construction with an arbitrary cut probability to express each quadratic form as an expected sum of squares and to compare different correlation matrices.

\begin{lemma}[Geometric correlations and random partitions]
\label{lem:partition}
Let $0\le\chi<1$, and let $\cP$ be the random partition of $\{1,\ldots,h\}$ into intervals obtained by cutting each of the $h-1$ boundaries between consecutive rounds independently with probability $1-\chi$.
\begin{enumerate}
\item Every $y\in\R^h$ satisfies $y^{\mathsf T}\mathcal K_\chi y=\E\bigl[\sum_{J\in\cP}(\sum_{t\in J}y_t)^2\bigr]$, so in particular $\mathcal K_\chi\succeq0$.
\item If $0\le\chi'\le\chi$, then $\mathcal K_\chi\preceq\frac{2(1-\chi')}{1-\chi}\,\mathcal K_{\chi'}$. In particular, $\mathcal K_\chi\preceq\frac2{1-\chi}\,I$.
\item $\mathcal K_\chi\succeq\one\one^{\mathsf T}/(1+(1-\chi)(h-1))$, where $\one\in\R^h$ is the all-ones vector.
\end{enumerate}
For vectors $g_1,\ldots,g_h\in\R^K$, the identity in (i) also holds in the form $\|g\|_{\mathcal K_\chi}^2=\E\bigl[\sum_{J\in\cP}\|\sum_{t\in J}g_t\|_p^2\bigr]$.
\end{lemma}

The partitions in this lemma are used to represent and compare quadratic forms, so a cut probability different from $\tau$ is only an auxiliary choice in the proof and does not change the algorithm's restart rule.

\begin{proof}
We first prove (i) by expanding the squared interval sums. For each ordered pair of rounds, the term $y_ty_{t'}$ appears precisely when $t$ and $t'$ belong to the same interval, which happens exactly when none of the $|t-t'|$ boundaries between them is cut and therefore has probability $\chi^{|t-t'|}$. Taking expectations gives the stated identity, and its right side is nonnegative because it is an expectation of a sum of squares, proving the positive semidefiniteness assertion as well.

To prove (ii), we construct the partition with fewer cuts by independently removing cuts from a finer partition. Let $\chi''\coloneqq(\chi-\chi')/(1-\chi')\in[0,1)$, so that $(1-\chi')(1-\chi'')=1-\chi$, draw a partition $\cP'$ with cut probability $1-\chi'$, and retain each of its cuts independently with probability $1-\chi''$. The retained cuts are independent with probability $(1-\chi')(1-\chi'')=1-\chi$, so the resulting partition has the law of $\cP$.

Conditional on $\cP'$, let $z$ be the vector whose entries are the sums of $y$ over the consecutive intervals of $\cP'$. Each interval of $\cP$ is a union of consecutive intervals of $\cP'$, with the boundaries between those intervals retained independently with probability $1-\chi''$, so (i), applied to $z$, identifies the conditional expectation of $\sum_{J\in\cP}(\sum_{t\in J}y_t)^2$ given $\cP'$ as $z^{\mathsf T}\mathcal K_{\chi''}z$, where $\mathcal K_{\chi''}$ now has the size of $z$. This symmetric matrix has nonnegative entries and row sums at most $1+2\chi''/(1-\chi'')\le2/(1-\chi'')$, because at most two entries occur at each positive distance from the diagonal, and hence its largest eigenvalue is at most $2/(1-\chi'')$. It follows that $z^{\mathsf T}\mathcal K_{\chi''}z\le(2/(1-\chi''))\|z\|^2$, where $\|z\|^2=\sum_{J'\in\cP'}(\sum_{t\in J'}y_t)^2$, and $1/(1-\chi'')=(1-\chi')/(1-\chi)$ gives the required factor. Averaging over $\cP'$ and applying (i) to both partitions proves the first inequality in (ii), while the choice $\chi'=0$, for which $\mathcal K_0=I$, gives the second.

For (iii), let $|\cP|$ be the number of intervals of $\cP$, and apply the Cauchy-Schwarz inequality to their sums to obtain $\sum_{J\in\cP}(\sum_{t\in J}y_t)^2\ge(\sum_{t=1}^hy_t)^2/|\cP|$. Jensen's inequality for the convex function $u\mapsto1/u$ on $(0,\infty)$ gives $\E[1/|\cP|]\ge1/\E[|\cP|]=1/(1+(1-\chi)(h-1))$, since the number of intervals is one plus the number of cuts, and taking expectations in the preceding inequality and using (i) proves (iii). Finally, applying (i) to each coordinate of the vectors $g_t$ in a $p$-orthonormal basis and summing over those coordinates gives the vector identity, which completes the proof.
\end{proof}

We can now verify the promised nonnegativity of the steering energy: by \cref{lem:partition}(i), every inner double sum in \cref{eq:energy} is nonnegative, and its weight is nonnegative, so $\cQ_*\ge0$. Let $[f_{1:h}]_n$ denote the sequence $([f_t]_n)_{t=1}^h$ of $n$-th coordinates, so that the same energy can be written as $\cQ_*(f_{1:h})=\E_\tau\bigl[\sum_{n=1}^{K-1}\psi_\tau(\alpha_n)[f_{1:h}]_n^{\mathsf T}\mathcal K_{\chi_n}[f_{1:h}]_n\bigr]$. Combining $2\Gamma=\cQ_*-\cQ_{\rm diag}$ with the diagonal bound gives the first inequality below, while nonnegativity gives the lower bound on the gain, for losses in $[0,1]^K$:
\begin{align}
\cQ_*(\ell_{1:h})\le2\Bigl(\Gamma(\ell_{1:h})+\frac{\mu h}{16}\Bigr),
\qquad
\Gamma(\ell_{1:h})\ge-\frac{\cQ_{\rm diag}(\ell_{1:h})}2\ge-\frac{\mu h}{16}.
\label{eq:gain-energy}
\end{align}

\paragraph{Cut probabilities.}
The geometric ratios in the steering energy and the reference covariance are $\chi_\tau(\alpha)$, and the partition comparisons depend on their corresponding cut probabilities. Let $\bar\chi_\tau(\alpha)\coloneqq1-\chi_\tau(\alpha)$, which has the explicit expression
\[
\bar\chi_\tau(\alpha)=\tau+(1-\tau)\frac{\tau\alpha}{2(1+\tau\alpha)}.
\]
We will compare these cut probabilities at different eigenvalues and relate them to the energy weight $\psi_\tau$, using the following properties. For $0<\tau\le1$ and $\alpha\ge0$, the function $\bar\chi_\tau(\alpha)$ is nondecreasing in $\alpha$, and
\begin{align}
\bar\chi_\tau(\alpha)\ge\tau,
\qquad
\bar\chi_\tau(\alpha)\ge\frac{\tau\alpha}{2(1+\tau\alpha)},
\qquad
\bar\chi_\tau(2\alpha)\le2\bar\chi_\tau(\alpha),
\qquad
\psi_\tau(\alpha)\ge\frac{\bar\chi_\tau(\alpha)}8\quad\text{if }\alpha\ge\tfrac12.
\label{eq:scalar}
\end{align}
To verify these claims, monotonicity follows from that of $\alpha\mapsto\tau\alpha/(1+\tau\alpha)$, and the first two bounds follow from the explicit expression: its second term is nonnegative, and $\tau+(1-\tau)y\ge y$ for $y\le1$. The third bound follows from the direct comparison
\[
\bar\chi_\tau(2\alpha)=\tau+(1-\tau)\frac{\tau\alpha}{1+2\tau\alpha}\le2\tau+(1-\tau)\frac{\tau\alpha}{1+\tau\alpha}=2\bar\chi_\tau(\alpha).
\]
For the last bound, let $z\coloneqq\tau\alpha$, and use $\alpha\ge\tfrac12$ to obtain $\tau\le2z$, hence $\bar\chi_\tau(\alpha)\le\tau+z/2\le5z/2$, together with the bound $\bar\chi_\tau(\alpha)\le1$. If $z\le1$, then $\psi_\tau(\alpha)=z/(2+z)\ge z/3\ge\tfrac2{15}\bar\chi_\tau(\alpha)$, while if $z>1$, then $\psi_\tau(\alpha)\ge\tfrac13\ge\tfrac13\bar\chi_\tau(\alpha)$. Both cases give $\psi_\tau(\alpha)\ge\bar\chi_\tau(\alpha)/8$, completing the verification of \cref{eq:scalar}.

\paragraph{The reference errors.}
We next isolate the contribution of the estimation errors under the reference law, including their correlations across rounds within a chunk. For a deterministic real sequence $f_1,\ldots,f_h\in\R^K$, let
\[
V_{\rm err}(f_{1:h})\coloneqq\E_{\tau,\cP_\tau}\Bigl[\sum_{J\in\cP_\tau}\E_0\Bigl[\Bigl\|\sum_{t\in J}E_{X_t}f_t\Bigr\|_w^2\Bigr]\Bigr],
\]
where, under $\E_0$, the actions in each chunk form a stationary $P_\tau$-chain. Each summand is an average of a squared norm of a linear function of the sequence, so $V_{\rm err}$ is a nonnegative quadratic form in $f_{1:h}$. We define it for arbitrary real sequences because the proof will first treat the components of the losses in one band of eigenvalues, which need not themselves be loss vectors. The next lemma bounds this quadratic form by the steering energy, and we first use it to prove \cref{lem:energy}, then prove the reference error bound in the remainder of the appendix.

\begin{lemma}[Reference errors bounded by the steering energy]
\label{lem:reference-bound}
Let $(\Psi_s)_s$ satisfy conditions (ii) and (iii) of \cref{def:balanced}. Every deterministic real sequence $f_1,\ldots,f_h\in\R^K$ satisfies
\[
V_{\rm err}(f_{1:h})\le\frac{3072\,mL_0}\kappa\,\cQ_*(f_{1:h}).
\]
\end{lemma}

Here $m$ counts the dyadic bands used in the proof, including any empty bands, and the factor $m$ is the only additional factor needed to pass from the bound for one band to the bound for a general sequence.

\begin{proof}[Proof of \cref{lem:energy}]
We first separate the corrected second moment into the squared block mean, the squared chunk means, and the reference error term, using the chunk likelihood ratios and conditional independence. We then bound each of the two mean terms by the steering energy through the partition comparisons, and apply \cref{lem:reference-bound} to the error term.

\emph{Change of measure and decomposition.} Condition on the rate $\tau$ and the partition $\cP_\tau$, so that each chunk starts from a fresh draw from $p$ and uses fresh randomness, making the corrected chunk outputs independent. By \cref{lem:probing}, these outputs have means $C_wL_J$ and second moments at most $4/3$ times those of their reference sums $G_J$, so the second moment of their sum is the squared total mean plus the sum of their individual variances. Applying the second-moment bound and separating the squared mean from each reference second moment gives
\[
\begin{aligned}
\E[\|Z\|_w^2\mid\tau,\cP_\tau]
&=\|C_wL\|_w^2+\sum_{J\in\cP_\tau}\left(\E_{\rm tilt}\bigl[\|G_J/R_J\|_w^2\bigr]-\|C_wL_J\|_w^2\right)\\
&\le\|C_wL\|_w^2+\frac43\sum_{J\in\cP_\tau}\E_0\bigl[\|G_J-C_wL_J\|_w^2\bigr]+\frac13\sum_{J\in\cP_\tau}\|C_wL_J\|_w^2.
\end{aligned}
\]
The same conditional means give $\E[Z]=C_wL$, proving the unbiasedness assertion. Under the reference law, $G_J-C_wL_J=\sum_{t\in J}E_{X_t}\ell_t$, so averaging the middle sum over the rate and the partition gives $V_{\rm err}(\ell_{1:h})$. Let $V_{\rm mean}\coloneqq\E_{\tau,\cP_\tau}\bigl[\sum_{J\in\cP_\tau}\|C_wL_J\|_w^2\bigr]$ and, within this proof, let $\cQ_*\coloneqq\cQ_*(\ell_{1:h})$. Since $w\le2p$, we have $\Var_w(f)\le2\Var_p(f)$ for every $f$, as follows by evaluating the weighted squared deviations for $w$ at the mean under $p$ and then applying the pointwise bound on the weights.

\emph{The squared block mean.} We first compare the energy with the square of the sum over all rounds. Fix a rate $\tau$ and an eigenvalue $\alpha\coloneqq\alpha_n$, and let $\bar\chi\coloneqq\bar\chi_\tau(\alpha)$ and $\psi\coloneqq\psi_\tau(\alpha)$, so that \cref{lem:partition}(iii) and \cref{eq:scalar} give
\[
\psi\,\mathcal K_{\chi_\tau(\alpha)}
\succeq\frac{\psi}{1+\bar\chi(h-1)}\one\one^{\mathsf T}
\succeq\frac{\bar\chi}{8(1+\bar\chi h)}\one\one^{\mathsf T}
\succeq\frac{\kappa}{8(1+\kappa h)}\one\one^{\mathsf T}
\succeq\frac\kappa9\one\one^{\mathsf T},
\]
where the third step uses $\bar\chi\ge\tau\ge\kappa$ and the fact that $y\mapsto y/(1+yh)$ is increasing, and the last uses $\kappa h\le\tfrac18$. Applying this matrix comparison to the coordinate sequences $[\ell_{1:h}]_n$ and averaging over the rate gives $\cQ_*\ge(\kappa/9)\sum_{n=1}^{K-1}(\sum_{t=1}^h[\ell_t]_n)^2=(\kappa/9)\sum_{n=1}^{K-1}[L]_n^2=(\kappa/9)\Var_p(L)$ by \cref{eq:coordinates}, since summing a coordinate over time gives the same coordinate of the block total. Consequently, $\|C_wL\|_w^2=\Var_w(L)\le2\Var_p(L)\le18\cQ_*/\kappa$, which is the required bound for the squared block mean.

\emph{The squared chunk means.} We now apply the partition identity to the algorithm's actual chunks. Fix the rate, so that these chunks form a partition with cut probability $\tau$, and let $L_J\coloneqq\sum_{t\in J}\ell_t$, as above. Then \cref{lem:partition}(i) and \cref{eq:coordinates} give
\[
\E_{\cP_\tau}\Bigl[\sum_{J\in\cP_\tau}\Var_p(L_J)\Bigr]=\sum_{n=1}^{K-1}[\ell_{1:h}]_n^{\mathsf T}\mathcal K_{1-\tau}[\ell_{1:h}]_n.
\]
Since $\bar\chi_\tau(\alpha_n)\ge\tau$, the energy uses a cut probability at least as large as that of the chunks, and \cref{lem:partition}(ii) with $\chi\coloneqq1-\tau$ and $\chi'\coloneqq\chi_\tau(\alpha_n)$ gives $\mathcal K_{1-\tau}\preceq(2\bar\chi_\tau(\alpha_n)/\tau)\mathcal K_{\chi_\tau(\alpha_n)}$. The latter matrix is at most $(16/\kappa)\psi_\tau(\alpha_n)\mathcal K_{\chi_\tau(\alpha_n)}$ by \cref{eq:scalar} and $\tau\ge\kappa$, so summing over the coordinates and averaging over the rate gives $\E_{\tau,\cP_\tau}\bigl[\sum_{J\in\cP_\tau}\Var_p(L_J)\bigr]\le16\cQ_*/\kappa$. The variance comparison between $w$ and $p$ therefore gives $V_{\rm mean}\le2\E_{\tau,\cP_\tau}\bigl[\sum_{J\in\cP_\tau}\Var_p(L_J)\bigr]\le32\cQ_*/\kappa$, completing the bound for the squared chunk means.

\emph{The reference errors.} The remaining term is covered by \cref{lem:reference-bound}, which gives $V_{\rm err}(\ell_{1:h})\le3072mL_0\cQ_*/\kappa$.

\emph{Conclusion.} Averaging the conditional decomposition and substituting the three bounds, with their respective coefficients, gives
\[
\E[\|Z\|_w^2]
\le\frac{18+(4/3)\,3072\,mL_0+32/3}{\kappa}\cQ_*
\le\frac{2^{13}mL_0}{\kappa}\cQ_*,
\]
because $mL_0\ge1$. Finally, \cref{eq:gain-energy} gives $\cQ_*\le2(\Gamma+\mu h/16)$, and substituting this bound yields the claimed second-moment inequality, completing the proof of \cref{lem:energy}.
\end{proof}

It remains to prove \cref{lem:reference-bound}, for which we need two comparisons. The first compares the time correlations of errors and of the steering energy at different eigenvalues by changing the rate (\cref{lem:rate-comparison}), and the second uses the endpoint bounds to control the errors in all bands that are caused by a loss component in one band (\cref{lem:band-errors}).

\paragraph{Matching dyadic rates.}
At a fixed switching rate $\tau$, an error component and the loss component that produces it can have different time correlations because they belong to different eigenvalue bands. We will compare these correlations by changing the rate, and the dyadic law in \cref{eq:rate-law}, which assigns each rate below one a probability proportional to its inverse, lets us account explicitly for the change in probability. For $\alpha>0$, let the two averaged correlation matrices be
\[
\mathcal K^{\rm cov}_\alpha\coloneqq\E_\tau\bigl[\mathcal K_{\chi_\tau(\alpha)}\bigr],
\qquad
\mathcal K^{\rm gain}_\alpha\coloneqq\E_\tau\bigl[\bar\chi_\tau(\alpha)\,\mathcal K_{\chi_\tau(\alpha)}\bigr].
\]
The first matrix describes the reference covariance at eigenvalue $\alpha$, averaged over the rate. For the dyadic bands used below, the second matrix gives, up to an absolute constant, a lower bound on the steering energy at eigenvalues in $[\alpha,2\alpha)$, by \cref{lem:partition}(ii) and \cref{eq:scalar}, with the precise comparison established in \cref{eq:within-band}.

\begin{lemma}[Dyadic rate shift]
\label{lem:rate-comparison}
Let $k$ and $k'$ be integers with $-1\le k,k'\le m-2$. Then
\[
2^{\min\{k,k'\}}\,\mathcal K^{\rm cov}_{2^{k+1}}\preceq\frac{12}\kappa\,\mathcal K^{\rm gain}_{2^{k'}}.
\]
\end{lemma}

Up to absolute constant factors, the covariance matrix on the left controls the averaged covariance of error components in $[2^k,2^{k+1})$, while the matrix on the right gives a lower bound on the energy of loss components in $[2^{k'},2^{k'+1})$. Solving the displayed comparison for the covariance introduces the factor $12/(\kappa2^{\min\{k,k'\}})$, whose factor $2^{-\min\{k,k'\}}$ is exactly the weight used in \cref{lem:band-errors} to count errors in the first band caused by a loss in the second. Thus the rate comparison contributes $12/\kappa$, with the remaining weight handled by the endpoint bounds in that lemma.

\begin{proof}
We compare the matrices at specified rates using \cref{lem:partition}(ii) and \cref{eq:scalar}, and then average with the dyadic probabilities. Let $\tau_j\coloneqq\kappa2^j$ for $0\le j\le m$, so that $\Prb[\tau=\tau_j]=2^{-j-1}$ for $j<m$ and $\Prb[\tau=\tau_m]=\kappa$. We first record a bound that will be used in both cases of the proof: for every rate $\tau\ge\kappa$, the second inequality in \cref{eq:scalar}, applied at $2^{k+1}$, gives $\bar\chi_\tau(2^{k+1})\ge\tau2^k/(1+\tau2^{k+1})$, hence
\begin{align}
\frac{2^k}{\bar\chi_\tau(2^{k+1})}\le\frac1\tau+2^{k+1}\le\frac3{2\kappa},
\label{eq:ratio-bound}
\end{align}
because $\tau\ge\kappa$ and $\kappa2^{k+1}\le\tfrac12$, the latter following from the assumed range of the band index.

Suppose first that $k<k'$, in which case we can compare the two matrices at the same rate. Then $k+1\le k'$, so monotonicity gives $\bar\chi_\tau(2^{k+1})\le\bar\chi_\tau(2^{k'})$, and applying \cref{lem:partition}(ii) with these ordered cut probabilities, followed by \cref{eq:ratio-bound}, yields
\[
2^k\mathcal K_{\chi_\tau(2^{k+1})}\preceq\frac{2^{k+1}}{\bar\chi_\tau(2^{k+1})}\,\bar\chi_\tau(2^{k'})\mathcal K_{\chi_\tau(2^{k'})}\preceq\frac3\kappa\,\bar\chi_\tau(2^{k'})\mathcal K_{\chi_\tau(2^{k'})}.
\]
Averaging over $\tau$ gives the desired comparison with a constant no larger than the one in the statement, proving the case of a smaller error-band index.

Suppose next that $k\ge k'$, so that the eigenvalue used to represent the covariance is larger than the one used to represent the energy. We seek a comparison between the rate $\tau_j$ on the left and a faster rate $\tau_{j'}$ on the right, where $j'\coloneqq j+k+1-k'$, to match the products of rate and representative eigenvalue. For rates below one, the dyadic law makes $\tau_{j'}$ less likely than $\tau_j$ by the factor $2^{k+1-k'}$, and the matrix comparison below supplies the inverse power needed to cancel this change of probability. When $j'\ge m$, the proposed faster rate is at least one, and we will instead use the atom at $\tau=1$.

Consider first an index $j$ with $j'<m$, for which $\tau_{j'}=2^{k+1-k'}\tau_j$ satisfies $\tau_{j'}2^{k'}=\tau_j2^{k+1}$. Let $z$ denote this common value, so that the explicit expression for $\bar\chi$ gives
\[
\bar\chi_{\tau_{j'}}(2^{k'})-\bar\chi_{\tau_j}(2^{k+1})=(\tau_{j'}-\tau_j)\Bigl(1-\frac z{2(1+z)}\Bigr)\ge0.
\]
Thus the shifted rate has a cut probability at least as large as that on the covariance side, which is the ordering needed to apply \cref{lem:partition}(ii). Applying that comparison and then \cref{eq:ratio-bound} gives
\[
2^{k'}\mathcal K_{\chi_{\tau_j}(2^{k+1})}\preceq\frac{2^{k'+1}}{\bar\chi_{\tau_j}(2^{k+1})}\,\bar\chi_{\tau_{j'}}(2^{k'})\mathcal K_{\chi_{\tau_{j'}}(2^{k'})}\preceq\frac{3\cdot2^{k'-k}}{\kappa}\,\bar\chi_{\tau_{j'}}(2^{k'})\mathcal K_{\chi_{\tau_{j'}}(2^{k'})}.
\]
Since $j'<m$, both rates have probabilities given by the part of the dyadic law below one, and $\Prb[\tau=\tau_j]=2^{k+1-k'}\Prb[\tau=\tau_{j'}]$. Multiplying the last display by $\Prb[\tau=\tau_j]$ cancels the dependence on the difference between the band indices and gives
\[
2^{k'}\Prb[\tau=\tau_j]\,\mathcal K_{\chi_{\tau_j}(2^{k+1})}\preceq\frac6\kappa\,\Prb[\tau=\tau_{j'}]\,\bar\chi_{\tau_{j'}}(2^{k'})\mathcal K_{\chi_{\tau_{j'}}(2^{k'})}.
\]
Distinct indices $j$ give distinct $j'$, so summing over the indices for which the shifted rate is below one bounds their contribution to $2^{k'}\mathcal K^{\rm cov}_{2^{k+1}}$ by $(6/\kappa)\mathcal K^{\rm gain}_{2^{k'}}$. It remains to bound the contribution of the other indices, which we do using the rate-one atom.

For these remaining indices, $j'\ge m$ and hence $\tau_j2^{k+1}=\kappa2^{j'+k'}\ge\kappa2^{m+k'}=2^{k'}\ge\tfrac12$. Since $z\mapsto z/(2(1+z))$ is increasing, the second bound in \cref{eq:scalar} gives $\bar\chi_{\tau_j}(2^{k+1})\ge\tfrac16$, and \cref{lem:partition}(ii) then gives $\mathcal K_{\chi_{\tau_j}(2^{k+1})}\preceq12I$. Let $j_0\coloneqq m-k-1+k'$, which satisfies $0\le j_0<m$, so these indices have total probability $\sum_{j=j_0}^{m-1}2^{-j-1}+\kappa=2^{-j_0}=2^{k+1-k'}\kappa$, including the original atom at rate one. Their contribution to $2^{k'}\mathcal K^{\rm cov}_{2^{k+1}}$ is consequently at most $12\cdot2^{k'}\cdot2^{k+1-k'}\kappa\,I=24\kappa2^kI\preceq6I$, where the last step uses $\kappa2^k\le\tfrac14$.

Finally, the atom at $\tau=1$ gives $\mathcal K^{\rm gain}_{2^{k'}}\succeq\kappa\,\bar\chi_1(2^{k'})\mathcal K_{\chi_1(2^{k'})}=\kappa I$, because $\chi_1=0$, so $6I\preceq(6/\kappa)\mathcal K^{\rm gain}_{2^{k'}}$. Adding this bound to the contribution from the shifted rates below one proves the stated comparison and completes the proof.
\end{proof}

\paragraph{Errors caused by one band.}
We now use the endpoint bounds to control how a loss component in one eigenvalue band produces estimation errors across all bands. For each coordinate $i$, let $\cT_i:\R^K\to\R^K$ be the linear map defined by $(\cT_if)_s\coloneqq(E_sf)_i$, so that $\cT_if$ lists the errors in output coordinate $i$ as the source varies. Condition (ii) of \cref{def:balanced} makes every $\cT_if$ $p$-centered and gives $\cT_i\one=0$, since averaging the error matrices under the source law gives zero and each error matrix annihilates constant vectors. On the subspace of $p$-centered vectors, $A$ has eigenvalues $\alpha_n\ge\tfrac12$ by \cref{lem:spectral}, and $A^{-1}$ denotes its inverse on that subspace.

For an integer $-1\le k\le m-2$, let $\Pi_k$ be the orthogonal projection for $\ip\cdot\cdot_p$ onto the span of the eigenvectors $\phi_n$ with $\alpha_n\in[2^k,2^{k+1})$, which we call the $k$-th band, and let $\Pi_k=0$ when that band contains no eigenvalue. The bounds $\tfrac12\le\alpha_n\le\lmax\le1/(4\kappa)=2^{m-2}$ ensure that each eigenvalue lies in exactly one of these $m$ bands and that $\sum_{k=-1}^{m-2}\Pi_k=C_p$, with each $\Pi_k$ commuting with $A$. The next lemma bounds the weighted sum of errors over these bands for a loss component in a single band, using the persistence bound for slower error components and the one-round bound for faster ones.

\begin{lemma}[Errors caused by one band]
\label{lem:band-errors}
Let $(\Psi_s)_s$ satisfy conditions (ii) and (iii) of \cref{def:balanced}, and let $-1\le k'\le m-2$. Every $p$-centered $g$ with $\Pi_{k'}g=g$ satisfies
\[
\sum_{i=1}^Kw_i\sum_{k=-1}^{m-2}2^{-\min\{k,k'\}}\|\Pi_k\cT_ig\|_p^2\le4L_0\|g\|_p^2.
\]
\end{lemma}

\begin{proof}
We first express the two endpoint bounds using $A$, and will then apply them to the two parts of the sum, split at $k'$. To identify the expression needed for the persistence bound, we claim that every $p$-centered $g'$ satisfies $A^{-1}g'=C_p(g'/\lambda)$, where $g'/\lambda$ has entries $g'_s/\lambda_s$. Indeed, let $y\coloneqq g'/\lambda$, and use $x_s=p_s\lambda_s$ and the centering of $g'$ to obtain $x^{\mathsf T}y=\sum_{s=1}^Kp_sg'_s=0$, so $(Ay)_s=\lambda_s(y_s-x^{\mathsf T}y)=g'_s$ by \cref{eq:generator}. Since $A\one=0$, we also have $A\,C_py=g'$, and $C_py$ is $p$-centered, proving the claimed formula for the inverse. Pairing with the centered vector $g'$ removes the constant term and gives $\ip{g'}{A^{-1}g'}_p=\ip{g'}y_p=\sum_{s=1}^Kp_s(g'_s)^2/\lambda_s$.

Apply this identity to $g'=\cT_ig$ and sum with the weights $w_i$, using $\sum_{i=1}^Kw_i(\cT_ig)_s^2=\|E_sg\|_w^2$ to identify the left side of the persistence bound. The two endpoint bounds in \cref{eq:endpoint} then give
\[
\sum_{i=1}^Kw_i\ip{\cT_ig}{A^{-1}\cT_ig}_p\le L_0\Var_p(g),
\qquad
\sum_{i=1}^Kw_i\|\cT_ig\|_p^2\le L_0\ip g{Ag}_p,
\]
where the second inequality also uses $\Var_x(g)=\ip g{Ag}_p$ from \cref{lem:spectral}(i). These inequalities provide the two estimates needed for the slower and faster error components, respectively.

Next, fix $i$ and consider the bands up to the band containing the loss component. Every eigenvalue in band $k$ satisfies $\alpha_n<2^{k+1}$, hence $2^{-k}<2/\alpha_n$, so summing over the bands $k\le k'$ bounds their weighted squared norms by the inverse-operator expression:
\[
\sum_{k=-1}^{k'}2^{-k}\|\Pi_k\cT_ig\|_p^2\le2\sum_{n=1}^{K-1}\frac{[\cT_ig]_n^2}{\alpha_n}=2\ip{\cT_ig}{A^{-1}\cT_ig}_p.
\]
For the bands $k>k'$, the weight is $2^{-k'}$, and orthogonality implies that their projections have total squared norm at most $\|\cT_ig\|_p^2$. Summing both contributions with the weights $w_i$ and applying the two endpoint estimates above therefore bounds the left side of the lemma by
\[
2L_0\Var_p(g)+2^{-k'}L_0\ip g{Ag}_p\le2L_0\|g\|_p^2+2^{-k'}L_0\,2^{k'+1}\|g\|_p^2=4L_0\|g\|_p^2,
\]
because $g$ is $p$-centered and its nonzero eigen-coordinates lie in band $k'$, where the eigenvalues are smaller than $2^{k'+1}$. This proves the weighted error bound for a loss component in one band.
\end{proof}

\begin{proof}[Proof of \cref{lem:reference-bound}]
We now combine the rate comparison with the bound for errors caused by one band. We first express the reference covariance in the common eigenbasis, then bound it for a sequence in one band by shifting the rate and applying \cref{lem:band-errors} to interval sums, and finally combine the band components by the Cauchy-Schwarz inequality.

\emph{Coordinates.} Fix the rate and the partition, and consider rounds $t\le t'$ in one chunk, so that stationarity and the correlation formula from \cref{sec:background} give
\[
\E_0\bigl[(E_{X_t}f_t)_i(E_{X_{t'}}f_{t'})_i\bigr]=\ip{\cT_if_t}{P_\tau^{t'-t}\cT_if_{t'}}_p=\sum_{n=1}^{K-1}\nu_\tau(\alpha_n)^{t'-t}[\cT_if_t]_n[\cT_if_{t'}]_n,
\]
where the constant coordinate vanishes because $\cT_if_t$ is $p$-centered. We now expand the squared norm in the definition of the reference error, sum with the weights $w_i$, and average over the restarts, which keep $t$ and $t'$ in the same chunk with probability $(1-\tau)^{|t-t'|}$. Multiplying this probability by the chain correlation and then averaging over the rate gives
\begin{align}
V_{\rm err}(f_{1:h})=\E_\tau\Bigl[\sum_{i=1}^Kw_i\sum_{n=1}^{K-1}\sum_{t,t'=1}^h\chi_n^{|t-t'|}[\cT_if_t]_n[\cT_if_{t'}]_n\Bigr].
\label{eq:error-covariance}
\end{align}
We next compare each correlation matrix with the matrices at the endpoints of its band. For an eigenvalue $\alpha_n$ in band $k$ and every rate, monotonicity and the third bound in \cref{eq:scalar} give $\bar\chi_\tau(\alpha_n)\le\bar\chi_\tau(2^{k+1})\le2\bar\chi_\tau(2^k)\le2\bar\chi_\tau(\alpha_n)$. Applying \cref{lem:partition}(ii) to compare with the upper endpoint gives the first inequality below, while applying it with the lower endpoint and using the last bound in \cref{eq:scalar} gives the second:
\begin{align}
\mathcal K_{\chi_\tau(\alpha_n)}\preceq4\,\mathcal K_{\chi_\tau(2^{k+1})},
\qquad
\psi_\tau(\alpha_n)\mathcal K_{\chi_\tau(\alpha_n)}\succeq\frac{\bar\chi_\tau(\alpha_n)}8\mathcal K_{\chi_\tau(\alpha_n)}\succeq\frac1{16}\bar\chi_\tau(2^k)\mathcal K_{\chi_\tau(2^k)}.
\label{eq:within-band}
\end{align}

\emph{One band.} Fix a band $k'$ and suppose that $\Pi_{k'}f_t=f_t$ for every $t$, so all loss components lie in that band. Grouping the error-coordinate indices $n$ in \cref{eq:error-covariance} by band and applying the first inequality in \cref{eq:within-band} gives the first bound below after averaging over the rate, and \cref{lem:rate-comparison} then gives the second, with the same averaged matrix for every error band:
\[
\begin{aligned}
V_{\rm err}(f_{1:h})&\le4\sum_{i=1}^Kw_i\sum_{k=-1}^{m-2}\|\Pi_k\cT_if\|_{\mathcal K^{\rm cov}_{2^{k+1}}}^2\\
&\le\frac{48}\kappa\sum_{i=1}^Kw_i\sum_{k=-1}^{m-2}2^{-\min\{k,k'\}}\|\Pi_k\cT_if\|_{\mathcal K^{\rm gain}_{2^{k'}}}^2,
\end{aligned}
\]
where $\Pi_k\cT_if$ denotes the sequence $(\Pi_k\cT_if_t)_{t=1}^h$. To apply the bound for errors caused by one band, we express this common temporal quadratic form using the vector identity in \cref{lem:partition}(i): at each rate, with cut probability $\bar\chi_\tau(2^{k'})$, every sequence $g_1,\ldots,g_h\in\R^K$ satisfies
\[
\|g\|^2_{\mathcal K^{\rm gain}_{2^{k'}}}=\E_\tau\Bigl[\bar\chi_\tau(2^{k'})\,\E\Bigl[\sum_{J\in\cP}\Bigl\|\sum_{t\in J}g_t\Bigr\|_p^2\Bigr]\Bigr],
\]
where, given $\tau$, the expectation inside the brackets is over a random partition $\cP$ with cut probability $\bar\chi_\tau(2^{k'})$. For $g_t=\Pi_k\cT_if_t$, linearity makes the interval sums $\Pi_k\cT_i\sum_{t\in J}f_t$, and every $\sum_{t\in J}f_t$ is $p$-centered and belongs to band $k'$, as each of its summands does. We can therefore use the same auxiliary partition for all error bands and apply \cref{lem:band-errors} to each interval sum before averaging, which yields
\[
\sum_{i=1}^Kw_i\sum_{k=-1}^{m-2}2^{-\min\{k,k'\}}\|\Pi_k\cT_if\|_{\mathcal K^{\rm gain}_{2^{k'}}}^2\le4L_0\|f\|_{\mathcal K^{\rm gain}_{2^{k'}}}^2.
\]
Finally, the coordinates of every $f_t$ vanish outside band $k'$, and the second inequality in \cref{eq:within-band}, with $k'$ in place of $k$, gives $\cQ_*(f_{1:h})\ge\tfrac1{16}\|f\|^2_{\mathcal K^{\rm gain}_{2^{k'}}}$. Combining this with the preceding two bounds gives $V_{\rm err}(f_{1:h})\le(48/\kappa)\cdot4L_0\cdot16\,\cQ_*(f_{1:h})=(3072L_0/\kappa)\cQ_*(f_{1:h})$, proving the reference error bound for a sequence in one band.

\emph{All bands.} For a general sequence, $\cT_i\one=0$ implies that $V_{\rm err}$ is unchanged when each $f_t$ is replaced by $C_pf_t=\sum_{k'=-1}^{m-2}\Pi_{k'}f_t$, and $\cQ_*$ is also unchanged because its definition uses only the nonconstant coordinates. Since $V_{\rm err}$ is a nonnegative quadratic form, the Cauchy-Schwarz inequality for its associated bilinear form bounds each pairing of two band components by the geometric mean of their individual quadratic forms. Applying the Cauchy-Schwarz inequality once more to the sum over the $m$ bands gives $V_{\rm err}(f_{1:h})\le m\sum_{k'=-1}^{m-2}V_{\rm err}(\Pi_{k'}f_{1:h})$, where $\Pi_{k'}f_{1:h}$ denotes the sequence $(\Pi_{k'}f_t)_{t=1}^h$. The energy is diagonal in the common eigenbasis, so $\sum_{k'=-1}^{m-2}\cQ_*(\Pi_{k'}f_{1:h})=\cQ_*(f_{1:h})$, and applying the bound for one band to each $\Pi_{k'}f_{1:h}$ proves the stated reference error inequality.
\end{proof}

\section{Proofs of the lower bounds}
\label{app:lower}

We first prove \cref{thm:lower} and then \cref{lem:few-observers}. For the proof of \cref{thm:lower}, let $\mathcal S$ and $\mathsf m$ satisfy its hypotheses, and for each action $a$ let $\mathcal S_a\coloneqq\{v\in\mathcal S:(a,v)\in E\}$ be the set of members that $a$ observes, so that $|\mathcal S_a|\le\mathsf m$ for every action $a$ and $|\mathcal S_v|\le1$ for every member $v$. Every subset of $\mathcal S$ with at least two members satisfies these same assumptions with the same $\mathsf m$, and $T\ge2$, so we may replace $\mathcal S$ by a subset of size $\min\{|\mathcal S|,T\}$. We therefore assume from now on that $2\le|\mathcal S|\le T$ and prove the lower bound $(|\mathcal S|/\mathsf m)^{1/3}T^{2/3}/(256\Upsilon)$, using the parameters
\[
\Upsilon\coloneqq\log_2(2T),
\qquad
\varsigma\coloneqq\frac1{16\Upsilon},
\qquad
\vartheta\coloneqq\frac1{32\Upsilon}\Bigl(\frac{|\mathcal S|}{\mathsf mT}\Bigr)^{1/3}.
\]
Since $|\mathcal S|\le T\le\mathsf mT$ and $\Upsilon\ge2$, these choices give $\vartheta\le1/64$, which will keep the raw mean losses away from the clipping thresholds.

The proof constructs a reference environment in which all members of $\mathcal S$ have the same expected loss, and for each member $\mathsf j$ another environment in which the raw mean loss of $\mathsf j$ is smaller than that of every other member by $\vartheta$, with the effect of clipping included in \cref{lem:lower-regret}. In the reference environment, switches and rounds outside $\mathcal S$ contribute to the regret lower bound, while in the environment of $\mathsf j$, rounds spent on members other than $\mathsf j$ contribute through the mean gap. We then bound the information that distinguishes each environment from the reference environment and sum these bounds over $\mathsf j$, using a quantity controlled by switches and rounds outside $\mathcal S$, with its dependence on $\mathsf m$ and the parameters of the construction made explicit in \cref{lem:reconstruct,lem:lower-kl}. Finally, \cref{prop:bayes} combines the regret and information bounds to show that some environment forces the required regret.

\paragraph{The multiscale walk.}
For $t\in[T]$, let the parent $\mathrm{pa}(t)$ be $t$ minus the largest power of two that divides $t$, so that $0\le\mathrm{pa}(t)<t$. The walk has one coordinate common to all members of $\mathcal S$, indexed by $0$, which is not an action, and one coordinate for each member. Let $\upsilon_1,\ldots,\upsilon_T$ be independent Gaussian vectors in $\R^{\{0\}\cup\mathcal S}$ with mean zero and covariance $\varsigma^2I$, and let $\Xi_0\coloneqq0$ and $\Xi_t\coloneqq\Xi_{\mathrm{pa}(t)}+\upsilon_t$, with coordinates $\Xi_{c,t}$ for $c\in\{0\}\cup\mathcal S$. Each coordinate follows the multiscale random walk of Dekel, Ding, Koren, and Peres \citep{DekelDingKorenPeres2014}, and the coordinate processes are independent because they use independent Gaussian increments. Independently of the walk, let $\iota_1,\ldots,\iota_T$ be independent signs, uniform on $\{-1,+1\}$, and draw all of the increments and signs before play. The following lemma bounds the lengths of parent chains and the number of parent intervals that can contain a switch, which will let us control both the variance of the walk and the information obtained by switching.

\begin{lemma}[Properties of the parent map]
\label{lem:dyadic}
\begin{enumerate}
\item For every $t\in[T]$, the chain of parents from $t$ to $0$ has at most $\Upsilon$ steps, so $\Var(\Xi_{c,t})\le\Upsilon\varsigma^2$ for every coordinate $c$.
\item Every $u\in\{0,\ldots,T\}$ is the parent of at most $\Upsilon$ rounds.
\item For every sequence $I_1,\ldots,I_T$ of actions, with $I_0$ a symbol distinct from all actions, $\sum_{t=1}^T\one\{I_t\ne I_{\mathrm{pa}(t)}\}\le\Upsilon(N_T+1)$.
\end{enumerate}
\end{lemma}

\begin{proof}
For (i), taking the parent removes the least significant nonzero binary digit of $t$, so the number of steps to $0$ is at most the number of binary digits of $t$, which is at most $1+\lfloor\log_2T\rfloor\le\Upsilon$. The coordinate $\Xi_{c,t}$ is therefore a sum of at most $\Upsilon$ independent increments, each with variance $\varsigma^2$, proving the variance bound in (i).

For (ii), if $\mathrm{pa}(t)=u$, then $t=u+2^k$, where $2^k$ is the largest power of two dividing $t$. Once $u$ and $k$ are fixed, there is at most one possible $t$, and $2^k\le T$ allows at most $1+\lfloor\log_2T\rfloor\le\Upsilon$ values of $k$, proving the bound on the number of rounds with a given parent.

For (iii), let $\mathsf s_u\coloneqq\one\{I_u\ne I_{u-1}\}$ for $u\in[T]$, so that $\mathsf s_1=1$ and $\sum_{u=1}^T\mathsf s_u=N_T+1$. If $I_t\ne I_{\mathrm{pa}(t)}$, then $\mathsf s_u=1$ for some $u$ with $\mathrm{pa}(t)<u\le t$, because different actions at the endpoints require a change somewhere between them. For each fixed largest power $2^k$ dividing $t$, the intervals $(\mathrm{pa}(t),t]$ are the disjoint intervals $(2y2^k,(2y+1)2^k]$ with integers $y\ge0$, so a given $u$ belongs to at most one interval for each power and hence at most $\Upsilon$ intervals in total. Thus, when we sum over $t$, each change counted by the indicators is counted at most this many times, giving (iii) and completing the proof.
\end{proof}

\paragraph{Environments.}
Let $\mathrm{clip}(z)\coloneqq\min\{\tfrac34,\max\{\tfrac14,z\}\}$, which projects a raw loss onto the prescribed interval. For $\mathsf j\in\{0\}\cup\mathcal S$, define environment $\mathsf j$ by the raw member losses
\[
\widetilde\ell^{\,\mathsf j}_t(v)\coloneqq\frac12+\Xi_{0,t}+\iota_t\Xi_{v,t}-\vartheta\one\{v=\mathsf j\},
\qquad v\in\mathcal S,
\]
and let the actual losses be $\ell^{\mathsf j}_t(v)\coloneqq\mathrm{clip}(\widetilde\ell^{\,\mathsf j}_t(v))$ for $v\in\mathcal S$ and $\ell^{\mathsf j}_t(a)\coloneqq1$ for $a\notin\mathcal S$. Environment $0$ is the reference environment, in which all members have the same law of losses, while in environment $\mathsf j\in\mathcal S$, the raw loss of member $\mathsf j$ has mean smaller than that of every other member by $\vartheta$. Clipping can change this mean gap, and \cref{lem:lower-regret} accounts for the resulting error when comparing regrets. Every action outside $\mathcal S$ is worse than every member by at least $\tfrac14$ in every round, since its loss is one and every member's clipped loss is at most three quarters. Before clipping, two members $u,v\in\mathcal S$ satisfy
\begin{align}
\widetilde\ell^{\,\mathsf j}_t(u)-\widetilde\ell^{\,\mathsf j}_t(v)=\iota_t(\Xi_{u,t}-\Xi_{v,t})+\vartheta\bigl(\one\{v=\mathsf j\}-\one\{u=\mathsf j\}\bigr).
\label{eq:lower-gap}
\end{align}
Clipping ensures that all losses lie in $[0,1]$, and we now bound the total expected change it makes to the raw member losses, so that this change can be included explicitly in the regret comparison.

\begin{lemma}[Clipping error]
\label{lem:clip}
In every environment $\mathsf j$, the clipping errors $\mathrm{err}^{\mathsf j}_t(v)\coloneqq|\ell^{\mathsf j}_t(v)-\widetilde\ell^{\,\mathsf j}_t(v)|$ satisfy $\E\bigl[\sum_{t=1}^T\sum_{v\in\mathcal S}\mathrm{err}^{\mathsf j}_t(v)\bigr]\le1/1024$.
\end{lemma}

\begin{proof}
Fix $t$ and $v\in\mathcal S$, and note that $\iota_t$ is independent of the walk, while $\Xi_{v,t}$ is symmetric and independent of $\Xi_{0,t}$. Multiplication by the independent sign therefore preserves the Gaussian law, so the Gaussian part $\mathsf g\coloneqq\Xi_{0,t}+\iota_t\Xi_{v,t}$ of $\widetilde\ell^{\,\mathsf j}_t(v)$ is centered Gaussian and, by \cref{lem:dyadic}(i), has variance $\mathsf v\le2\Upsilon\varsigma^2=1/(128\Upsilon)$. The mean $\tfrac12-\vartheta\one\{v=\mathsf j\}$ lies in $[\tfrac{31}{64},\tfrac12]$, at distance at least $\tfrac7{32}$ from the boundary of $[\tfrac14,\tfrac34]$, so clipping can change the raw loss only when $|\mathsf g|\ge\tfrac7{32}$ and changes it by at most $|\mathsf g|$. Taking expectations and evaluating the first absolute moment in this Gaussian tail gives
\[
\begin{aligned}
\E[\mathrm{err}^{\mathsf j}_t(v)]&\le\E\bigl[|\mathsf g|\one\{|\mathsf g|\ge\tfrac7{32}\}\bigr]=\sqrt{\frac{2\mathsf v}\pi}\exp\Bigl(-\frac{(7/32)^2}{2\mathsf v}\Bigr)\\
&\le\frac{e^{-3\Upsilon}}{8\sqrt{\pi\Upsilon}}\le\frac{(2T)^{-4}}{20}.
\end{aligned}
\]
To bound the Gaussian expression, we use its monotonicity in $\mathsf v$ and the calculation $(7/32)^2\cdot64\Upsilon=49\Upsilon/16\ge3\Upsilon$, while the last inequality uses $e^{-3}\le2^{-4}$, $2^\Upsilon=2T$, $\Upsilon\ge2$, and $8\sqrt{2\pi}\ge20$. There are at most $T^2$ pairs $(t,v)$ because the member set has size at most the horizon, so summing the per-pair bound gives a total at most $1/(320T^2)\le1/1024$, proving the clipping bound.
\end{proof}

\paragraph{Regret and action counts.}
For $v\in\mathcal S$, let $T_v$ be the number of rounds in which the learner plays $v$, and let $T_\circ\coloneqq T-\sum_{v\in\mathcal S}T_v$ be the number of rounds outside $\mathcal S$. Let $\E_{\mathsf j}$ denote expectation in environment $\mathsf j$, over the walk, the signs, and the learner's randomness, and let $\Reg^{\mathsf j}_T$ denote the expected regret there, averaged over the loss sequence as well. The next lemma relates this regret to the action counts and the number of switches.

\begin{lemma}[Regret]
\label{lem:lower-regret}
Every learner satisfies
\[
\begin{aligned}
\Reg^0_T&\ge\E_0[N_T]+\tfrac14\E_0[T_\circ]-\tfrac1{512},\\
\Reg^{\mathsf j}_T&\ge\vartheta\bigl(T-\E_{\mathsf j}[T_{\mathsf j}]\bigr)-\tfrac1{512}\quad\text{for every }\mathsf j\in\mathcal S.
\end{aligned}
\]
\end{lemma}

\begin{proof}
Fix $\mathsf j\in\{0\}\cup\mathcal S$, let $\bar v\coloneqq\mathsf j$ when $\mathsf j\in\mathcal S$, and let $\bar v$ be a fixed member of $\mathcal S$ when $\mathsf j=0$. We compare the learner with this fixed action, whose total loss is at least that of the best action in hindsight, and obtain
\[
\Reg^{\mathsf j}_T\ge\sum_{t=1}^T\E_{\mathsf j}\bigl[\ell^{\mathsf j}_t(I_t)-\ell^{\mathsf j}_t(\bar v)\bigr]+\E_{\mathsf j}[N_T].
\]
We bound the loss difference inside the expectation by considering the played action. If $I_t=\bar v$, this difference is zero, while if $I_t=u\in\mathcal S\setminus\{\bar v\}$, \cref{eq:lower-gap} shows that it is at least $\iota_t(\Xi_{u,t}-\Xi_{\bar v,t})+\vartheta\one\{\mathsf j\in\mathcal S\}-\mathrm{err}^{\mathsf j}_t(u)-\mathrm{err}^{\mathsf j}_t(\bar v)$. If $I_t\notin\mathcal S$, it is at least $1-\tfrac34=\tfrac14$.

The fresh sign makes the walk contribution vanish in expectation even though the learner can use past feedback to choose its action. Indeed, through feedback from rounds before $t$, $I_t$ is a function of the learner's randomness, $\Xi_1,\ldots,\Xi_{t-1}$, and $\iota_1,\ldots,\iota_{t-1}$, while $\iota_t$ is independent of these quantities jointly with $\Xi_t$ and has mean zero. Consequently, $\E_{\mathsf j}[\one\{I_t=u\}\iota_t(\Xi_{u,t}-\Xi_{\bar v,t})]=0$. In each round, the clipping errors of the played member and of $\bar v$ sum to at most $2\sum_{v\in\mathcal S}\mathrm{err}^{\mathsf j}_t(v)$, so summing the loss differences and using \cref{lem:clip} gives
\[
\Reg^{\mathsf j}_T\ge\vartheta\one\{\mathsf j\in\mathcal S\}\,\E_{\mathsf j}\bigl[T-T_{\bar v}-T_\circ\bigr]+\tfrac14\E_{\mathsf j}[T_\circ]+\E_{\mathsf j}[N_T]-\tfrac1{512}.
\]
For $\mathsf j=0$, this is the first claim. For $\mathsf j\in\mathcal S$, we have $\bar v=\mathsf j$, and $\vartheta\le\tfrac14$ together with $N_T\ge0$ makes the right side at least $\vartheta(T-\E_{\mathsf j}[T_{\mathsf j}])-\tfrac1{512}$: rounds outside the set contribute at least as much as the mean gap, and the switching term is nonnegative. This proves the second claim and completes the regret comparison.
\end{proof}

\paragraph{An augmented game.}
We now give the learner additional observations whose information content is easier to calculate, and we will verify that these observations allow it to reconstruct all feedback from the graph. For $c\in\{0\}\cup\mathcal S$, let $e_c$ be the corresponding coordinate vector of $\R^{\{0\}\cup\mathcal S}$, and for $\mathsf j\in\{0\}\cup\mathcal S$ and a round $t$, define the shift
\[
\mathfrak s^{\mathsf j}_t\coloneqq
\begin{cases}
-\vartheta\one\{\mathsf j\in\mathcal S_{I_t}\}\,e_0 & \text{if }I_t\in\mathcal S,\\
-\iota_t\vartheta\one\{\mathsf j\in\mathcal S_{I_t}\}\,e_{\mathsf j} & \text{if }I_t\notin\mathcal S,
\end{cases}
\]
and let $\mathfrak s^{\mathsf j}_0\coloneqq0$. All shifts vanish in the reference environment because $0$ is not a member of the set. In environment $\mathsf j$, let the augmented observation in round $t$ be $O_t\coloneqq\Xi_t+\mathfrak s^{\mathsf j}_t$, and let the learner receive $(\iota_t,O_t)$ after choosing $I_t$, with the coordinates of the observation indexed in the same way as those of the walk. Thus, a round at a member shifts the common coordinate if the member it observes is $\mathsf j$, while a round outside $\mathcal S$ shifts the coordinate of $\mathsf j$, using the sign $\iota_t$, if it observes $\mathsf j$. The next lemma verifies that these observations reveal at least all losses supplied by the graph.

\begin{lemma}[Reconstruction]
\label{lem:reconstruct}
There are maps that do not depend on $\mathsf j$ and compute, from $(I_t,\iota_t,O_t)$, every loss that the learner observes in round $t$.
\end{lemma}

\begin{proof}
For $v\in\mathcal S$, $o\in\R^{\{0\}\cup\mathcal S}$, and a sign $\iota$, let $\mathrm{dec}_v(o,\iota)\coloneqq\mathrm{clip}(\tfrac12+o_0+\iota o_v)$. The observed actions are determined by $I_t$ and the known graph, every observed loss outside $\mathcal S$ equals one, and the observed members form the set $\mathcal S_{I_t}$. It therefore suffices to prove that $\mathrm{dec}_v(O_t,\iota_t)=\ell^{\mathsf j}_t(v)$ for every $v\in\mathcal S_{I_t}$, with nothing to check for member losses if this set is empty.

First suppose that $I_t\in\mathcal S$ and fix such an observed member, so the hypothesis on observations within the set forces $\mathcal S_{I_t}=\{v\}$. The shift then gives $O_{0,t}=\Xi_{0,t}-\vartheta\one\{v=\mathsf j\}$ and $O_{v,t}=\Xi_{v,t}$, which yield $\mathrm{dec}_v(O_t,\iota_t)=\ell^{\mathsf j}_t(v)$ when substituted into the decoder.

If $I_t\notin\mathcal S$, then $O_{0,t}=\Xi_{0,t}$, so we only need to check whether the observed member's own coordinate is shifted. For $v\ne\mathsf j$, the shift leaves coordinate $v$ unchanged and gives $\mathrm{dec}_v(O_t,\iota_t)=\mathrm{clip}(\tfrac12+\Xi_{0,t}+\iota_t\Xi_{v,t})=\ell^{\mathsf j}_t(v)$. For $v=\mathsf j$, the fact that this member is observed gives $O_{\mathsf j,t}=\Xi_{\mathsf j,t}-\iota_t\vartheta$, hence $\iota_tO_{\mathsf j,t}=\iota_t\Xi_{\mathsf j,t}-\vartheta$ and again $\mathrm{dec}_{\mathsf j}(O_t,\iota_t)=\ell^{\mathsf j}_t(\mathsf j)$. The decoder and the rule for reporting constant losses outside the set are independent of $\mathsf j$, so they reconstruct the entire graph feedback as claimed.
\end{proof}

By \cref{lem:reconstruct}, any learner for $\mathcal G$ can be run in the augmented game by feeding it the reconstructed feedback, which preserves its law of actions within each environment. Let $\mathbb Q_{\mathsf j}$ be the law of this augmented transcript in environment $\mathsf j$, where the transcript consists of the learner's randomness and $(I_t,\iota_t,O_t)$ for $t\in[T]$. The counts $T_v$, $T_\circ$, and $N_T$ are functions of the transcript, so a comparison of transcript laws will also compare the expected action counts used in the regret bound.

\begin{lemma}[Relative entropy]
\label{lem:lower-kl}
\[
\sum_{\mathsf j\in\mathcal S}\mathrm{KL}(\mathbb Q_0\,\Vert\,\mathbb Q_{\mathsf j})\le\frac{2\mathsf m\vartheta^2\Upsilon}{\varsigma^2}\,\E_0\bigl[N_T+T_\circ+1\bigr].
\]
\end{lemma}

\begin{proof}
We first calculate the relative entropy by comparing the conditional observation laws one round at a time. Let $\mathcal H_t$ consist of the learner's randomness, $(I_s,\iota_s,O_s)$ for $s<t$, and $(I_t,\iota_t)$, so it includes the current action and sign but not the current augmented observation. Given $\mathcal H_t$, both shifts $\mathfrak s^{\mathsf j}_t$ and $\mathfrak s^{\mathsf j}_{\mathrm{pa}(t)}$ and the parent observation $O_{\mathrm{pa}(t)}$ are fixed, with $O_0\coloneqq0$, and the recursion for the walk gives
\[
O_t=O_{\mathrm{pa}(t)}+\upsilon_t+\mathfrak s^{\mathsf j}_t-\mathfrak s^{\mathsf j}_{\mathrm{pa}(t)},
\]
where $\upsilon_t$ is independent of $\mathcal H_t$, since the past observations depend only on earlier increments and the current sign is independent of the whole walk. Thus, conditional on $\mathcal H_t$, the observation $O_t$ is Gaussian with covariance $\varsigma^2I$, and at any fixed history its mean in environment $\mathsf j$ differs from its mean in the reference environment by $\mathfrak s^{\mathsf j}_t-\mathfrak s^{\mathsf j}_{\mathrm{pa}(t)}$. The signs and the learner's choices have the same conditional laws in every environment: signs are drawn independently, and the learner uses the same decision rule applied to feedback reconstructed by maps that do not depend on the environment. These quantities therefore contribute no relative entropy, while two Gaussian laws with covariance $\varsigma^2I$ contribute the squared Euclidean distance between their means divided by $2\varsigma^2$. The chain rule for relative entropy \citep[Chapters~14 and~15]{lattimore2020bandit} consequently gives
\[
\mathrm{KL}(\mathbb Q_0\,\Vert\,\mathbb Q_{\mathsf j})=\frac1{2\varsigma^2}\E_0\Bigl[\sum_{t=1}^T\bigl\|\mathfrak s^{\mathsf j}_t-\mathfrak s^{\mathsf j}_{\mathrm{pa}(t)}\bigr\|_2^2\Bigr].
\]
We now sum this expression over $\mathsf j\in\mathcal S$ and bound the changes in the shifts along parent edges. In every round $u\ge1$, their definition gives $\sum_{\mathsf j\in\mathcal S}\|\mathfrak s^{\mathsf j}_u\|_2^2=\vartheta^2|\mathcal S_{I_u}|\le\mathsf m\vartheta^2$, and $\mathfrak s^{\mathsf j}_0=0$. If $\mathrm{pa}(t)\ge1$ and $I_t=I_{\mathrm{pa}(t)}\in\mathcal S$, then $\mathfrak s^{\mathsf j}_t=\mathfrak s^{\mathsf j}_{\mathrm{pa}(t)}$ for every $\mathsf j$, because the shift in a round at a member depends only on the played member and does not involve the fresh sign. Otherwise, the inequality $\|\mathfrak s^{\mathsf j}_t-\mathfrak s^{\mathsf j}_{\mathrm{pa}(t)}\|_2^2\le2\|\mathfrak s^{\mathsf j}_t\|_2^2+2\|\mathfrak s^{\mathsf j}_{\mathrm{pa}(t)}\|_2^2$ bounds the sum over $\mathsf j$ by $4\mathsf m\vartheta^2$. With $I_0$ a symbol distinct from all actions, any such remaining case either has different actions at the two ends of the parent edge or has its current action outside the member set, so
\[
\sum_{\mathsf j\in\mathcal S}\bigl\|\mathfrak s^{\mathsf j}_t-\mathfrak s^{\mathsf j}_{\mathrm{pa}(t)}\bigr\|_2^2\le4\mathsf m\vartheta^2\bigl(\one\{I_t\ne I_{\mathrm{pa}(t)}\}+\one\{I_t\notin\mathcal S\}\bigr).
\]
By \cref{lem:dyadic}(iii), the first indicator summed over all rounds is at most $\Upsilon(N_T+1)$, while the sum of the second is $T_\circ$. Since $\Upsilon\ge1$, summing the right side of the display over all rounds gives at most $4\mathsf m\vartheta^2\Upsilon(N_T+T_\circ+1)$, and taking expectation under the reference law and dividing by $2\varsigma^2$ in the entropy identity proves the stated bound.
\end{proof}

\paragraph{Conclusion.}
We now combine the regret and information bounds by considering two cases. If the learner makes many switches or spends many rounds outside $\mathcal S$ in the reference environment, the regret there is already large. Otherwise, the summed information bound prevents the learner from spending almost all rounds on the better member across the environments $\mathsf j\in\mathcal S$, and the mean gap gives a lower bound on their average regret.

\begin{proposition}[Regret in the constructed environments]
\label{prop:bayes}
If $\vartheta T\ge1$, then every learner satisfies $\max\bigl\{\Reg^0_T,\ \frac1{|\mathcal S|}\sum_{\mathsf j\in\mathcal S}\Reg^{\mathsf j}_T\bigr\}\ge\vartheta T/8$.
\end{proposition}

\begin{proof}
Suppose first that $\E_0[N_T]+\tfrac14\E_0[T_\circ]\ge\vartheta T/4$, so that switches and rounds outside the set already contribute enough to the reference regret. By \cref{lem:lower-regret}, $\Reg^0_T\ge\vartheta T/4-\tfrac1{512}\ge\vartheta T/8$, where the last inequality uses $\vartheta T\ge1$, proving the claim in this case.

Suppose instead that $\E_0[N_T]+\tfrac14\E_0[T_\circ]<\vartheta T/4$. Multiplying this inequality by four and using the nonnegativity of the switch count gives $\E_0[N_T+T_\circ+1]<\vartheta T+1\le2\vartheta T$, so \cref{lem:lower-kl}, together with $\varsigma^{-2}=256\Upsilon^2$ and the chosen value of $\vartheta$, yields
\[
\sum_{\mathsf j\in\mathcal S}\mathrm{KL}(\mathbb Q_0\,\Vert\,\mathbb Q_{\mathsf j})<1024\,\mathsf m\Upsilon^3\vartheta^3T=\frac{|\mathcal S|}{32}.
\]
Let $\mathrm{TV}_{\mathsf j}$ be the total variation distance between $\mathbb Q_0$ and $\mathbb Q_{\mathsf j}$. Pinsker's inequality gives $\mathrm{TV}_{\mathsf j}\le\sqrt{\mathrm{KL}(\mathbb Q_0\Vert\mathbb Q_{\mathsf j})/2}$, and the concavity of the square root makes the average of $\mathrm{TV}_{\mathsf j}$ over $\mathsf j\in\mathcal S$ less than $\sqrt{1/64}=1/8$, by the preceding entropy bound. A random variable with values in $[0,T]$, such as $T_{\mathsf j}$, has expectations under two laws that differ by at most $T$ times their total variation distance \citep[Chapters~14 and~15]{lattimore2020bandit}. Applying this to the number of times the learner plays the better member gives $\E_{\mathsf j}[T_{\mathsf j}]\le\E_0[T_{\mathsf j}]+T\,\mathrm{TV}_{\mathsf j}$, and averaging these inequalities uses $\sum_{\mathsf j\in\mathcal S}T_{\mathsf j}\le T$ and $|\mathcal S|\ge2$ to give
\[
\frac1{|\mathcal S|}\sum_{\mathsf j\in\mathcal S}\E_{\mathsf j}[T_{\mathsf j}]<\frac T{|\mathcal S|}+\frac T8\le\frac{5T}8.
\]
Thus a constant fraction of rounds, on average over the environments, is spent away from the better member, and \cref{lem:lower-regret} gives $\frac1{|\mathcal S|}\sum_{\mathsf j\in\mathcal S}\Reg^{\mathsf j}_T\ge\vartheta(T-\tfrac{5T}8)-\tfrac1{512}\ge\vartheta T/8$. This proves the claim in the second case and completes the proof.
\end{proof}

\begin{proof}[Proof of \cref{thm:lower}]
For any fixed learner, we construct a distribution over oblivious loss sequences under which its expected regret is at least $(|\mathcal S|/\mathsf m)^{1/3}T^{2/3}/(256\Upsilon)$, which ensures that some loss sequence in the support has at least this regret. If $\vartheta T\ge1$, \cref{prop:bayes} shows that either the reference environment or the uniform mixture of the environments $\mathsf j\in\mathcal S$ gives expected regret at least $\vartheta T/8=(|\mathcal S|/\mathsf m)^{1/3}T^{2/3}/(256\Upsilon)$, as required.

If $\vartheta T<1$, we use a simpler construction: draw $j^\star$ uniformly from two fixed actions, give $j^\star$ loss zero and every other action loss one in round $1$, and give every action loss zero in all later rounds. The learner chooses $I_1$ without information about $j^\star$, so its expected loss in round $1$ is at least $\tfrac12$, while $j^\star$ has total loss zero. Later losses and switching costs are nonnegative, so its expected regret is at least $\tfrac12>\vartheta T/8$. Both cases therefore yield the desired loss sequence, and recalling the initial reduction to a subset of the original member set proves \cref{thm:lower}.
\end{proof}

\subsection{Sets with few observers}
\label{app:few-observers}

We now prove \cref{lem:few-observers}, following the initial removal argument in \citep[Lemma~8]{AlonCDK15} and then using independent sampling and deletions. We first remove actions until no action observes a large fraction of those that remain, then sample a small random subset and delete members to enforce the required observation and edge conditions. For an action $a$ and a set $\mathcal W$ of actions, let $\mathcal W_a\coloneqq\{v\in\mathcal W:(a,v)\in E\}$ denote the members of $\mathcal W$ that $a$ observes.

\begin{proof}[Proof of \cref{lem:few-observers}]
Let $\gamma\coloneqq\gamma(\mathcal G)$ and $\mathsf f\coloneqq2\log K/\gamma$, so that the hypothesis of the lemma gives $\mathsf f\le\tfrac1{16}$ and $\gamma\ge2$. We use the fraction $\mathsf f$ to control how many actions any one action can observe in the set that remains after the initial removals.

\emph{Removing large observed sets.} Start with $\mathcal W\coloneqq[K]$, and whenever an action $a$ satisfies $|\mathcal W_a|>\mathsf f|\mathcal W|$, choose such an action and remove $\mathcal W_a$ from $\mathcal W$. Each removal reduces $|\mathcal W|$ by a factor strictly smaller than $1-\mathsf f$, so if the process performed $\gamma-1$ removals, the remaining size would satisfy
\[
|\mathcal W|<K(1-\mathsf f)^{\gamma-1}\le Ke^{-\mathsf f(\gamma-1)}\le Ke^{-\mathsf f\gamma/2}=1,
\]
and $\mathcal W$ would be empty because its size is an integer. Every removed action is observed by the action chosen when it was removed, so these $\gamma-1$ chosen actions would form an observing set, contradicting the definition of $\gamma$. The process must therefore stop after fewer than $\gamma-1$ removals, and the same observing-set argument shows that its final set $\mathcal W$ is nonempty. At termination, every action $a$ satisfies $|\mathcal W_a|\le\mathsf f|\mathcal W|$, while observability ensures that some action observes a member of $\mathcal W$, giving $\mathsf f|\mathcal W|\ge1$. These are the two properties of the remaining set needed for the sampling step.

\emph{Sampling and deletion.} Let $\mathsf q\coloneqq1/(4\mathsf f|\mathcal W|)\le\tfrac14$, and include each member of $\mathcal W$ in a random set $\mathcal X$ independently with probability $\mathsf q$. For every action $a$, the size $|\mathcal X_a|$ is binomial with $|\mathcal W_a|\le\mathsf f|\mathcal W|$ trials and success probability $\mathsf q$, so $|\mathcal W_a|\mathsf q\le\tfrac14$ and $\Prb[|\mathcal X_a|=b]\le(|\mathcal W_a|\mathsf q)^b/b!\le4^{-b}/b!$ for every integer $b\ge0$, with zeroth powers interpreted as one. We will use this probability bound to control how many deletions are caused by large observed subsets.

Let $\mathsf k\coloneqq\lceil\log K\rceil+1$, and delete members of $\mathcal X$ in two steps. First, for every action $a$ with $|\mathcal X_a|\ge\mathsf k$, delete all members of $\mathcal X_a$, and second, for every ordered pair $(u,v)$ of distinct remaining members with $(u,v)\in E$, delete $u$. Let $\mathcal S$ be the set that remains, which has no edges between distinct members because the second step deletes the first endpoint of every remaining directed edge. Every action $a$ observes at most $\mathsf k-1\le1+\log K$ members of $\mathcal S$: if $|\mathcal X_a|\ge\mathsf k$, the first step deletes all of them, and otherwise there are at most $|\mathcal X_a|\le\mathsf k-1$ to begin with. Thus the required graph properties hold for every sampling outcome, and it remains to show that some outcome leaves a sufficiently large set.

\emph{Size of the remaining set.} The expected size of $\mathcal X$ is $\mathsf q|\mathcal W|=1/(4\mathsf f)=\gamma/(8\log K)$, and we will subtract upper bounds on the expected numbers of deletions in the two steps. The first step deletes at most $\sum_{a\in[K]}|\mathcal X_a|\one\{|\mathcal X_a|\ge\mathsf k\}$ members, with any repeated deletions only making this an overestimate, and the binomial probability bound gives an expectation at most
\[
\sum_{a\in[K]}\sum_{b=\mathsf k}^\infty b\,\frac{4^{-b}}{b!}\le K\sum_{b=\mathsf k}^\infty4^{-b}=\frac43K4^{-\mathsf k}\le\frac13K^{1-\log4}\le\frac13,
\]
where the first inequality uses that each factorial in the denominator is at least its integer argument, and the last two use $4^{-\mathsf k}\le4^{-1-\log K}=K^{-\log4}/4$ and $\log4>1$. The second step deletes at most as many members as there are ordered pairs $(u,v)$ of distinct members of $\mathcal X$ with $(u,v)\in E$, since each deletion is caused by such a pair. Each pair of distinct members of $\mathcal W$ lies in $\mathcal X$ with probability $\mathsf q^2$, and every $u\in\mathcal W$ observes at most $\mathsf f|\mathcal W|$ members of $\mathcal W$, so the expected number of these pairs is at most $\mathsf q^2|\mathcal W|\cdot\mathsf f|\mathcal W|=\mathsf q|\mathcal W|/4$. Subtracting the two expected deletion bounds from the expected sampled size therefore gives
\[
\E[|\mathcal S|]\ge\frac34\,\mathsf q|\mathcal W|-\frac13=\frac{3\gamma}{32\log K}-\frac13\ge\frac{\gamma}{16\log K},
\]
where the last inequality follows from $\gamma\ge32\log K$. Since this is a lower bound on the expected remaining size, some sampling outcome has a set $\mathcal S$ at least as large as required, and the graph properties already verified for every outcome complete the proof of \cref{lem:few-observers}.
\end{proof}

\end{document}